\documentclass[twoside,11pt]{article}

\usepackage{jmlr2e}
\usepackage{csquotes}
\usepackage{color}
\usepackage{colordvi}
\usepackage{amssymb}
\usepackage{colordvi}
\usepackage{float}
\usepackage{algorithm}
\usepackage{algpseudocode}

\usepackage{graphicx}
\usepackage{subcaption}
\usepackage{enumerate}
\usepackage{wrapfig}
\usepackage{enumitem}
\usepackage{bm}

\usepackage{bbm}
\usepackage{bigints}
\usepackage{soul}
\usepackage{pgffor}
\usepackage{tikz}
\usepackage[colorlinks]{hyperref}  
\usepackage{xcolor}
\usepackage{wrapfig}
\usepackage{minitoc}

\usepackage[%
    minnames=1,maxnames=99,maxcitenames=3,
    style=numeric-comp,
    sorting=none,
    doi=false,url=false,
    giveninits=true,
    hyperref,natbib,backend=biber]{biblatex}
\renewbibmacro{in:}{%
  \ifentrytype{article}{}{\printtext{\bibstring{in}\intitlepunct}}}
\bibliography{sample}

\usepackage[capitalize]{cleveref}

\usepackage{mathtools}

\definecolor{niceblue}{rgb}{0.10, 0.14, 0.76} 

\definecolor{nicered}{rgb}{0.70, 0.0, 0.0} 

\AtBeginDocument{%
\hypersetup{
    citecolor=niceblue,
    linkcolor=red,   
    urlcolor=niceblue,
    linktoc=none}
}

\newcommand{\reals}{\mathbb{R}}

\newcommand{\ito}{It$\hat{o}$ }

\newcommand{\E}{\mathbb{E}}
\newcommand{\Tr}{\textup{Tr}}

\newcommand{\normal}{\mathcal{N}}
\newcommand{\ind}{\mathbbm{1}}
\newcommand{\iid}{\textit{iid}~}

\newcommand{\bigO}{\mathcal{O}}

\newtheorem{prop}{Proposition}
\newtheorem{thm}{Theorem}

\newtheorem{lemma}{Lemma}

\newtheorem{definition}{Definition}

\jmlrheading{}{}{}{}{}{Soufiane Hayou}

\begin{document}

\title{On the infinite-depth limit of finite-width neural networks}

\author{\name Soufiane Hayou \email hayou@nus.edu.sg \\
      \addr Department of Mathematics\\
      National University of Singapore}

\editor{}

\maketitle\doparttoc 
\faketableofcontents 
\part{} 
\vspace{-1cm}
\begin{abstract}
In this paper, we study the infinite-depth limit of finite-width residual neural networks with random Gaussian weights. With proper scaling, we show that by fixing the width and taking the depth to infinity, the pre-activations converge in distribution to a zero-drift diffusion process. Unlike the infinite-width limit where the pre-activation converge weakly to a Gaussian random variable, we show that the infinite-depth limit yields different distributions depending on the choice of the activation function. We document two cases where these distributions have closed-form (different) expressions. We further show an intriguing change of regime phenomenon of the post-activation norms when the width increases from $3$ to $4$. Lastly, we study the sequential limit infinite-depth-then-infinite-width and compare it with the more commonly studied infinite-width-then-infinite-depth limit.\\

\textbf{Peer Reviewed version:} The first version of this paper was published at Transactions of Machine Learning Research (TMLR, \url{https://openreview.net/forum?id=RbLsYz1Az9}). This version contains some updates and improvements in the proofs.
\end{abstract}

\section{Introduction}

The empirical success of over-parameterized neural networks has sparked a growing interest in the theoretical understanding of these models. The large number of parameters -- millions if not billions -- and the complex (non-linear) nature of the neural computations (presence of non-linearities) make this hypothesis space highly non-trivial. However, in certain situations, increasing the number of parameters has the effect of `placing' the network in some `average' regime that simplifies the theoretical analysis. This is the case with the infinite-width asymptotics of random neural networks. The infinite-width limit of neural network architectures has been extensively studied in the literature, and has led to many interesting theoretical and algorithmic innovations. We summarize these results below.
\begin{itemize}[leftmargin=*]
    \item \emph{Initialization schemes}: the infinite-width limit of different neural architectures has been  extensively studied in the literature. In particular, for multi-layer perceptrons (MLP), a new initialization scheme that stabilizes forward and backward propagation (in the infinite-width limit) was derived in \citep{poole, samuel2017}. This initialization scheme is known as the Edge of Chaos, and empirical results show that it significantly improves performance. In \cite{yang2017meanfield, hayou21stable}, the authors derived similar results for the ResNet architecture, and showed that this architecture is \emph{placed} by-default on the Edge of Chaos for any choice of the variances of the initialization weights (Gaussian weights). In \cite{hayou2019impact}, the authors showed that an MLP that is initialized on the Edge of Chaos exhibits similar properties to ResNets, which might partially explain the benefits of the Edge of Chaos initialization. 
    
    \item \emph{Gaussian process behaviour}: Multiple papers (e.g. \cite{neal, lee_gaussian_process, yang_tensor3_2020, matthews, hron20attention}) studied the weak limit of neural networks when the width goes to infinity. The results show that a randomly initialized neural network (with Gaussian weights) has a similar behaviour to that of a Gaussian process, for a wide range of neural architectures, and under mild conditions on the activation function. In \cite{lee_gaussian_process}, the authors leveraged this result and introduced the neural network Gaussian process (NNGP), which is a Gaussian process model with a neural kernel that depends on the architecture and the activation function. Bayesian regression with the NNGP showed that NNGP surprisingly achieves performance close to the one achieved by an SGD-trained finite-width neural network.
    
    The large depth limit of this Gaussian process was studied in \cite{hayou21stable}, where the authors showed that with proper scaling, the infinite-depth (weak) limit is a Gaussian process with a universal kernel\footnote{A kernel is called universal when any continuous function on some compact set can be approximated arbitrarily well with kernel features.}.
    \item \emph{Neural Tangent Kernel (NTK)}: the infinite-width limit of the NTK is the so-called NTK regime or Lazy-training regime. This topic has been extensively studied in the literature. The optimization and generalization properties (and some other aspects) of the NTK have been studied in \cite{Liu2022connecting, arora2019finegrained, seleznova2022ntk, hayou2019trainingdynamicsNTK}. The large depth asymptotics of the NTK have been studied in \citep{hayou_ntk, hayou2022curse, jacot2019freeze, xiao2020disentangling}. We refer the reader to \cite{jacot2022thesis} for a comprehensive discussion on the NTK.
    
    \item \emph{Others}: the theory of  infinite-width neural networks has also been utilized for network pruning  \cite{hayou_pruning, hayou2020pruning2}, regularization \cite{hayou2021stochasticdepth}, feature learning \cite{hayou_eh}, and ensembling methods \citep{he2020ntkensembles} (this is by no means an exhaustive list).
\end{itemize}

The theoretical analysis of infinite-width neural networks has certainly led to many interesting  (theoretical and practical) discoveries. However, most works on this limit consider a fixed depth network.  \emph{What about infinite-depth}? Existing works on the infinite-depth limit can generally be divided into three categories:

\begin{itemize}[leftmargin=*]
    \item \emph{Infinite-width-then-infinite-depth limit}: in this case, the width is taken to infinity first, then the depth is take to infinity. This is the infinite-depth limit of infinite-width neural networks. This limit was particularly used to derive the Edge of Chaos initialization scheme \citep{samuel2017, poole}, study the impact of the activation function \citep{hayou2019impact}, the behaviour of the NTK \citep{hayou_ntk, xiao2020disentangling}, kernel shaping \citep{martens2021rapid, zhang2022deep} etc.
    
    \item \emph{The joint infinite-width-and-depth limit}: in this case, the depth-to-width ratio is fixed, and therefore, the width and depth are jointly taken to infinity at the same time. There are few works that study the joint width-depth limit. For instance, in \citep{li21loggaussian}, the authors showed that for a special form of residual neural networks (ResNet), the network output exhibits a (scaled) log-normal behaviour in this joint limit. This is different from the sequential limit where width is taken to infinity first, followed by the depth, in which case the distribution of the network output is asymptotically normal (\citep{samuel2017, hayou2019impact}). In \citep{li2022sde}, the authors studied the covariance kernel of an MLP in the joint limit, and showed that it converges weakly to the solution of Stochastic Differential Equation (SDE). In \cite{Hanin2020Finite}, the authors showed that in the joint limit case, the NTK of an MLP remains random when the width and depth jointly go to infinity. This is different from the deterministic limit of the NTK where the width is taken to infinity before depth \citep{hayou_ntk}. More recently, in \citep{hanin2022correlation}, the author explored the impact of the depth-to-width ratio on the correlation kernel and the gradient norms in the case of an MLP architecture, and showed that this ratio can be interpreted as an effective network depth.
    
    \item \emph{Infinite-depth limit of finite-width neural networks}:  in both previous limits (infinite-width-then-infinite-depth limit, and the joint infinite-width-depth limit), the width goes to infinity. Naturally, one might ask what happens if width is fixed and depth goes to infinity? What is the limiting distribution of the network output at initialization? In \citep{hanin2019finitewidth}, the author showed that neural networks with bounded width are still universal approximators, which motivates the study of finite-width large depth neural networks. In \citep{peluchetti2020resnetdiffusion}, the authors showed that the pre-activations of a particular ResNet architecture converge weakly to a diffusion process in the infinite-depth limit. This is the result of the fact that ResNet can be seen as discretizations of SDEs (see \cref{sec:infinite_depth_limit}).
\end{itemize}
In the present paper, we study the infinite-depth limit of finite-width ResNet with random Gaussian weights (an architecture that is different from the one studied in \citep{peluchetti2020resnetdiffusion}). We are particularly interested in the \emph{asymptotic behaviour of the pre/post-activation values}. Our contributions are four-fold:
\begin{enumerate}
    \item Unlike the infinite-width limit, we show that the resulting distribution of the pre-activations in the infinite-depth limit is not necessarily Gaussian. In the simple case of networks of width $1$, we study two cases where we obtain known but completely different distributions by carefully choosing the activation function.
    \item For ReLU activation function, we introduce and discuss the phenomenon of \emph{network collapse}. This phenomenon occurs when the pre-activations in some hidden layer have all non-positive values which results in zero post-activations. This leads to a stagnant network where increasing the depth beyond a certain level has no effect on the network output. For any fixed width, we show that in the infinite-depth limit, network collapse is a zero-probability event, meaning that almost surely, all post-activations in the network are non-zero.
    
    \item For networks with general width, where the distribution of the pre-activations is generally intractable, we focus on the norm of the post-activations with ReLU activation function, and show that this norm has approximately a Geometric Bronwian Motion (GBM) dynamics. We call this  Quasi-GBM. We also shed light on a regime change phenomenon that occurs when the width $n$ increases from $3$ to $4$. For width $n \leq 3$, resp. $n \geq 4$, the logarithmic growth factor of the post-activations is , resp. positive. 
    \item  We study the sequential limit infinite-depth-then-infinite-width, which is the converse of the more commonly studied  infinite-width-then-infinite-depth limit, and show some key differences between these limits. We particularly show that the pre-activations converge to the solution of a Mckean-Vlasov process, which has marginal Gaussian distributions, and thus we recover the Gaussian behaviour in this limit. We compare the two sequential limits and discuss some differences.
\end{enumerate}

The proofs of the theoretical results are provided in the appendix and referenced after each result. Empirical evaluations of these theoretical findings are also provided.

\section{The infinite-depth limit}\label{sec:infinite_depth_limit}
Hereafter, we denote the width, resp. depth, of the network by $n$, resp. $L$. We also denote the input dimension by $d$. Let $d, n, L \geq 1$, and consider the following ResNet architecture of width $n$ and depth $L$
\begin{equation}\label{eq:resnet}
\begin{aligned}
Y_0 &= W_{in} x, \quad x \in \reals^d\\
Y_l &= Y_{l-1} + \frac{1}{\sqrt{L}} W_l \phi(Y_{l-1}), \quad l = 1, \dots, L,
\end{aligned}
\end{equation}
where $\phi: \reals \to \reals$ is the activation function, $L \geq 1$ is the network depth, $W_{in} \in \reals^{n \times d}$, and $W_l \in \reals^{n \times n}$ is the weight matrix in the $l^{th}$ layer. We assume that the weights are randomly initialized with \iid Gaussian variables $W_l^{ij} \sim \normal(0, \frac{1}{n})$, $W_{in}^{ij} \sim \normal(0, \frac{1}{d})$. For the sake of simplification, we only consider networks with no bias, and we omit the dependence of $Y_l$ on $n$ in the notation. While the activation function is only defined for real numbers, we will abuse the notation and write $\phi(z) = (\phi(z^1), \dots, \phi(z^k))$ for any $k$-dimensional vector $z = (z^1, \dots, z^k) \in \reals^k$ for any $k \geq 1$. We refer to the vectors $\{Y_l, l=0, \dots, L\}$ by the \emph{pre-activations} and the vectors $\{\phi(Y_l), l=0, \dots, L\}$ by the \emph{post-activations}. Hereafter, $x \in \reals^d$ is fixed, and we assume that $x \neq 0$.

The $1/\sqrt{L}$ scaling in \cref{eq:resnet} is not arbitrary. This specific scaling was shown to stabilize the norm of $Y_l$ as well as gradient norms in the large depth limit (e.g. \cite{hayou21stable, hayou_pruning, marion2022scaling}). In the next result, we show that the infinite depth limit  of \cref{eq:resnet} (in the sens of the distribution) exists and has the same distribution of the solution of a stochastic differential equation. In the case of a single input, this has already been shown in \cite{peluchetti2020resnetdiffusion}. The details are provided in \cref{appendix:stochastic_calculus}. We also generalize this result in the case of multiple inputs and obtain similar SDE dynamics (see \cref{prop:main_conv_multiple} in the Appendix).

\begin{prop}\label{prop:main_conv}
Assume that the activation function $\phi$ is Lipschitz on $\reals^n$. Then, in the limit $L \to \infty$, the process $X^L_t = Y_{\lfloor t L\rfloor}$, $t\in [0,1]$, converges in distribution to the solution of the following SDE 
\begin{equation}\label{eq:main_sde}
    dX_t = \frac{1}{\sqrt{n}}\|\phi(X_t)\| dB_t, \quad X_0 = W_{in} x,
\end{equation}
where $(B_t)_{t\geq 0}$ is a Brownian motion (Wiener process), independent from $W_{in}$. Moreover, we have that for any $t \in [0,1]$ and any Lipschitz function $\Psi:\reals^n \to \reals$, 
$$
\E \Psi(Y_{\lfloor t L\rfloor}) = \E \Psi(X_t) + \bigO(L^{-1/2}),
$$
where the constant in $\bigO$ does not depend on $t$.\\
Moreover, if the activation function $\phi$ is only locally Lipschitz, then $X^L_t$ converges locally to $X_t$. More precisely, for any fixed $r > 0$, we consider the stopping times 
$$
\tau^L = \inf \{t \geq 0: \|X^L_t\| \geq r\}, \quad \tau = \inf \{t \geq 0: \|X_t\| \geq r\},
$$
then the stopped process $X^L_{t \land \tau^L}$ converges in distribution to the stopped solution $X_{t \land \tau}$ of the above SDE.
\end{prop}


The proof of \cref{prop:main_conv} is provided in \cref{sec:proof_main_conv}. We use classical results on the numerical approximations of SDEs. \cref{prop:main_conv} shows that the infinite-depth limit of finite-width ResNet (\cref{eq:resnet}) has a similar behaviour to the solution of the SDE given in \cref{eq:main_sde}. In this limit, $Y_{\lfloor t L\rfloor}$ converges in distribution to $X_t$. Hence, properties of the solutions of \cref{eq:main_sde} should theoretically be `shared' by the pre-activations $Y_{\lfloor t L\rfloor}$ when the depth is large. For the rest of the paper, we study some properties of the solutions of \cref{eq:main_sde}. This requires the definition of filtered probability spaces which we omit here. All the technical details are provided in \cref{appendix:stochastic_calculus}. We compare the theoretical findings with empirical results obtained by simulating the pre/post-activations of the original network \cref{eq:resnet}. We refer to $X_t$, the solution of \cref{eq:main_sde}, by the \emph{infinite-depth network}. 

The distribution of $X_1$ (the last layer in the infinite-depth limit) is generally intractable, unlike in the infinite-width-then-infinite-depth limit (Gaussian, \cite{hayou21stable}) or joint infinite-depth-and-width limit (involves a log-normal distribution in the case of an MLP architecture, \cite{li21loggaussian}). Intuitively, one should not expect a universal behaviour (e.g. the Gaussian behaviour in the infinite-width case) of the solution of \cref{eq:main_sde} as this latter is highly sensitive to the choice of the activation function, and different activation functions might yield completely different distributions of $X_1$. We demonstrate this in the next section by showing that we can recover closed-form distributions by carefully choosing the activation function. The main ingredient is the use of \ito's lemma. See \cref{appendix:stochastic_calculus} for more details.

\section{Different behaviours depending on the activation function}
In this section, we restrict our analysis to a width-$1$ ResNet with one-dimensional inputs, where each layer consists of a single neuron, i.e. $d = n =1$. In this case, the process $(X_{t})_{0 \leq t \leq 1}$ is one-dimensional and is solution of the following SDE
$$
dX_t = |\phi(X_t)| dB_t, \quad X_0 = W_{in} x.
$$
We can get rid of the absolute value in the equation above since the process $X_t$ has the same distribution as $\tilde{X}_t$, the solution of the SDE
$
d\tilde{X}_t = \phi(\tilde{X}_t) dB_t
$.
The intuition behind this is that the infinitesimal random variable `$dB_t$' is Gaussian distributed with zero mean and variance $dt$. Hence, it is a symmetric random variable and can absorb the sign of $\phi(X_t)$. The rigorous justification of this fact is provided in \cref{thm:same_law} in the Appendix. Hereafter in this section, we consider the process $X$,  solution of the SDE
$$
dX_t = \phi(X_t) dB_t, \quad X_0 = W_{in} x.
$$
Given a function $g \in \mathcal{C}^2(\reals)$\footnote{Here $\mathcal{C}^2(\reals)$ refers to the vector space of functions $g : \reals \to \reals$ that are twice differentiable and their second derivatives are continuous.}, we use \ito's lemma (\cref{lemma:ito} in the appendix) to derive the dynamics of the process $g(X_t)$. We obtain,
\begin{equation}\label{eq:ito_lemma_1d}
\begin{aligned}
d g(X_t)
= \underbrace{ \phi(X_t) g'(X_t)}_{ \textcolor{niceblue}{\sigma(X_t)}} dB_t + \underbrace{\frac{1}{2} \phi(X_t)^2 g''(X_t)}_{ \,\textcolor{nicered}{\mu(X_t)}} dt.
\end{aligned}
\end{equation}
In financial mathematics nomenclature, the function $\mu$ is called the \emph{drift} and $\sigma$ is called the \emph{volatility} of the diffusion process. \ito's lemma is a valuable tool in stochastic calculus and is often used to transform and simplify SDEs to better understand their properties. It can also be used to find candidate functions $g$ and activation functions $\phi$ such that the SDE
\cref{eq:ito_lemma_1d} admits solutions with known distributions, which yields a closed-form distribution for $X_t$. We consecrate the rest of this section to this purpose.

\subsection{ReLU activation}

ReLU is a piece-wise linear activation function. Let us first deal with the simpler case of linear activation functions. In the next result, we show that linear activation functions yield log-normal distributions. In this case, the process $X_t$ follows the Geometric Brownian motion dynamics. Later in this section, we show that this result can be adapted to the case of the ReLU activation function given by $\phi(x) = \max(x,0)$.

\begin{prop}\label{prop:gbm}
Let $x \in \reals$ such that $x \neq 0$. Consider a linear activation function
$
\phi(y) = \alpha y + \beta,
$
where $\alpha> 0, \beta \in \reals$ are constants. Let $\sigma>0$ and define the function $g$ by 
$
g(y) = (\alpha y + \beta)^{\gamma},
$
where $\gamma =\sigma  \alpha^{-1}$. Consider the stochastic process $X_t$ defined by 
$$
dX_t = |\phi(X_t)| dB_t, \quad X_0 = W_{in} x.
$$

Then, the process $g(X_t)$ is a solution of the SDE
$$
d g(X_t) = a g(X_t) dt + \sigma g(X_t) dB_t,
$$
where $a = \frac{1}{2} \sigma^2 \gamma^{-1} (\gamma -1)$. As a result, we have that for all $t \in [0,1]$,
$$
g(X_t) \sim g(X_0) \exp\left( \left(a - \frac{1}{2}\sigma^2\right)t +  \sigma B_t\right).
$$

\end{prop}

The proof of \cref{prop:gbm} is provided in \cref{sec:gbm}, and consists of using \ito lemma and solving a differential equation. When the activation function is ReLU, we still obtain a log-normal distribution conditionally on the event that the initial value $X_0$ is positive.

\begin{prop}\label{prop:relu_gbm_1d}
Let $x \in \reals$ such that $x \neq 0$, and let $\phi$ be the ReLU activation function given by $\phi(z) = \max(z,0)$ for all $z\in \reals$.
 Consider the stochastic process $X_t$ defined by 
$$
dX_t = \phi(X_t) dB_t, \quad X_0 = W_{in} x.
$$

Then, the process $X$ is a mixture of a Geometric Brownian motion and a constant process. More precisely, we have for all $t \in [0,1]$

$$
X_t \sim \ind_{\{X_0 >0\}} \,  X_0 \exp\left( -\frac{1}{2} t  + B_t \right) + \ind_{\{X_0 \leq 0\}} X_0.
$$
Hence, given a fixed $X_0>0$, the process $X$ is a Geometric Brownian motion.

\end{prop}
\begin{figure}[h!]
    \centering
    \begin{subfigure}[b]{0.49\textwidth}
         \centering
         \includegraphics[width=\textwidth]{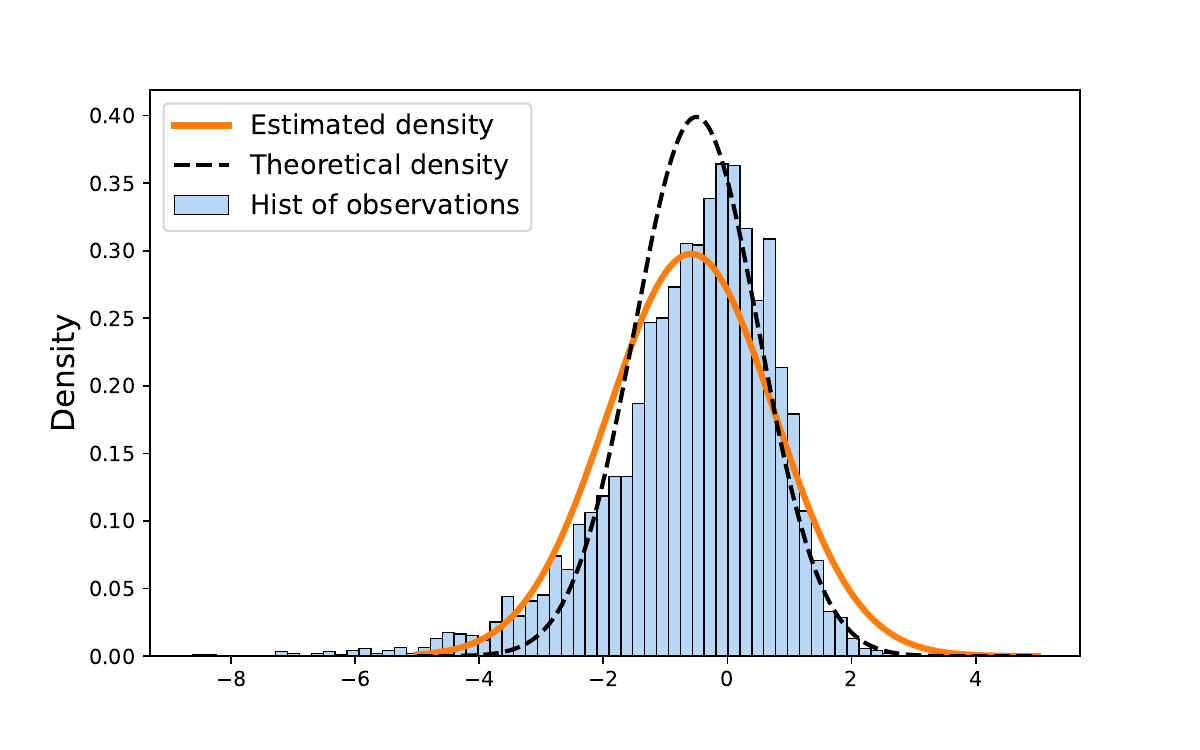}
         \caption{Distribution of $\log(Y_L)$ with $L=5$}
         \label{fig:gbm_depth5}
     \end{subfigure}
     \hfill
     \begin{subfigure}[b]{0.49\textwidth}
         \centering
         \includegraphics[width=\textwidth]{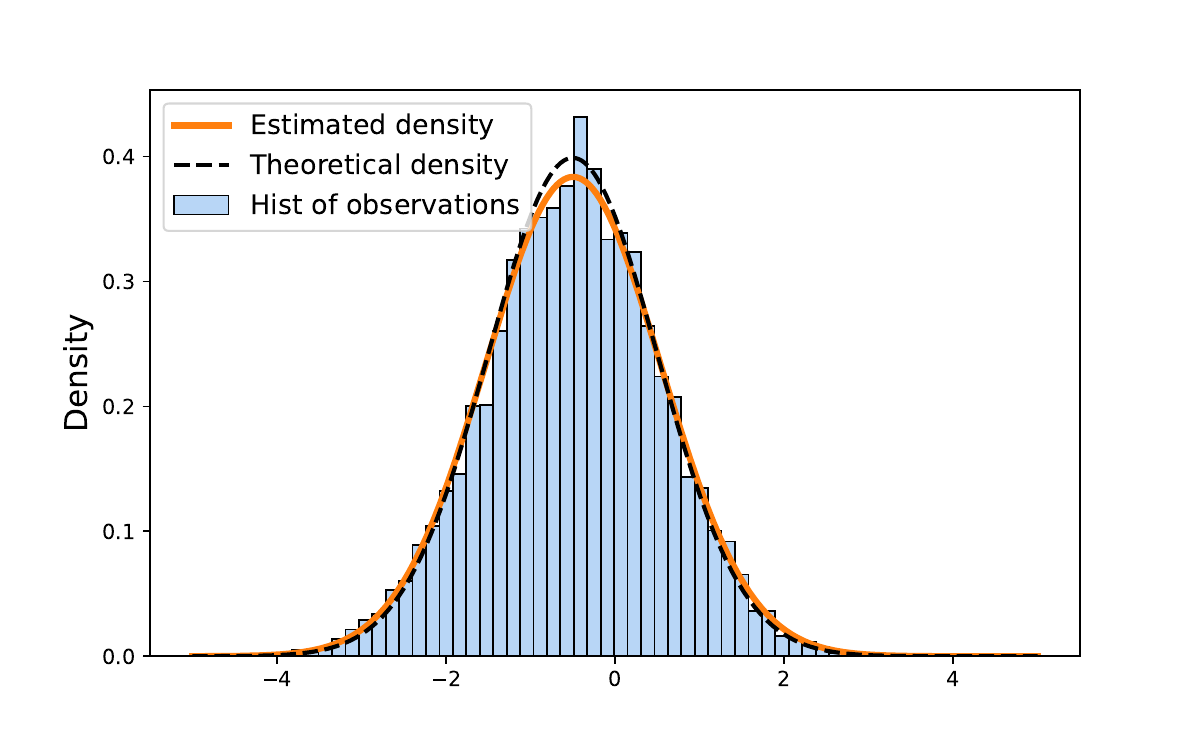}
         \caption{Distribution of $\log(Y_L)$ with $L=50$}
         \label{fig:gbm_depth50}
     \end{subfigure}
    \begin{subfigure}[b]{0.49\textwidth}
         \centering
         \includegraphics[width=\textwidth]{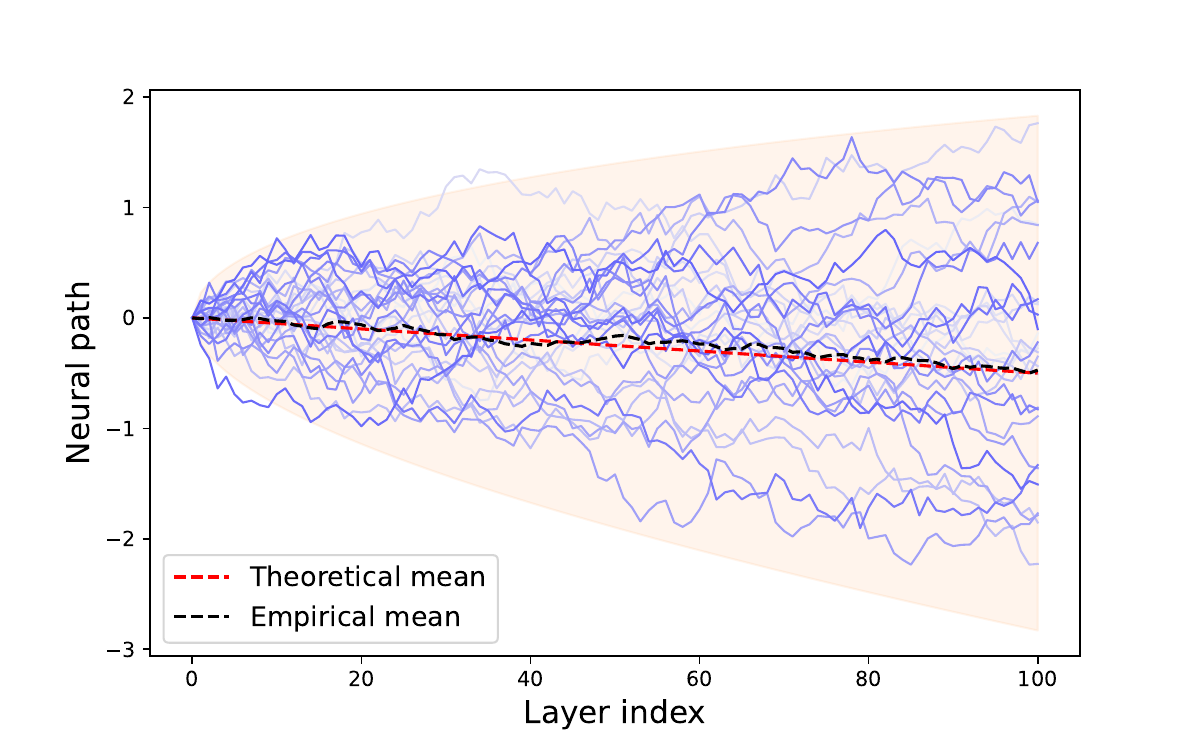}
         \caption{Neural path $(\log(Y_l)_{1 \leq l \leq L}$ with $L=100$}
         \label{fig:gbm_neural_path_log_depth100}
     \end{subfigure}
     \hfill
     \begin{subfigure}[b]{0.49\textwidth}
         \centering
         \includegraphics[width=\textwidth]{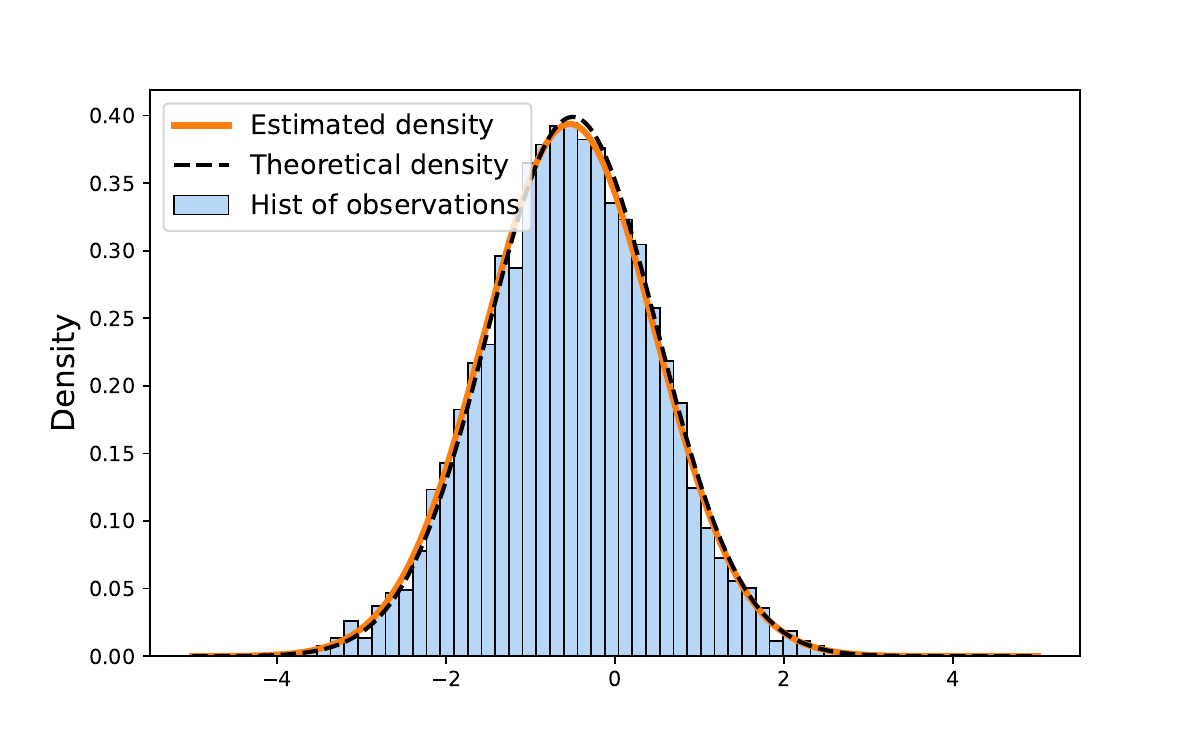}
         \caption{Distribution of $\log(Y_L)$ with $L=100$}
         \label{fig:gbm_depth50}
     \end{subfigure}     
    \begin{subfigure}[b]{0.49\textwidth}
         \centering
         \includegraphics[width=\textwidth]{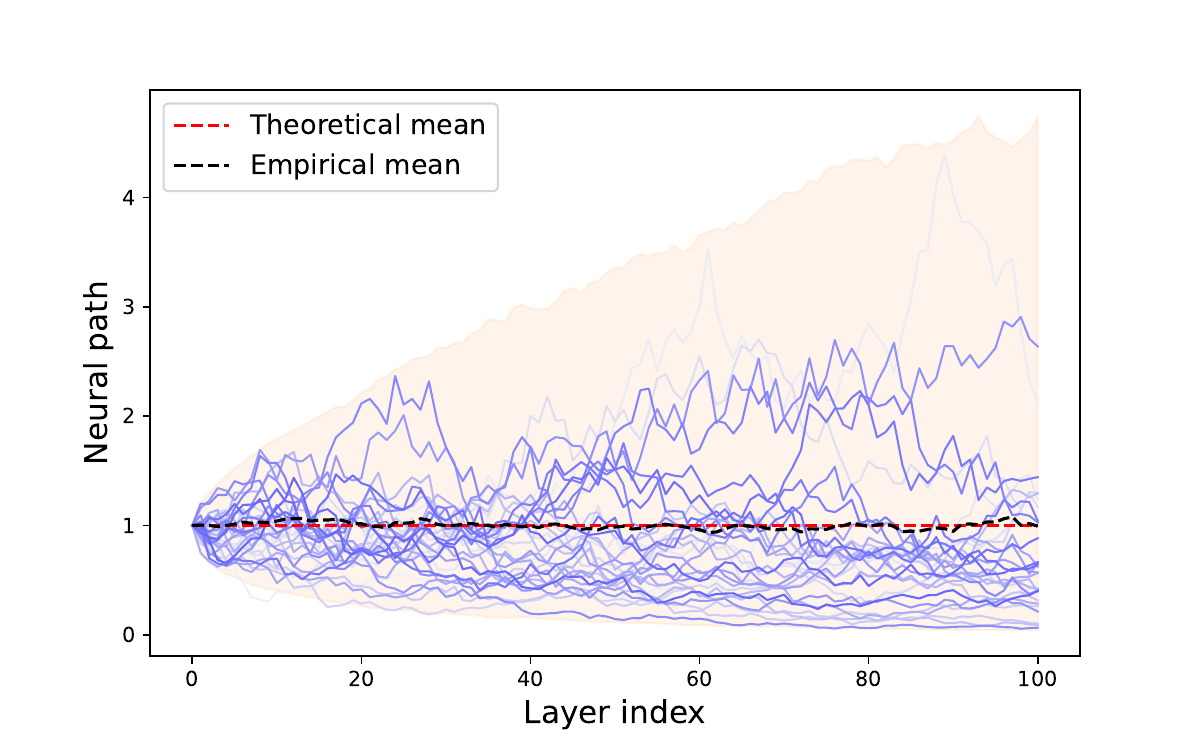}
         \caption{Neural path $(Y_l)_{1 \leq l \leq L}$ with $L=100$}
         \label{fig:gbm_neural_path_depth100}
     \end{subfigure}
     \hfill
     \begin{subfigure}[b]{0.49\textwidth}
         \centering
         \includegraphics[width=\textwidth]{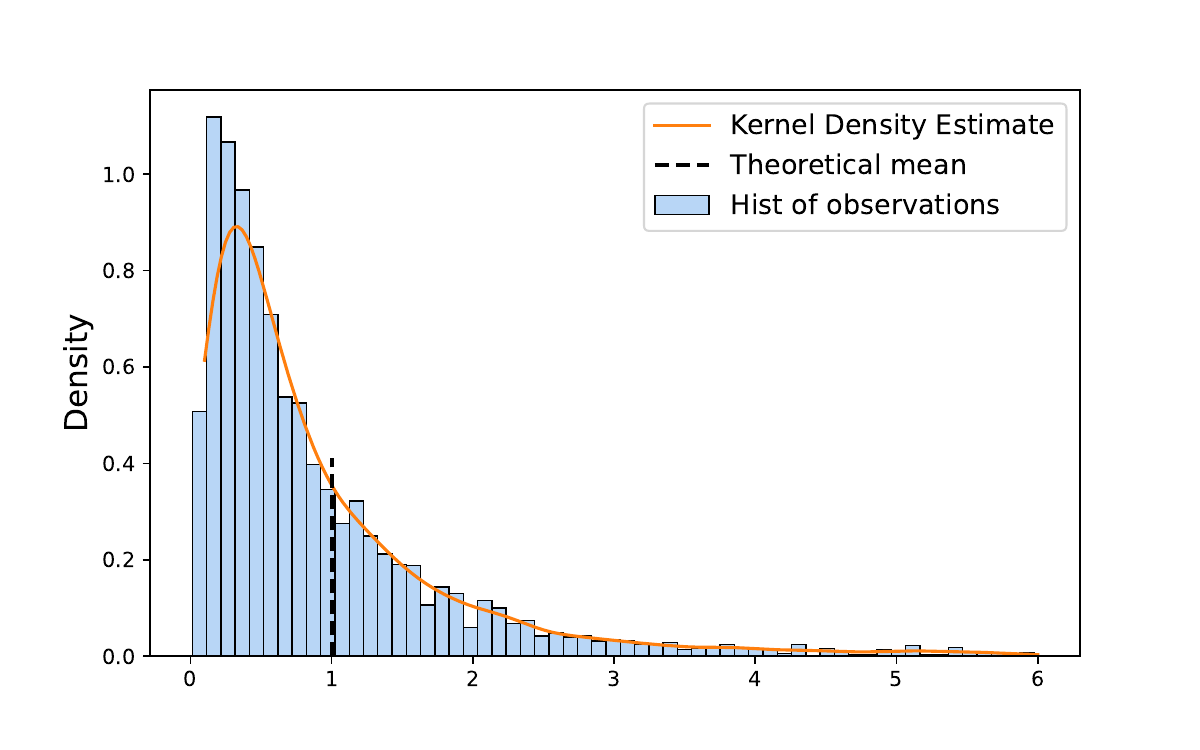}
         \caption{Distribution of $Y_L$ with $L=100$}
         \label{fig:gbm_x1_depth100}
     \end{subfigure}    
    \caption{Empirical verification of \cref{prop:gbm}. \textbf{(a), (b), (d)} Histograms of $\log(Y_L)$ and based on $N=5000$ simulations for depths $L \in \{5, 50, 100\}$ with $Y_0 = 1$. Estimated density (Gaussian kernel estimate) and theoretical density (Gaussian) are  illustrated on the same graphs. 
    \textbf{(c), (e)} 30 Simulations of the sequence $(\log(Y_l))_{ l\leq L}$ (c) and the sequence $(Y_l)_{ l\leq L}$ (e). We call such sequences Neural paths. The results are reported for depth $L=100$, with $Y_0 = 1$, $\phi$ being the ReLU activation. The theoretical mean of $\log(Y_l)$ is given by $m(l) = -\frac{l}{2 L}$ and that of $Y_l$ is equal to $Y_0=1$. We also illustrate the $99\%$ confidence intervals, based on the theoretical prediction for $\log(Y_l)$ (\cref{prop:gbm}), and the empirical Quantiles for $Y_l$. \textbf{ (f)} Histogram of $Y_L$ based on $N=5000$ simulations for depth $L=100$. }
    \label{fig:gbm_figs}
\end{figure}
The proof of \cref{prop:relu_gbm_1d} is provided in \cref{sec:relu_1d}. We show that conditionally on $X_0>0$, with probability $1$, the process $X_t$ is positive for all $t \in [0,1]$\footnote{In \cref{sec:relu_1d}, we show that the stopping $\tau = \inf \{t \geq 0: \textrm{ s.t. } X_t \leq 0 \}$ is infinite almost surely, which is stronger that what we need. This is a classic result in stochastic calculus.}. When $X_t>0$, the ReLU activation is just the identity function, which justifies the similarity between this result and the one obtained with linear activations (\cref{prop:gbm}). Conversely, if $X_0<0$, the process is constant equal to $X_0$ since the updates `$dX_t$' are equal to zero in this case. A rigorous justification of this is given for general width $n$ later in the paper (\cref{lemma:constant_after_hit}).  An empirical verification of \cref{prop:gbm} is provided in \cref{fig:gbm_figs} where we compare the theoretical results to simulations of the \emph{neural paths} $(Y_{l})_{1 \leq l \leq L}$ and $(\log(Y_{l}))_{1 \leq l \leq L}$  from the original (finite-depth) ResNet given by \cref{eq:resnet}. We observe an excellent match with theoretical predictions for depths $L =50$ and $L=100$. In the case of a small depth ($L=5$), the theoretical distribution does not fit well the empirical one (obtained by simulations), which is expected since the dynamics of $X$ describe (only) the infinite-depth limit of the ResNet.
More figures are provided in \cref{sec:additional_experiments}.\\
\textit{Remark:} notice that the log-normal behaviour is a result of the fact that we only consider the case $n=1$ (width one). Indeed, the single neuron case forces ReLU to act like a linear activation when $X_0 > 0$, and like a `zero' activation when $X_0 \leq 0$. For general width $n \geq 1$, such behaviour does not hold in general, and usually some coordinates of $X_t$ will be negative while others are non-negative, which implies that the volatility term $\|\phi(X_t)\|$ has non-trivial dependence on $X_t$. We discuss this in more details in \cref{sec:general_width}. In the next section, we illustrate a case of an exotic (non-standard) activation function that yields a completely different closed-form distribution of $X_t$.
\subsection{Exotic activation}
The next result shows that with a particular choice of the activation function $\phi$ and mapping $g$, the stochastic process $g(X_t)$ is the solution of well-known type of SDEs known as the Ornstein-Uhlenbeck SDEs. In this case, the activation function is non-standard and involves the inverse of the imaginary error function, a variant of the error function.
\begin{wrapfigure}{r}{0.3\textwidth}
  \begin{center}
    \includegraphics[width=0.33\textwidth]{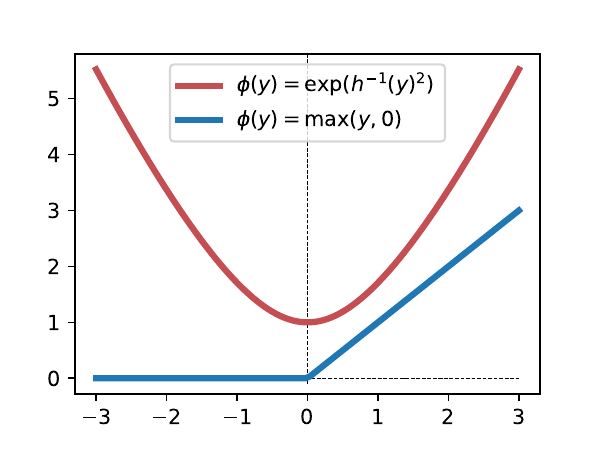}
  \end{center}
  \vspace{-0.5cm}
  \caption{Exotic Activation.}\label{fig:exotic_act}
  \vspace{-0.5cm}
\end{wrapfigure}
\begin{prop}[Ornstein-Uhlenbeck neural networks]\label{prop:ou_nn}
Let $x \in \reals$ such that $x \neq 0$. Consider the following activation function $\phi$
$$
\phi(y) = \exp(h^{-1}(\alpha y + \beta)^2),
$$
where $\alpha, \beta \in \reals$ are constants and $h^{-1}$ is the inverse function of the imaginary error function given by $h(z) = \frac{2}{\sqrt{\pi}} \int_{0}^z e^{t^2} dt$\footnote{Although the name might be misleading, the imaginary error function is real valued, and it has a an inverse $h^{-1}$ that is continuous and increasing.}. Let $g$ be the function defined by 
$$
g(y) = \alpha \sqrt{\pi} h^{-1}(\alpha y + \beta).
$$
Consider the stochastic process $X_t$ defined by\footnote{in \cref{sec:ou_process}, we show that the activation function $\phi$ is only locally Lipschitz. Hence, the solution of this SDE exists only in the local sense and the convergence in distribution of $Y_{\lfloor t L\rfloor}$ to $X_t$ is also in the local sense (\cref{prop:main_conv}). However, by continuity of the Brownian path, the stopping times $\tau^L$ and $\tau$ diverge almost surely when $r$ goes to infinity. Therefore, the conclusion of \cref{prop:ou_nn} remains true for all $t \in [0,1]$. Technical details are provided in \cref{sec:ou_process}. } 
$$
dX_t = |\phi(X_t)| dB_t, \quad X_0 = W_{in} x.
$$
Then, the stochastic process $g(X_t)$ follows the Ornstein-Uhlenbeck dynamics on $(0,1]$ given by 
$$
dg(X_t) = a g(X_t) dt + 2 a dB_t,\quad g(X_0) = g(W_{in}x),
$$
where $a = \frac{\pi \alpha^2}{4}$. As a result, conditionally on $X_0$ (fixed $X_0$), we have that for all $t \in [0,1]$,
$$
g(X_t) \sim \mathcal{N}\left( g(X_0) e^{-a t}, \frac{\pi}{2} ( 1 - e^{-2a t})\right),
$$
and the process $X_t$ is distributed as $X_t \sim \alpha^{-1} (h(\alpha^{-1} \pi^{-1/2} \mathcal{N}\left( g(X_0) e^{-a t},  \frac{\pi}{2} ( 1 - e^{-2a t})\right))- \beta)$. 
\end{prop}

\cref{fig:exotic_act} shows the graph of the activation function $\phi(y) = \exp(h^{-1}(y)^2)$ mentioned in \cref{prop:ou_nn} with $\alpha = 1$ and $\beta =0$. With this choice of the activation function, the infinite-depth network output $X_1$ has the distribution $g^{-1}\left( \mathcal{N}\left( g(X_0) e^{-a t}, 2 ( 1 + e^{-2a t})\right)\right)$ (conditionally on $X_0$), where $g$ is given in the statement of the proposition. This distribution, although easy to simulate, is different from both the Gaussian distribution that we obtain in the infinite-width limit and the log-normal distribution associated with ReLU activation. This confirms that not only do neural networks exhibit completely different behaviours when the ratio depth-to-width is large, but in this case, that their behaviour is very sensitive to the choice of the activation function.

\begin{figure}[h!]
    \centering
    \begin{subfigure}[b]{0.49\textwidth}
         \centering
         \includegraphics[width=\textwidth]{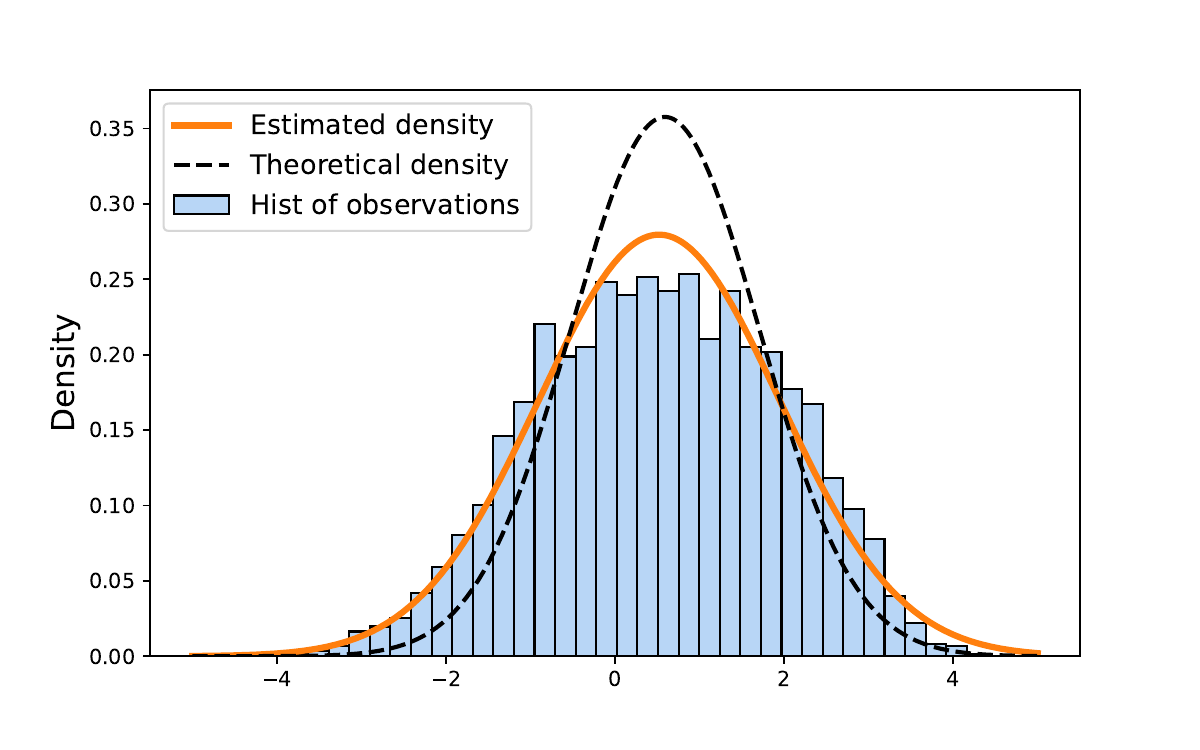}
         \caption{Distribution of $g(Y_L)$ with $L=5$}
         \label{fig:gbm_depth5}
     \end{subfigure}
     \hfill
     \begin{subfigure}[b]{0.49\textwidth}
         \centering
         \includegraphics[width=\textwidth]{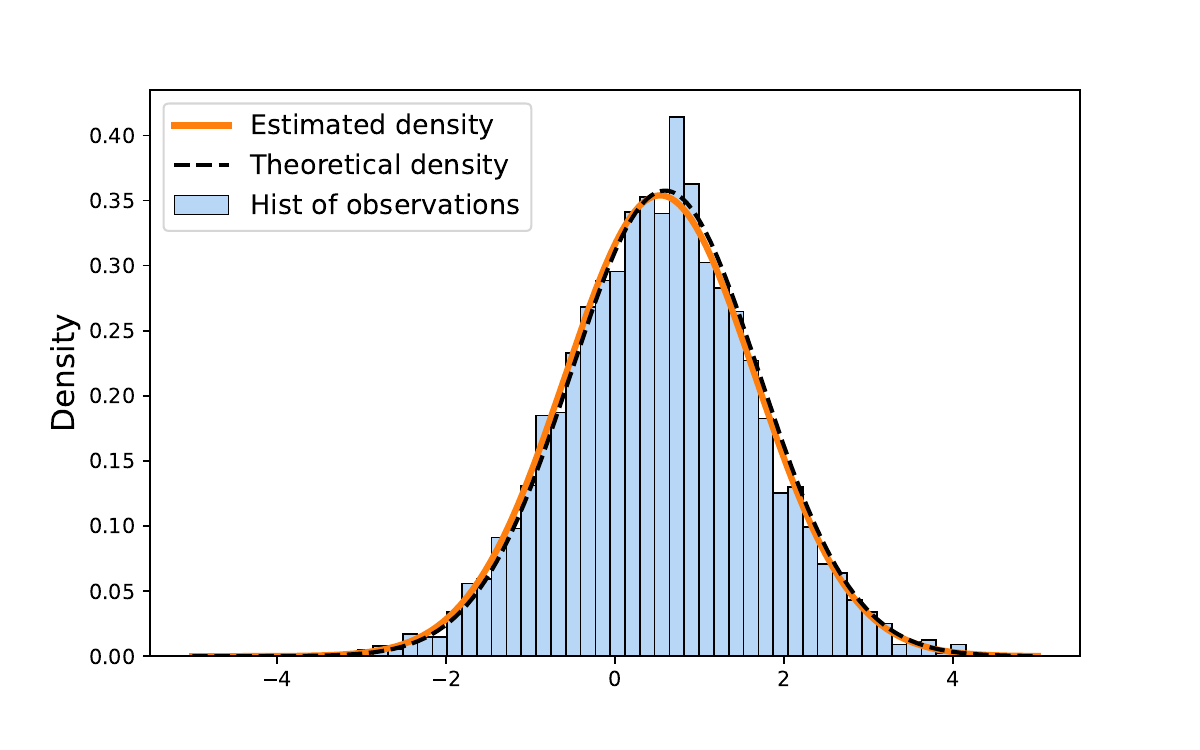}
         \caption{Distribution of $g(Y_L)$ with $L=50$}
         \label{fig:gbm_depth50}
     \end{subfigure}
    \begin{subfigure}[b]{0.49\textwidth}
         \centering
         \includegraphics[width=\textwidth]{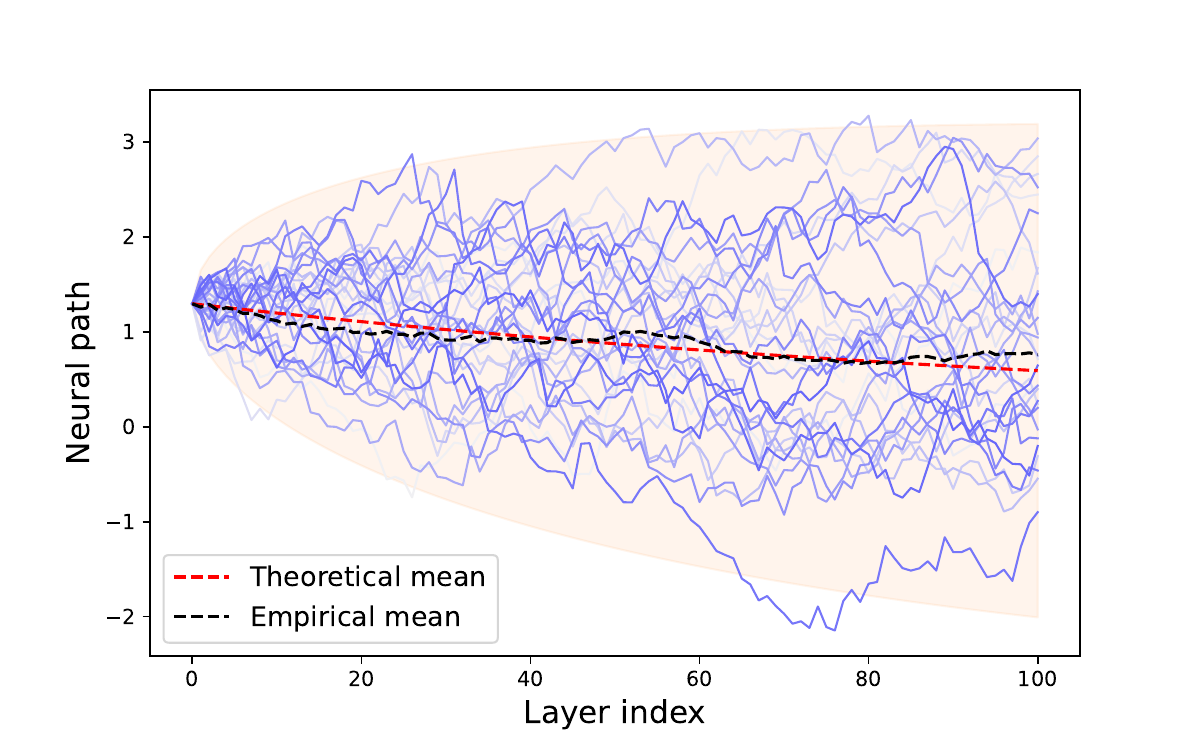}
         \caption{Neural path $(g(Y_l)_{1 \leq l \leq L}$ with $L=100$}
         \label{fig:gbm_neural_path_log_depth100}
     \end{subfigure}
     \hfill
     \begin{subfigure}[b]{0.49\textwidth}
         \centering
         \includegraphics[width=\textwidth]{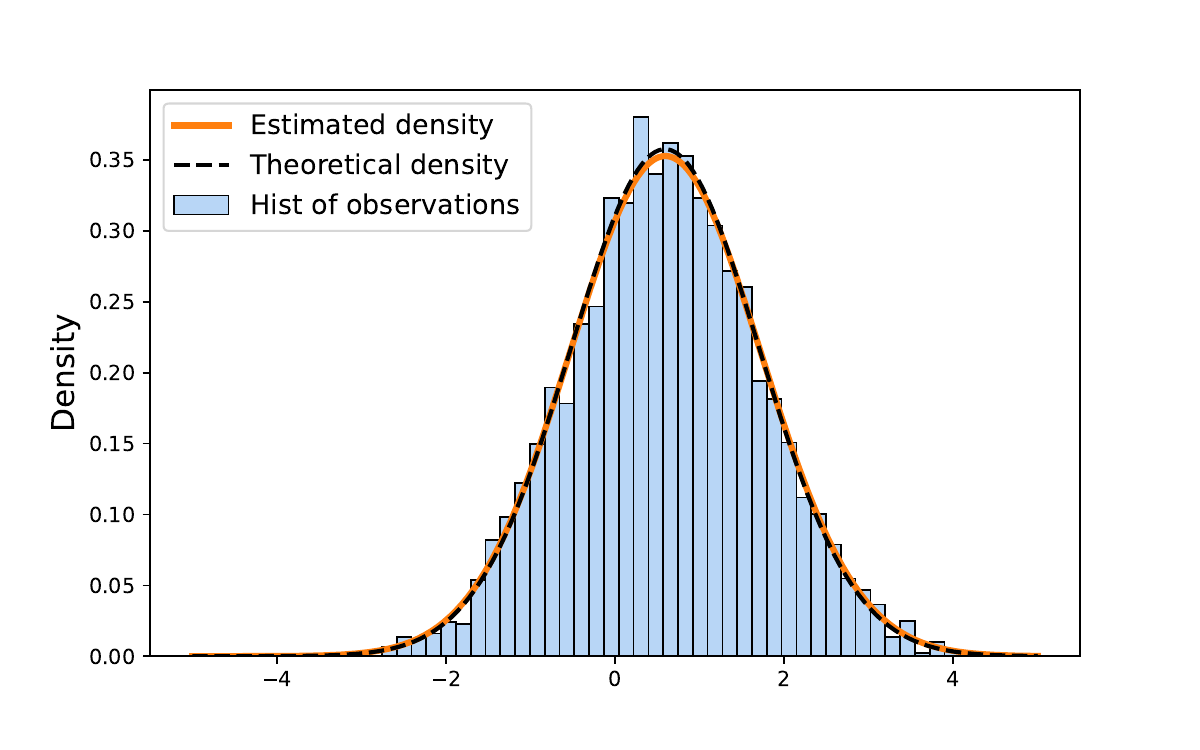}
         \caption{Distribution of $g(Y_L)$ with $L=100$}
         \label{fig:gbm_depth50}
     \end{subfigure}     
    \begin{subfigure}[b]{0.49\textwidth}
         \centering
         \includegraphics[width=\textwidth]{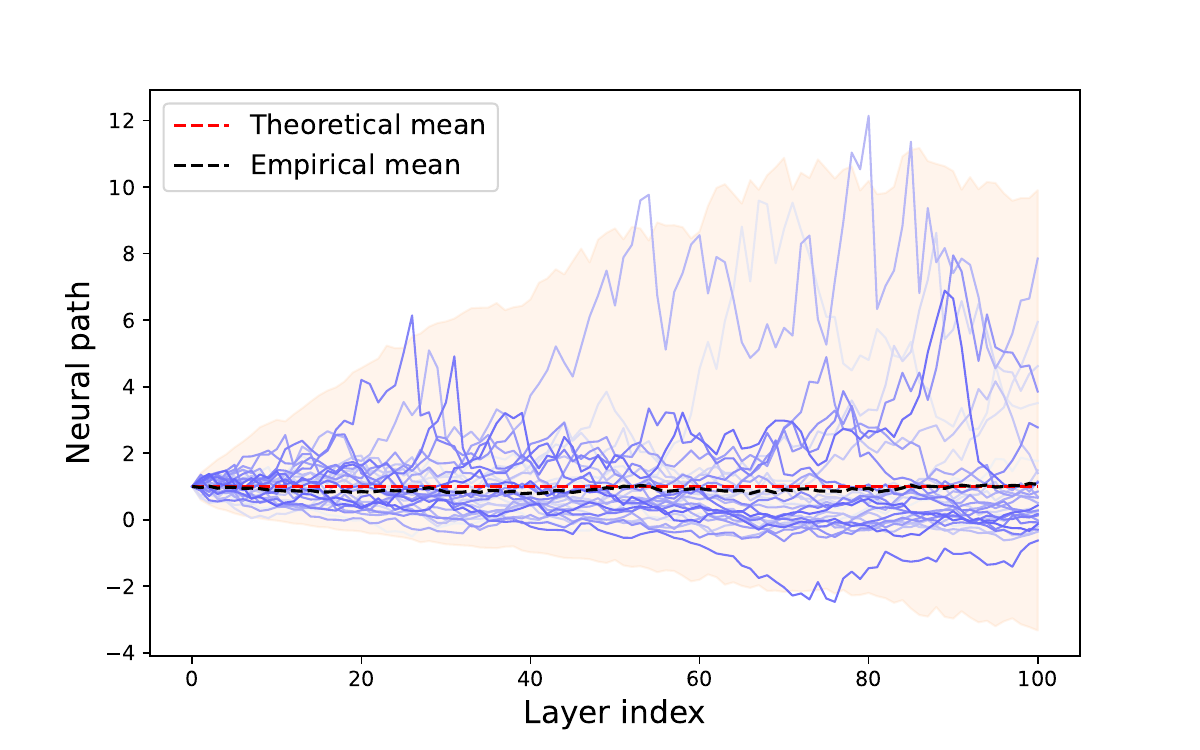}
         \caption{Neural path $(Y_l)_{1 \leq l \leq L}$ with $L=100$}
         \label{fig:gbm_neural_path_depth100}
     \end{subfigure}
     \hfill
     \begin{subfigure}[b]{0.49\textwidth}
         \centering
         \includegraphics[width=\textwidth]{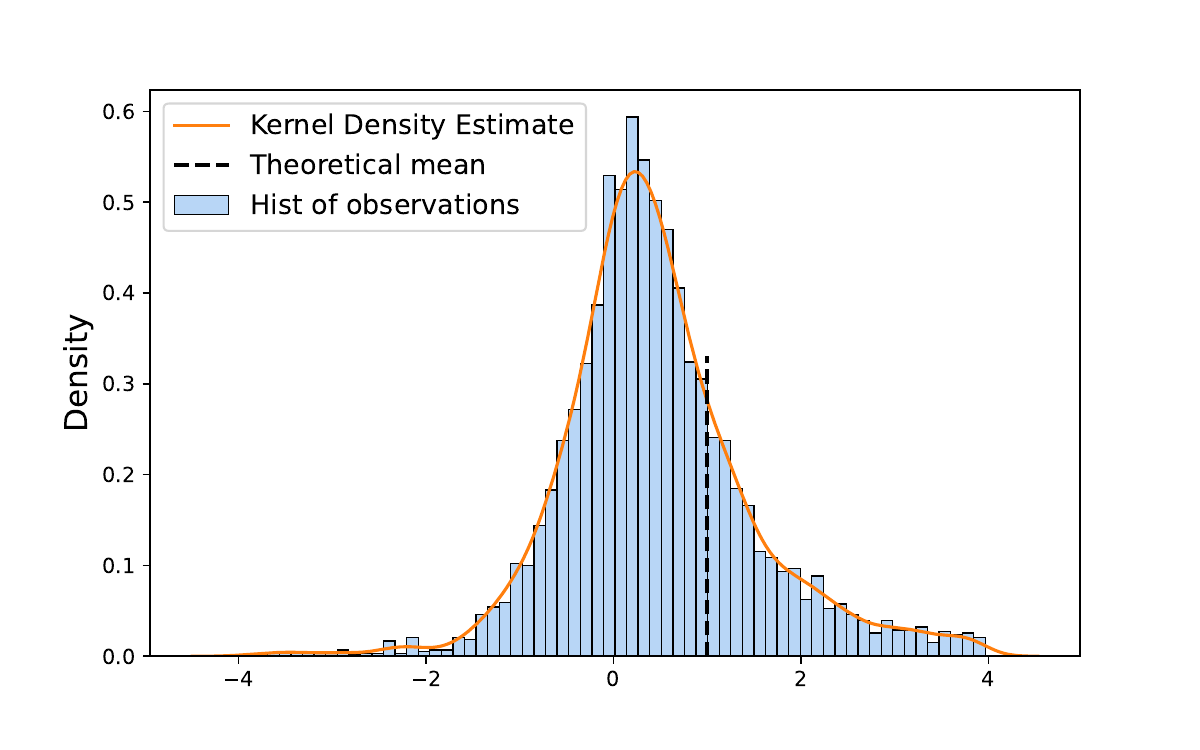}
         \caption{Distribution of $Y_L$ with $L=100$}
         \label{fig:gbm_x1_depth100}
     \end{subfigure}    
    \caption{Empirical verification of \cref{prop:ou_nn}. \textbf{(a), (b), (d)} Histograms of $g(Y_L)$ based on $N=5000$ simulations for depths $L \in \{5, 50, 100\}$ with $Y_0 = 1$. Estimated density (Gaussian kernel estimate) and theoretical density (Gaussian) are  illustrated on the same graphs. 
    \textbf{(c), (e)} 30 Simulations of the neural paths $(g(Y_l))_{ l\leq L}$ (c) and $(Y_l)_{ l\leq L}$ (e). The results are reported for depth $L=100$, with $Y_0 = 1$, $\phi$ is given un \cref{prop:ou_nn}. The theoretical mean of $g(Y_l)$ (conditionally on $Y_0$) is approximated by $m(l) = g(Y_0) e^{-\frac{\pi l}{3 L}}$ and that of $Y_l$ is equal to $Y_0=1$. We also illustrate the $99\%$ confidence intervals, based on the theoretical prediction for $g(Y_l)$ (\cref{prop:gbm}), and the empirical Quantiles for $Y_l$. \textbf{ (f)} Histogram of $Y_L$ based on $N=5000$ simulations for depth $L=100$. }
    \label{fig:ou_figs}
\end{figure}
The results of \cref{prop:ou_nn} are empirically confirmed in \cref{fig:ou_figs}. The original ResNet given by \cref{eq:main_sde} with depth $L=100$ exhibit very similar behaviour to that of the SDE.

\section{General width $n\geq 1$}\label{sec:general_width}
Let $n\geq 1$ and $x \in \reals^d$ such that $x \neq 0$. Consider the process $X$ given by the SDE

\begin{equation}\label{eq:main_general_n}
 dX_t = \frac{1}{\sqrt{n}} \|\phi(X_t)\| dB_t, \quad X_0 = W_{in} x,   
\end{equation}
where $\phi$ is the activation function, and $B$ is an $n$-dimensional Brownian motion, independent from $W_{in}$.  Intuitively, if for some $s$, $\|\phi(X_s)\| = 0$, then for all $t \geq s$, $X_t = X_s$ since the increments '$dX_t$' are all zero for $t\geq s$. This holds for any choice of the activation function $\phi$, provided that the process $X$ exists, i.e. the SDE has a unique solution. We summarize this in the next lemma.

\begin{lemma}[Collapse]\label{lemma:constant_after_hit}
Let $x \in \reals^d$ such that $x \neq 0$, and $\phi : \reals \to \reals$ be a Lipschitz function. Let $X$ be the solution of the SDE given by \cref{eq:main_general_n}. Assume that for some $s \geq 0$, $\phi(X_s) = 0$. Then, for all $t\geq s$, $X_t = X_s$, almost surely.
\end{lemma}
\cref{lemma:constant_after_hit} is a particular case of \cref{lemma:zero_after_hit} in the Appendix. The proof consists of using the uniqueness of the solution of \cref{eq:main_general_n} when the volatility term is Lipschitz. This result is trivial in the finite depth case (\cref{eq:resnet}). When there exists $s$ such that $\phi(X_s)=0$, the process $X$ becomes constant (equal to $X_s$) for all $t \geq s$ (almost surely). We call this phenomenon \emph{process collapse}. In the case of finite-depth networks (\cref{eq:resnet}), we call the same phenomenon \emph{network collapse}. Understanding when, and whether, such event occurs is useful since it has significant implications on the the large depth behaviour of neural networks. Indeed, if such event occurs, it would mean that increasing depth has no effect on the network output after some time $s$ (or approximately, after layer index $\lfloor s L \rfloor$). In the next result, we show that under mild conditions on the activation function, process collapse is a zero-probability event. 

\subsection{Network collapse}
The next result gives (mild) sufficient conditions on the activation function so that the process $X$ almost surely does not collapse. In the proof, we use \ito's lemma in the multi-dimensional case, which states that for any function $g : \reals^n \to \reals $ that is $ \mathcal{C}^2(\reals^n)$, we have that 
\begin{align*}
d g(X_t) = \nabla g(X_t)^\top dX_t + \frac{1}{2n} \|\phi(X_t)\|^2 \Tr\left[\nabla^2 g(X_t)\right].
\end{align*}

\begin{lemma}\label{lemma:tau_general_zeta}
Let $x \in \reals^d$ such that $x \neq 0$, and consider the stochastic process $X$ given by the following SDE
$$
dX_t =\frac{1}{\sqrt{n}} \|\phi(X_t)\| dB_t \, ,\quad  t\in [0,\infty), \quad X_0 = W_{in} x, 
$$
where $\phi(z) : \reals \to \reals$ is Lipschitz, injective, $\mathcal{C}^2(\reals)$ and satisfies $\phi(0)=0$, and $\phi'$ and $\phi'' \phi$ are bounded on $\reals$, and $(B_t)_{t \geq 0}$ is an $n$-dimensional Brownian motion independent from $W_{in} \sim \normal(0, d^{-1} I)$. Let $\tau$ be the stopping time given by 
$$
\tau = \min \{ t\geq 0 : \phi(X_t)\ = 0\}.
$$
Then, we have that 
$$
\mathbb{P}\left(\tau = \infty \right) = 1.
$$
\end{lemma}

The proof of \cref{lemma:tau_general_zeta} is provided in \cref{sec:proofs_lemmas_tau}. Many standard activation functions satisfy the conditions of \cref{lemma:tau_general_zeta}. Examples include Hyperbolic Tangent $\textrm{Tanh}(z) = \frac{e^{2z}  - 1}{e^{2z} + 1}$, and smooth versions of ReLU activation such as GeLU given by $\phi_{GeLU}(z) = z \Psi(z)$ where $\Psi$ is the cumulative distribution function of the standard Gaussian variable, and Swish (or SiLU) given by $\phi_{Swish}(z) = z h(z)$ where $h(z) = (1 + e^{-z})^{-1}$ is the Sigmoid function.
The result of \cref{lemma:tau_general_zeta} can be extended to the case when $\phi$ is the ReLU function with miner changes.

\begin{lemma}\label{lemma:tau_genera_n_relu}
Consider the stochastic process \eqref{eq:main_sde} given by the SDE
$$
dX_t = \frac{1}{\sqrt{n}} \|\phi(X_t)\| dB_t \, ,\quad  t\in [0,\infty), \quad X_0 = W_{in} x, 
$$
where $\phi$ is the ReLU activation function, and $(B_t)_{t \geq 0}$ is an $n$-dimensional Brownian motion independent from $W_{in} \sim \normal(0, d^{-1} I)$. Let $\tau$ be the stopping time given by 
$$
\tau = \min \{ t\geq 0 : \|\phi(X_t)\| = 0\} = \min \{ t\geq 0 : \,  \forall i \in [n], \, X^i_t \leq 0\}.
$$
Then, we have that 
$$
\mathbb{P}\left(\tau = \infty \, \huge| \, \|\phi(X_0)\| > 0\right) = 1.
$$
As a result, we have that 
$$
\mathbb{P}(\tau = \infty) = 1 - 2^{-n}.
$$
\end{lemma}

The proof of \cref{lemma:tau_genera_n_relu} relies on a particular choice of a sequence of functions 
$(\phi_m)_{m \geq 1}$ that approximate the ReLU activation $\phi$. Details are provided in \cref{sec:proofs_lemmas_tau}.

The result of \cref{lemma:tau_genera_n_relu} shows that for all $T>0$, with probability $1$, if there exists $j \in [n]$ such that $X_0^j > 0$, then for all $t \in [0,T]$, there exists a coordinate $i$ such that $X^i_t>0$, which implies that the volatility of the process $X$ given by $\frac{1}{\sqrt{n}} \|\phi(X_t)\|$ does not vanish in finite time $t$. Notably, this implies that for any $t \in [0,1]$, the norm of post-activations given by  $\|\phi(X_t)\|$ does not vanish (with probability 1). This is important as it ensures that the vector $\phi(X_t)$, which represents the post-activations in the infinite-depth network, does not vanish, and therefore the process $X_t$ does not get stuck in 
\begin{wrapfigure}{r}{0.45\textwidth}
  \begin{center}
    \includegraphics[width=0.4\textwidth]{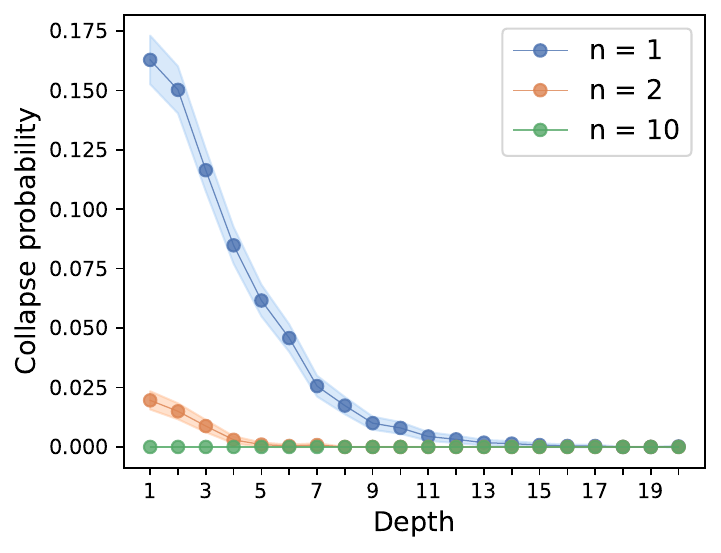}
  \end{center}
  \vspace{-1em}
  \caption{Probability of the event $\{ \exists l \in [L] \textrm{ such that } \phi(Y_l) = 0\}$ (collapse) for varying widths and depths. The probability and the $95\%$ confidence intervals are estimated using $N = 5000$ samples. }
  \label{fig:killed_probs}
\end{wrapfigure}
an absorbent point. The dependence between the coordinates of the process $X_t$ is crucial in this result. In the opposite case where $X_t$ are independent, the event $\{\|\phi(X_t)\|=0\}$ has probability $2^{-n}$. Notice also that this result holds only in the infinite-depth limit. With finite-depth ResNet (\cref{eq:resnet}) with ReLU activation, it is not hard to show that the network collapse event $\{\exists l \in [L], \textrm{ s.t. }\|\phi(Y_{\lfloor t L \rfloor})\|=0\}$ has non-zero probability. However, as the depth increases, the probability of network collapse goes to zero. \cref{fig:killed_probs} shows the probability of network collapse for a finite-width and depth ResNet (\cref{eq:resnet}). As the depth $L$ increases, it becomes unlikely that the network collapses. This is in agreement with our theoretical prediction that the infinite-depth network represented by the process $X_t$ has zero-probability collapse event, conditionally on the fact that $\|\phi(X_0)\|>0$. The probability of neural collapse also decreases with width, which is expected, since it becomes less likely to have all pre-activations non-positive as the width increases.
\subsection{Post-activation norm}
As a result of \cref{lemma:tau_genera_n_relu}, conditionally on $\| \phi(X_0)\| > 0$, we can safely consider manipulating functions that require positiveness such as the logarithm of the norm of the post-activations. In the next result, we show that the norm of the post-activations has a distribution that
\begin{figure}[h]
     \centering
     \begin{subfigure}[b]{0.42\textwidth}
         \centering
         \includegraphics[width=\textwidth]{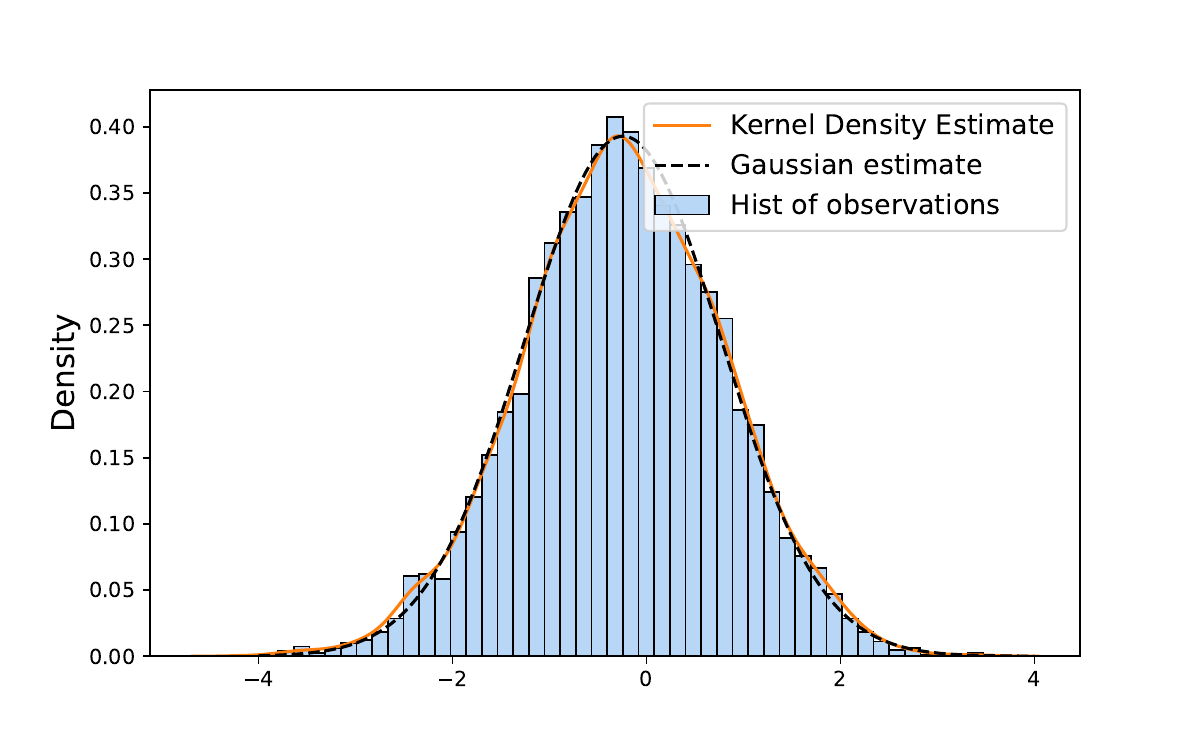}
         \caption{$n=2$}
     \end{subfigure}
     \begin{subfigure}[b]{0.42\textwidth}
         \centering
         \includegraphics[width=\textwidth]{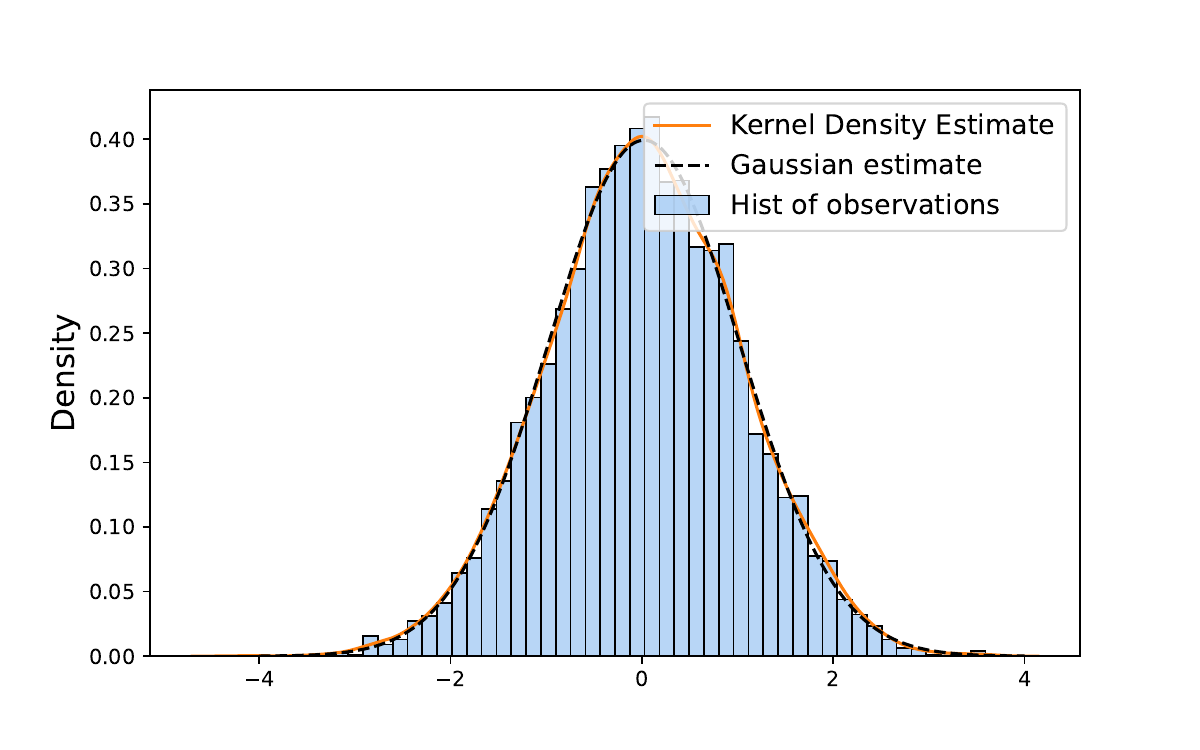}
         \caption{$n=3$}
     \end{subfigure}
     \begin{subfigure}[b]{0.42\textwidth}
         \centering
         \includegraphics[width=\textwidth]{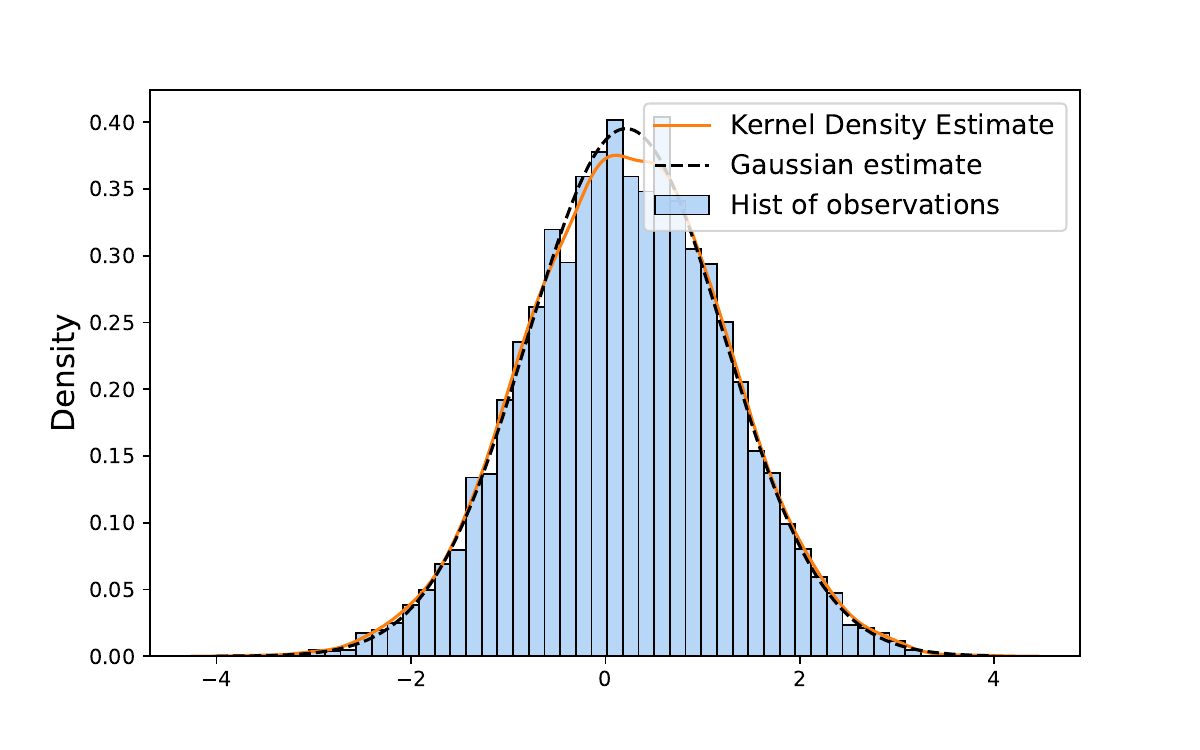}
         \caption{$n=4$}
     \end{subfigure}
     \begin{subfigure}[b]{0.42\textwidth}
         \centering
         \includegraphics[width=\textwidth]{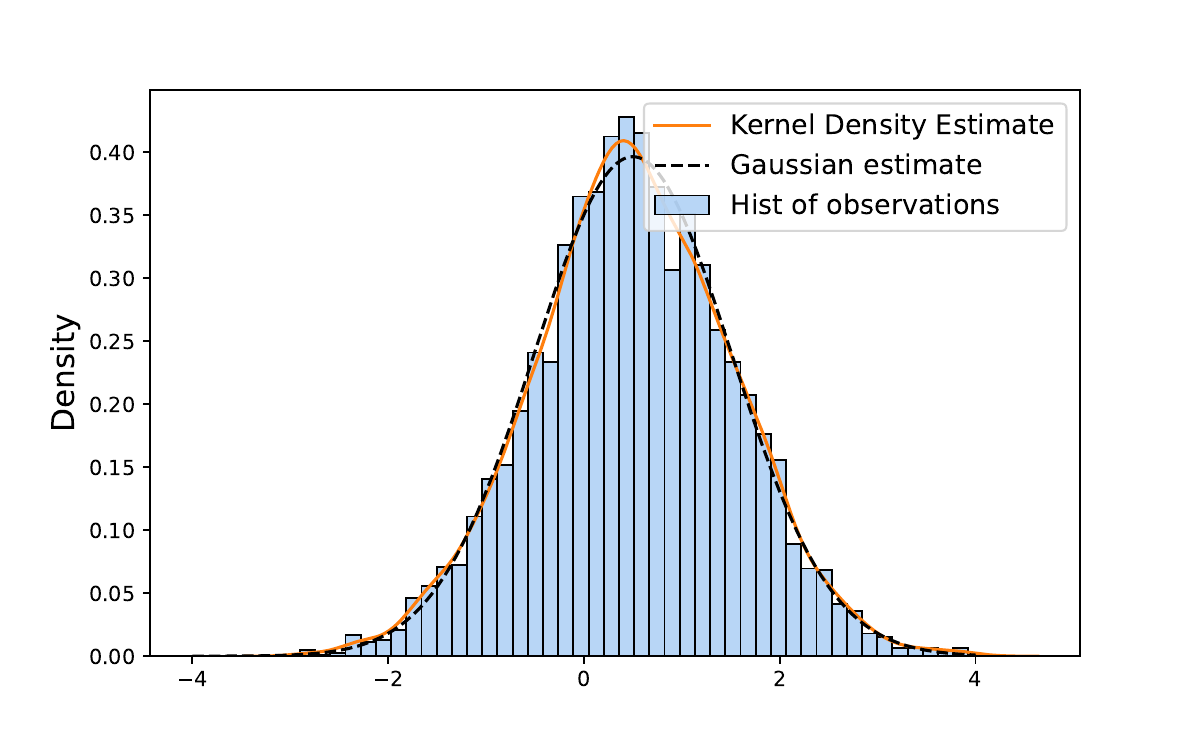}
         \caption{$n=6$}
     \end{subfigure}
     \begin{subfigure}[b]{0.42\textwidth}
         \centering
         \includegraphics[width=\textwidth]{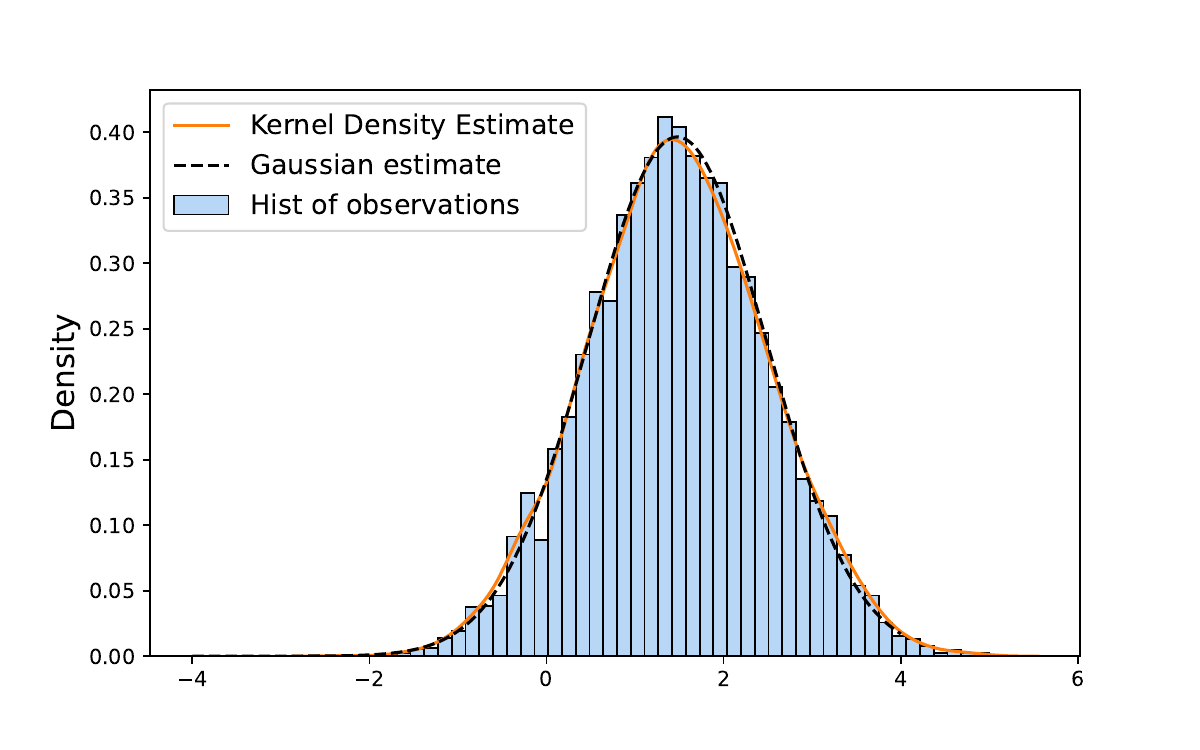}
         \caption{$n=20$}
     \end{subfigure}
     \begin{subfigure}[b]{0.42\textwidth}
         \centering
         \includegraphics[width=\textwidth]{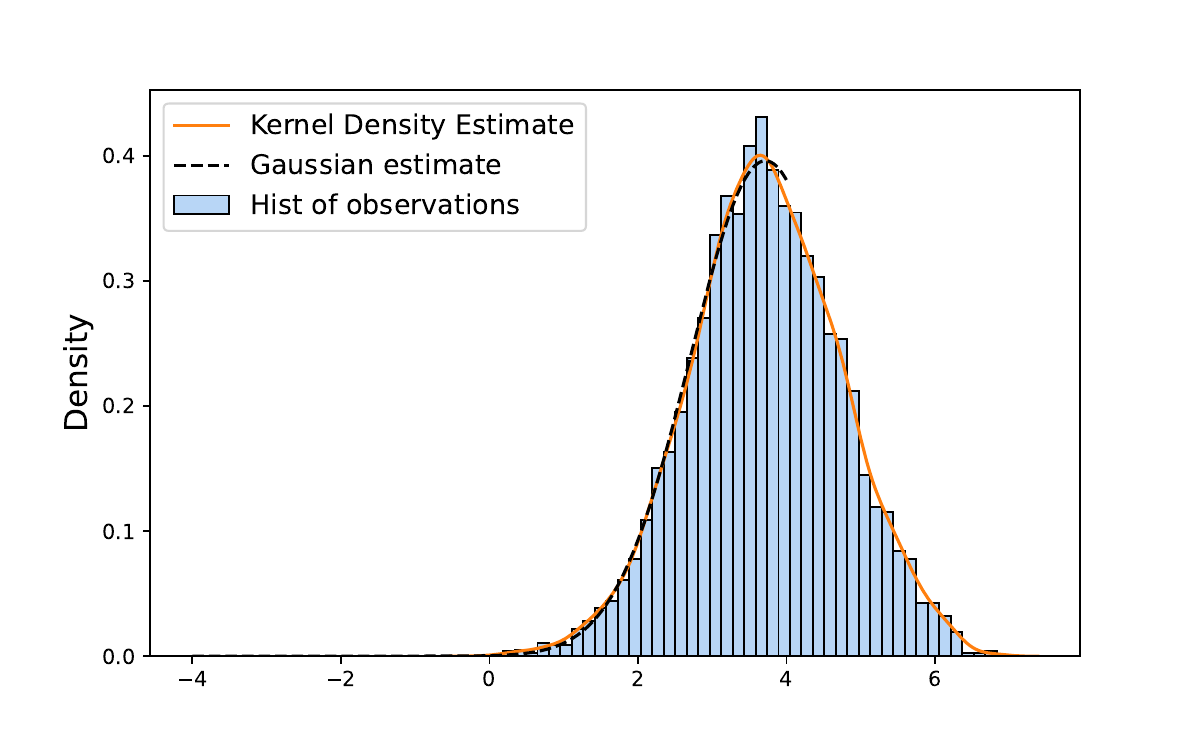}
         \caption{$n=100$}
     \end{subfigure}
        \caption{Histogram of  $\sqrt{n} \log(\|\phi(Y_{L})\|/ \|\phi(Y_{0})\|)$ for depth $L = 100$ and different widths $n \in \{2,3,4,6,20, 100\}$ based on $N=5000$ simulations. Gaussian density estimate and (Gaussian) kernel density estimates are shown. We observe a great match between the best Gaussian estimate and the empirical distribution, which confirms the quasi-log-normal theoretical predictions from \cref{thm:norm_post_act}.}
        \label{fig:quasi_gaussian_behaviour}
\end{figure}
resembles the log-normal distribution. We call this Quasi Geometric Brownian Motion distribution (Quasi-GBM).

\begin{thm}[Quasi-GBM behaviour of the post-activations norm]\label{thm:norm_post_act}
We have that for all $t \in [0,1]$,
$$
\|\phi(X_t)\| = \|\phi(X_0)\| \exp\left(\frac{1}{\sqrt{n}} \hat{B}_t + \frac{1}{n} \int_{0}^t \mu_s ds\right), \quad \textrm{almost surely,}
$$
where $\mu_s =  \frac{1}{2}\|\phi'(X_s)\|^2 - 1$, and $(\hat{B})_{t \geq 0}$ is a one-dimensional Brownian motion. As a result, for all $0 \leq s\leq t \leq 1$
$$
\E\left[ \log \left( \frac{\|\phi(X_t)\|}{\|\phi(X_s)\|} \right) \huge| \, \|\phi(X_0)\| > 0 \right] = \left(\frac{(1- 2^{-n})^{-1}}{4} - \frac{1}{n}\right) (t-s).
$$
Moreover, for $n\geq 2$, we have 
$$
\textup{Var}\left[ \log \left( \frac{\|\phi(X_t)\|}{\|\phi(X_s)\|} \right) \huge| \, \|\phi(X_0)\| > 0 \right] \leq  \left(n^{-1/2} + \Gamma_{s,t}^{1/2}\right)^2 (t-s),
$$
where $\Gamma_{s,t} = \frac{1}{4}\int_{s}^t \,\left(\left(\E \phi'(X^1_u) \phi'(X^2_u) - \frac{(1 - 2^{-n})^2}{4}\right) + n^{-1}\left(\frac{1 - 2^{-n}}{2} - \E \phi'(X^1_u) \phi'(X^2_u)\right)\right) du$. 
\end{thm}
\begin{figure}[h]
     \centering
     \begin{subfigure}[b]{0.42\textwidth}
         \centering
         \includegraphics[width=\textwidth]{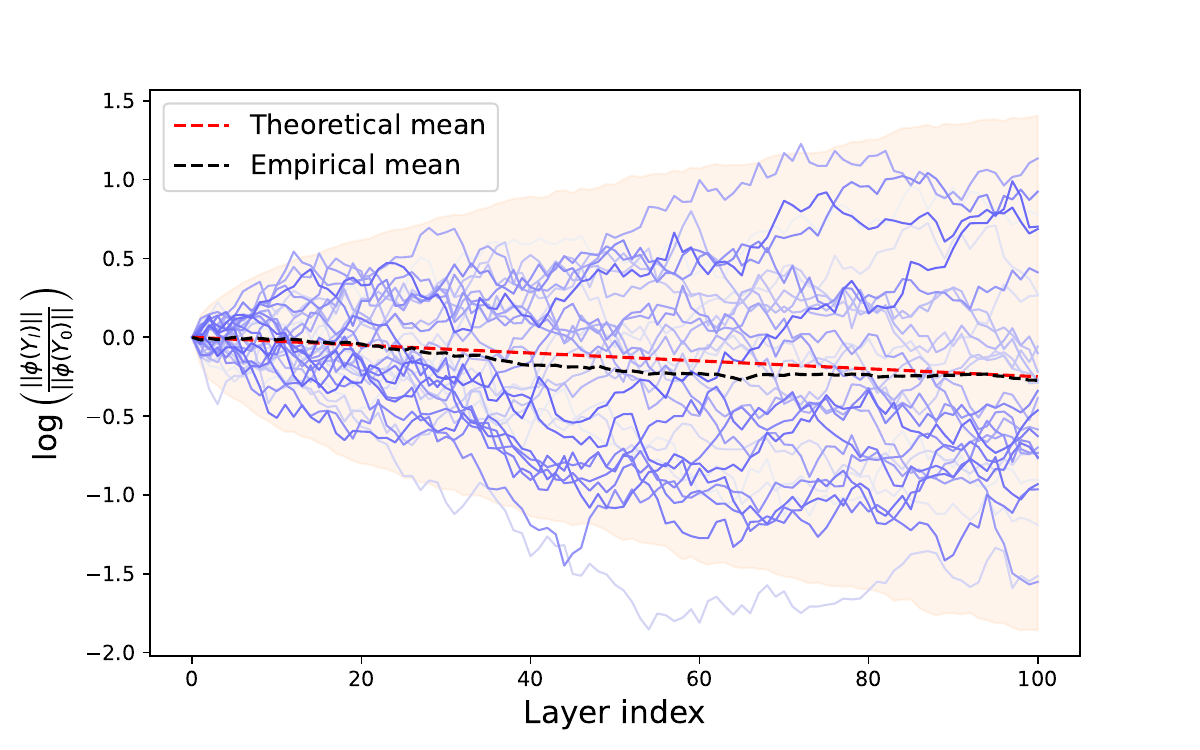}
         \caption{$n=2$}
     \end{subfigure}
     \begin{subfigure}[b]{0.42\textwidth}
         \centering
         \includegraphics[width=\textwidth]{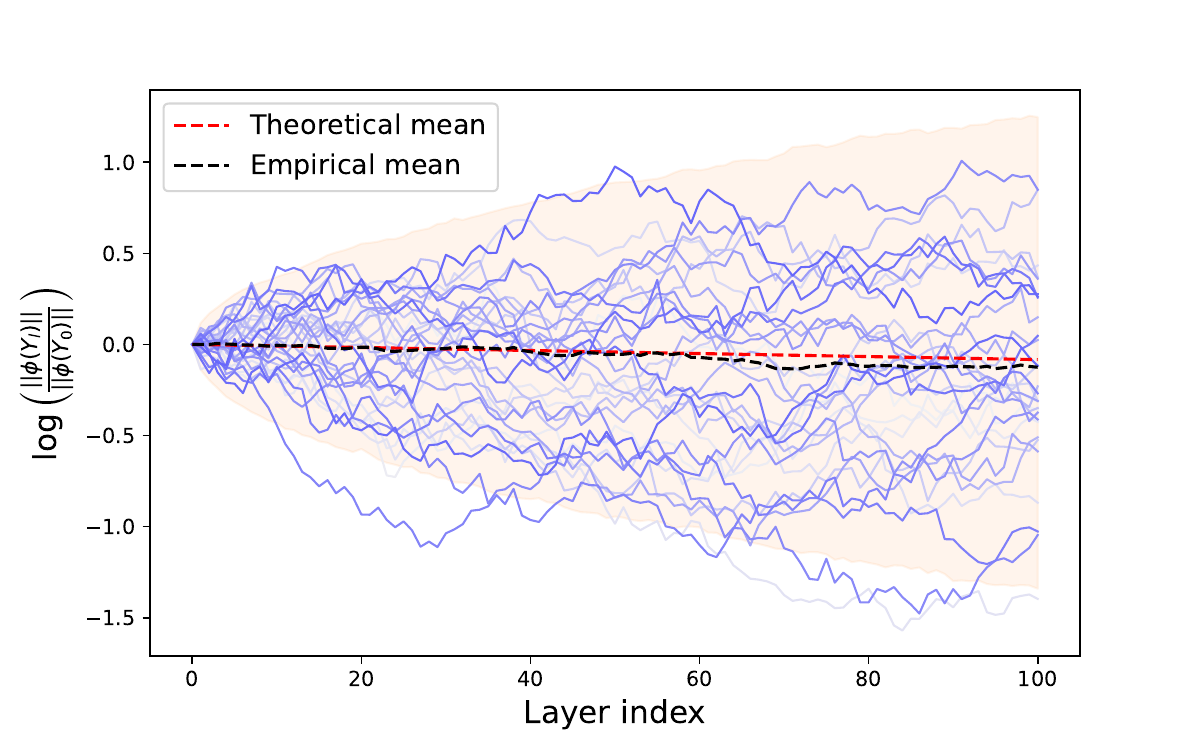}
         \caption{$n=3$}
     \end{subfigure}
     \begin{subfigure}[b]{0.42\textwidth}
         \centering
         \includegraphics[width=\textwidth]{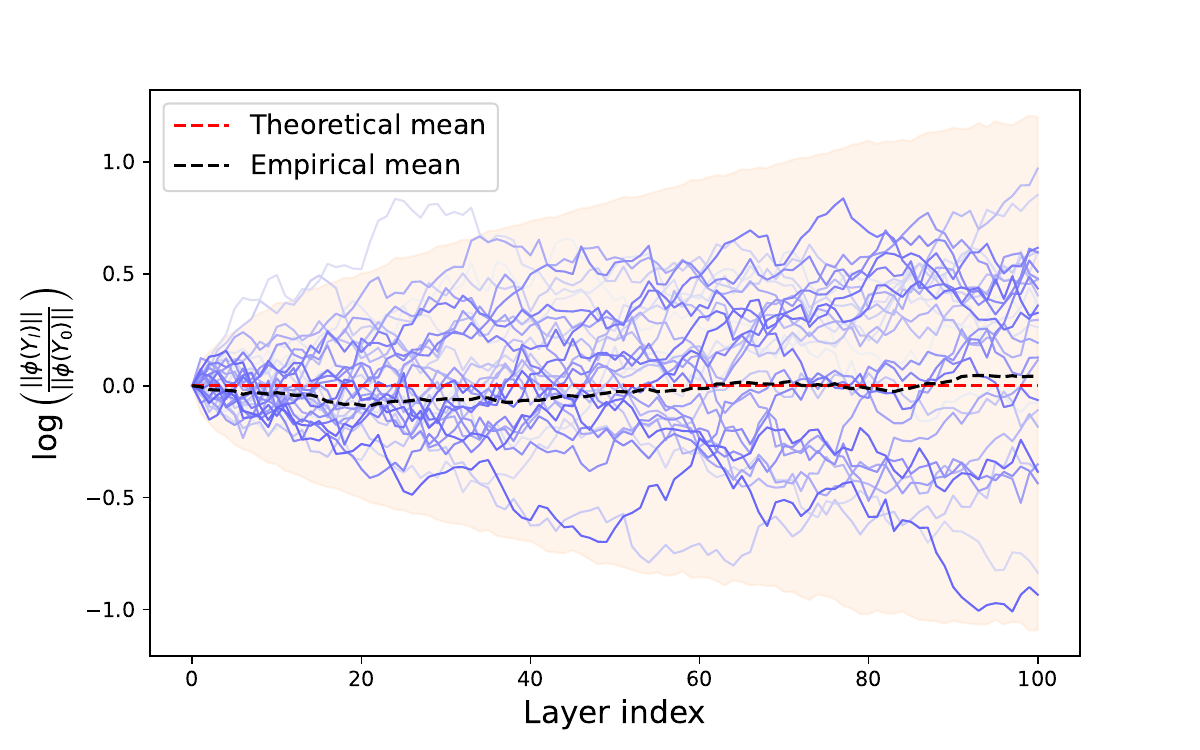}
         \caption{$n=4$}
     \end{subfigure}
     \begin{subfigure}[b]{0.42\textwidth}
         \centering
         \includegraphics[width=\textwidth]{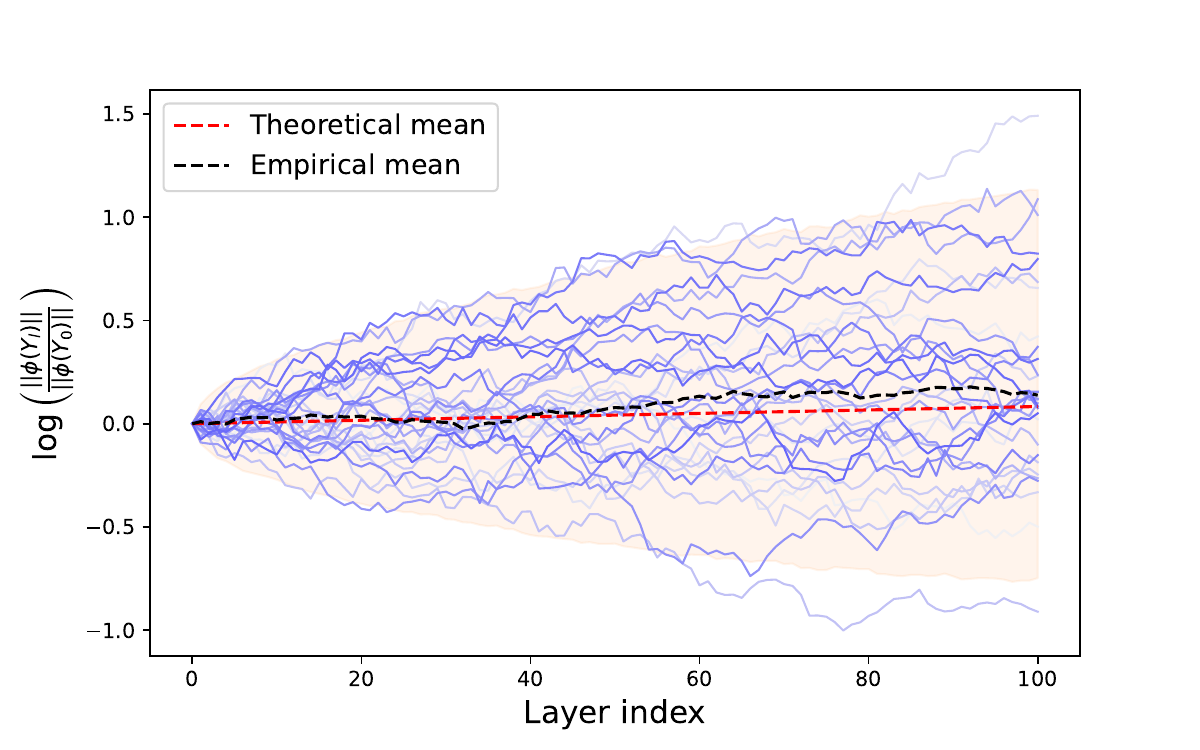}
         \caption{$n=6$}
     \end{subfigure}
     \begin{subfigure}[b]{0.42\textwidth}
         \centering
         \includegraphics[width=\textwidth]{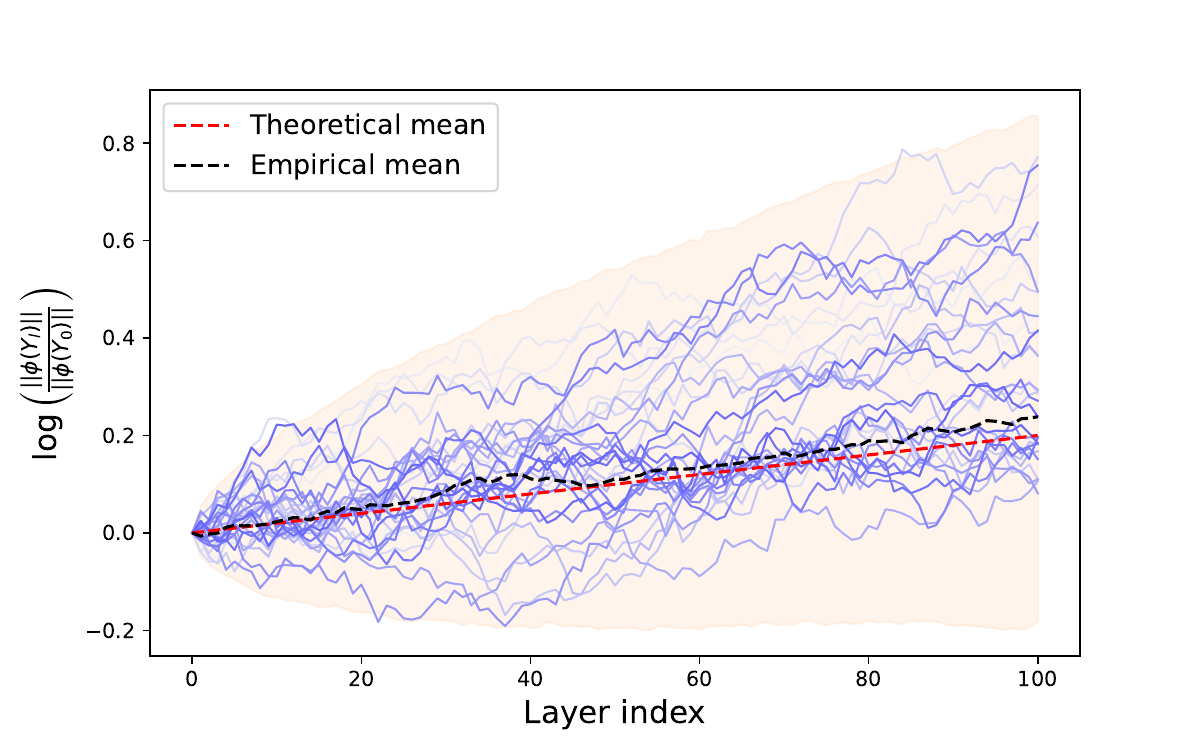}
         \caption{$n=20$}
     \end{subfigure}
     \begin{subfigure}[b]{0.42\textwidth}
         \centering
         \includegraphics[width=\textwidth]{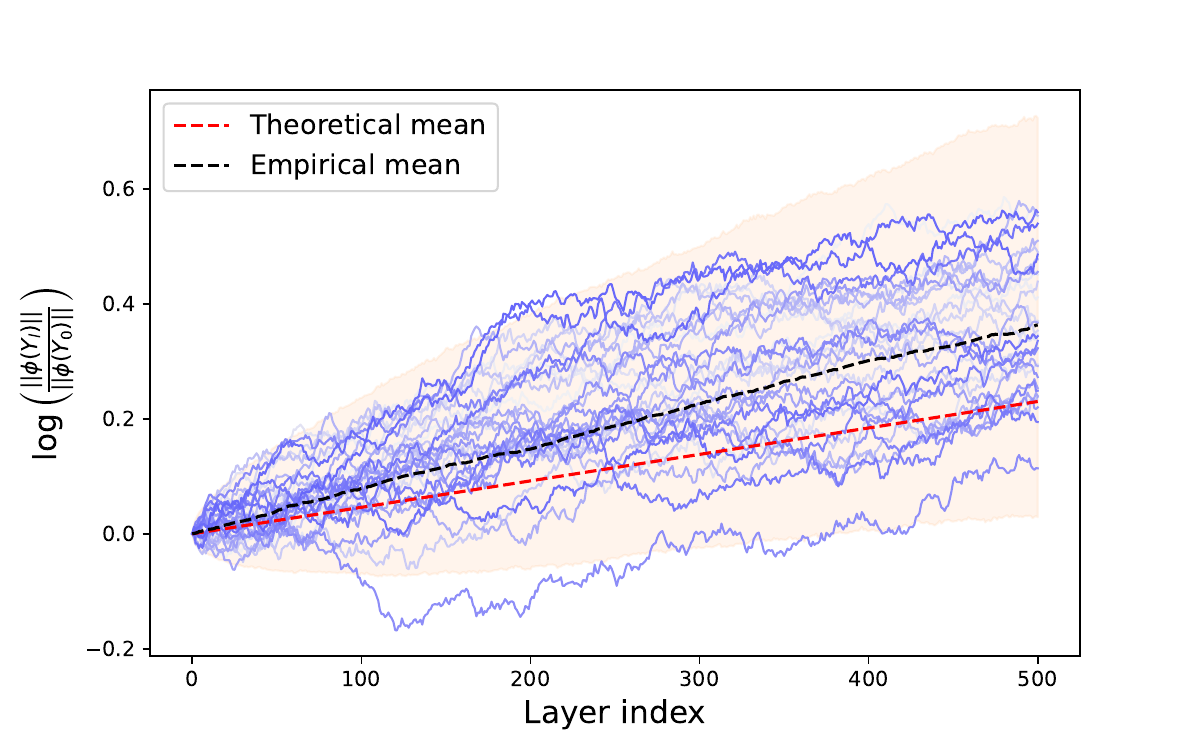}
         \caption{$n=100$}\label{fig:neural_paths_log_n100}
     \end{subfigure}
        \caption{30 simulations of the sequence $\left(\log(\|\phi(Y_{l})\|/ \|\phi(Y_{0})\|)\right)_{1 \leq l \leq L}$ for depth $L = 100$ and different widths $n \in \{2,3,4,6,20, 100\}$. Theoretical means from \cref{thm:norm_post_act} are shown in red dashed lines and compared to their empirical counterparts. We observe that when the ratio $L/n$ increases (especially for $n=100$), the empirical mean also increases and becomes significantly different from the theoretical prediction.}
        \label{fig:neural_paths_log}
\end{figure}

Different tools from stochastic calculus and probability theory are used in the proof of \cref{thm:norm_post_act}. Technical details are provided in \cref{sec:proof_main_thm}. The first result in the theorem suggests that the norm of the post-activations has a quasi-log-normal distribution (conditionally on $X_0$). The first term in the exponential is Gaussian ($n^{-1/2} \hat{B}_t$) and the second term depends on $n^{-1} \mu_s$, which involves an average over $(\phi'(X_s^i))_{1 \leq i \leq n}$. In the large width limit, this average concentrates around its mean as we will see in \cref{thm:infinite_width}. In \cref{fig:quasi_gaussian_behaviour}, we show the histogram of $\sqrt{n} \log(\|\phi(Y_L)\| / \|\phi(X_0)\|)$ for depth $L=100$ and varying widths $n$. Surprisingly, the log-normal approximation fits the empirical distribution very well even for small widths $n \in \{2, 3, 4, 6\}$ for which the term $n^{-1}\mu_s$ is not necessarily close to its mean\footnote{We currently do not have a rigorous explanation for this effect. A possible explanation for this empirical result is that the integral over $\mu_s$ has some `averaging' effect.}. More interestingly, the result of \cref{thm:norm_post_act} sheds light on an intriguing regime change that occurs between widths $n = 3$ and $n=4$. Indeed, for $n \leq 3$, the logarithmic growth factor of the norm of the post-activations $\|\phi(X_t)\|$ tends to decrease with depth on average, while it increases for $n \geq 4$. When $n=4$, the average growth is positive although very small. This regime change phenomenon suggests that for $n \leq 3$, the random variable $\|\phi(X_t)\|/\|\phi(X_s)\|$ has significant probability mass in the region $(0,1)$. This probability mass tends to $0$ as $n$ increases since $\|\phi(X_t)\|/\|\phi(X_s)\|$ converges to a deterministic constant (we will see this in the next theorem), and the variance upperbound in \cref{thm:norm_post_act} converges to $0$ when $n$ goes to infinity, which can be explained by the fact that $\E \phi'(X^1_u) \phi'(X^2_u) \overset{n\to\infty}{\longrightarrow} 1/4$ (the coordinates become independent in the large width limit, see next theorem). Experiments showing this concentration are provided in \cref{sec:non_scaled_log_ratio}.
In \cref{fig:neural_paths_log}, we simulate 30 neural paths (i.e. $(Y_l)_{1 \leq l \leq L}$) for depth $L = 100$ and compute the logarithmic factor $\log(\|\phi(Y_l)\| / \|\phi(Y_0)\|)$. An excellent match with the theoretical results is observed for widths $n \in \{2, 3, 4, 6, 20\}$. A mismatch between theory and empirical results appears when $n=50$, which is expected, since the theoretical results of \cref{thm:norm_post_act} yield good approximations only when $n \ll L$. 

Notice that the case of $n=1$ matches the result of \cref{prop:relu_gbm_1d}. Indeed, the latter implies that conditionally on $\phi(X_0)>0$, we have  $\log(\phi(X_t)/\phi(X_0)) = \log(X_t/X_0) \sim -t/2 + B_t$ where $B$ is a one-dimensional Brownian motion, and where we have used the fact that $X_t>0$ for all $t$. This result can be readily obtained from \cref{thm:norm_post_act} by setting $n=1$.

An interesting question is that of the infinite-width limit of the process $X_t$, which corresponds to the sequential limit infinite-depth-then-infinite-width of the ResNet $Y_{\lfloor tL \rfloor}$ (\cref{eq:resnet}). We discuss this in the next section.

\subsection{Infinite-width limit of infinite-depth networks}
In the next result, we show that when the width goes to infinity, the ratio $\|\phi(X_t)\| / \|\phi(X_0)\|$ concentrates around a layer dependent ($t$-dependent) constant. In this limit, the coordinates of $X_t$ converge in $L_2$ to a Mckean-Vlasov process, which allows us to recover the Gaussian behaviour of the pre-activations of the ResNet. We later compare this with the converse sequential limit infinite-width-then-infinite-depth where the pre-activations are also normally distributed, and show a key difference in the variance of the Gaussian distribution.
\begin{thm}[Infinite-depth-then-infinite-width limit]\label{thm:infinite_width} For $0 \leq s \leq t \leq 1$, we have
$$
\log\left(\frac{\|\phi(X_t)\|}{\|\phi(X_s)\|}\right)\,  \ind_{\{ \|\phi(X_0)\|>0\}} \underset{n \to \infty}{\longrightarrow} \frac{t-s}{4}, \quad \textrm{and, }\quad \frac{\|\phi(X_t)\|}{\|\phi(X_s)\|}\,  \ind_{\{ \|\phi(X_0)\|>0\}}  \underset{n \to \infty}{\longrightarrow} \exp\left(\frac{t-s}{4}\right).
$$
where the convergence holds in $L_1$.\\

Moreover, 
we have that
$$
\sup_{i \in [n]} \E \left(\sup_{t \in [0,1]} |X^i_t - \tilde{X}^i_t|^2\right) = \bigO(n^{-2/5}),
$$
where $X^{i}_t$ is the solution of the following (Mckean-Vlasov) SDE
$$
d\tilde{X}^i_t = \left(\E \phi(\tilde{X}^i_t)^2\right)^{1/2} dB^i_t, \quad \tilde{X}^i_0 = X^i_0. 
$$

As a result, 
the pre-activations $Y^i_{\lfloor t L \rfloor}$ (\cref{eq:resnet}) converge in distribution to a Gaussian distribution in the limit infinite-depth-then-infinite-width
$$
\forall i \in [n],\,\,\,\, Y^i_{\lfloor t L \rfloor}  \xrightarrow{L \to \infty \textrm{ then }n \to \infty} \normal(0, d^{-1} \|x\|^2 \exp(t/2)).
$$
\end{thm}
The proof of \cref{thm:infinite_width} requires the use of a special variant of the Law of large numbers for non \iid random variables, and a convergence result of particle systems from the theory of Mckean-Vlasov processes. Details are provided in \cref{sec:proof_infinite_width}. In neural network terms, \cref{thm:infinite_width} shows that the logarithmic growth factor of the norm of the post-activations, given by $\log\left( \|\phi(Y_{\lfloor t L \rfloor})\| / \|\phi(Y_{\lfloor s L \rfloor})\|\right)$, converges to $(t-s)/4$ in the sequential limit $L \to \infty$, then $n \to \infty$. More importantly, the pre-activations $Y_{\lfloor t L \rfloor}^i$ converge in distribution to a zero-mean Gaussian distribution in this limit, with a layer-dependent variance. In the converse sequential limit, i.e. $n \to \infty$, then $L \to \infty$, the limiting distribution of the pre-activations $Y_{\lfloor t L \rfloor}^i$ is also Gaussian with the same variance. We show this in the following result, which uses Lemma 5 in \citep{hayou21stable}.
\begin{thm}[Infinite-width-then-infinite-depth limit]\label{thm:infnite_width_then_infinite_depth}
Let $t\in [0,1]$. Then, in the limit $\lim_{L \to \infty} \lim_{n \to \infty}$ (infinite width, then infinite depth), we have that
$$
 \frac{\|\phi(Y_{\lfloor t L \rfloor})\|}{\|\phi(Y_0)\|}\,  \ind_{\{ \|\phi(Y_0)\|>0\}} \longrightarrow \exp\left(\frac{t}{2}\right),
$$
where the convergence holds in probability.\\

Moreover, the pre-activations $Y^i_{\lfloor t L \rfloor}$ (\cref{eq:resnet}) converge in distribution to a Gaussian distribution in the limit infinite-width-then-infinite-depth
$$
\forall i \in [n],\,\,\,\, Y^i_{\lfloor t L \rfloor}  \xrightarrow{n \to \infty \textrm{ then }L \to \infty} \normal(0, d^{-1} \|x\|^2 \exp(t)).
$$
\end{thm}

The proof of \cref{thm:infnite_width_then_infinite_depth} is provided in \cref{sec:infinite_width_then_infinite_depth}. We use existing results from \cite{hayou21stable} on the infinite-depth asymptotics of the neural network Gaussian process (NNGP). It turns out that the order to the sequential limit (taking the width to infinity first, then taking the depth to infinity, or the converse) does not affect the limiting distribution, which is a Gaussian with variance  $\propto \exp(t/2)$). Intuitively, by taking the width to infinity first, we make the coordinates independent from each other, and the processes $(Y^i_l)_{1 \leq L}$ become \iid Markov chains. Taking the infinite-depth limit after the infinite-width limit consists of taking the infinite-depth limit of one-dimensional Markov chains. On the other hand, when we take depth to infinity first, the coordinates $(X^i_t)_{1 \leq i \leq n}$ remain dependent (through the volatility term $n^{-1/2} \|\phi(X_t)\|$), which results in the Quasi-log-normal behaviour of the norm of the post-activations (\cref{thm:norm_post_act}). Taking the width to infinity then yields an asymptotic norm of the post-activations equal to $\|\phi(X_0)\| \exp(t/2)$ (\cref{thm:infinite_width}) which is the same norm in the converse limit (\cref{thm:infnite_width_then_infinite_depth}). It remains to take the width to infinity to decouple the coordinates and obtain the Gaussian distribution (through the Mckean-Vlasov dynamics). Knowing that the variance of the pre-activations is mainly determined by the norm of the post-activations (\cref{eq:main_general_n}), we can see why the variance is similar in both sequential limits. 

\section{Discussion on the case of multiple inputs}\label{sec:covariance}
The result of \cref{prop:main_conv} can be easily generalized to the multiple input case, and the resulting dynamics is still an SDE. The generalization to the multiple inputs case is given by \cref{prop:main_conv_multiple} in the Appendix.\\
\begin{wrapfigure}{r}{0.45\textwidth}
  \begin{center}
    \includegraphics[width=0.4\textwidth]{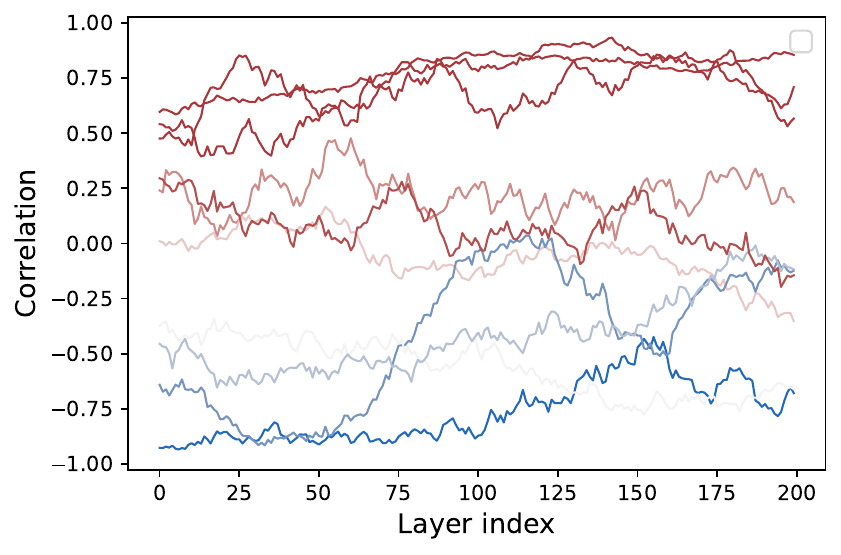}
  \end{center}
  \caption{10 Simulations of the correlation path $\left(\frac{\langle Y_{\lfloor t L\rfloor}(a), Y_{\lfloor t L\rfloor}(b) \rangle}{\|Y_{\lfloor t L\rfloor}(a)\| \|Y_{\lfloor t L\rfloor}(b)\|}\right)_{1 \leq l \leq L}$ for depth $L=200$, width $n=20$, and different $(a,b)$ (different initial correlations $c_0$). The color code depends only on the initial correlation value $c_0$ (red for the largest correlations values)}
  \label{fig:correlations}
  \vspace{-1em}
\end{wrapfigure} 

An important question in the literature on infinite-width neural networks is the behaviour of the correlation of the pre-activations (or the post-activations) for different inputs $a$ and $b$, which is given by $\frac{\langle Y_{\lfloor t L\rfloor}(a), Y_{\lfloor t L\rfloor}(b) \rangle}{\|Y_{\lfloor t L\rfloor}(a)\| \|Y_{\lfloor t L\rfloor}(b)\|}$. This correlation can be as a geometric measure of the information as it propagates through the network. In the infinite-width-then-depth limit, this correlation (generally) converges to a degenerate limit (a constant value) which results in either a constant or a sharp landscape of the network output and causes gradient exploding/vanishing issues \citep{samuel2017, yang2017meanfield, hayou2019impact}. Techniques such block scaling \citep{hayou21stable}, or kernel shaping \citep{zhang2022deep, martens2021rapid} solve this problem and ensure that the correlation is well-behaved in the large depth limit. 
In our case, when the width $n$ is finite and the depth $L$ is taken to infinity, we can define the correlation for two inputs $a \neq b$ and time $t \in [0,1]$ by 
$$c_t(a, b) \overset{def}{=} \frac{\langle X_t(a), X_t(b) \rangle}{\|X_t(a)\| \|X_t(b)\|}.$$
Using \ito's lemma, $c_t$ has dynamics of the form 
\begin{equation}\label{eq:correlation}
d  c_t(a,b) = \Psi(X_t(a), X_t(b)) dB_t,
\end{equation}
for some non-trivial mapping $\Psi$. Unfortunately, this kind of dynamics (which is not an SDE) is generally intractable, and we are currently investigating these dynamics for future work. However, since we scale the ResNet blocks with the factor $1/\sqrt{L}$ (\cref{eq:resnet}), which is the same scaling that solves the degeneracy issue in the infinite-width-then-depth limit \citep{hayou21stable}, it should be expected that the correlation kernel $c_t$ does not converge to a degenerate limit. 

In \cref{fig:correlations}, we simulate the correlation path in a ResNet of depth $L=200$ and width $n=20$. The paths exhibits some level of stochasticity but no degeneracy can be observed. Understanding the correlation dynamics (\cref{eq:correlation})  in the infinite-depth limit of finite-width networks is an interesting open question. The infinite-width limit\footnote{The infinite-width limit of infinite-depth correlations} of these dynamics is also an interesting open question. We leave this for future work. 
    
\section{Practical implications}
Our theoretical analysis has many interesting implications from a practical standpoint. Here we summarize some key insights form our results.

\paragraph{Initialization and stability in the large depth limit.} An important factor pertaining to the trainability of neural networks is the behaviour of the neurons (pre/post-activations). Ensuring that the neurons are well-behaved at initialization is crucial for training since the first step of any gradient-based training algorithm depends on the values of the neurons at initialization. 
This has led to interesting developments in initialization schemes for MLPs such as the Edge of Chaos \citep{poole, samuel2017} which ensures that the variance of the pre-activation does not (exponentially) vanish or explode in the large depth limit. In the case of ResNet, we know from the existing theory on the infinite-width limit of neural networks that scaling the residual blocks with $1/\sqrt{L}$ stabilizes the pre/post-activations in the large depth limit \citep{hayou21stable}. Hence, we do not need a special initialization scheme with this scaling. However, one could argue that this (approximately) ensures stability \emph{only} when the width is much larger than the depth. What about the other cases when $n \approx L$ or $n \ll L$? the last case can be studied by fixing the width and taking the depth to infinity. In our paper, we not \emph{only} show that the neurons remain stable in fixed-width large-depth networks, but we fully characterize their behaviour when the depth is infinite and show that it follows an SDE in this limit. To summarize, we show that initializing ResNet \cref{eq:resnet} with standard Gaussian random variables and scaling the blocks with $1/\sqrt{L}$ ensures stability inside the network in large-depth (fixed-width) networks (notice that this is actually equivalent to scaling the variance of the initialization weights with $1/L$, which can be seen as an initialization scheme). Intuitively, by stabilizing the pre-activations, we also stabilize the gradients. To confirm this intuition, we show in \cref{fig:gradients} the evolution of gradient norms as they back-propagate through the network. This experiment was conducted by fixing the last layer's gradient to a constant value and back-propagating the gradient from there. The result shows that the $1/\sqrt{L}$ scaling, along with standard Gaussian initialization, ensure well-behaved gradients which is a desirable property for gradient-based training. Another interesting property of the Edge of Chaos initialization scheme for MLPs is that it ensures that correlation kernel (correlation between the pre-activations for different inputs) does not exponentially converge to a degenerate value (constant value)\footnote{The correlation still converges to 1 with an EOC initialization. The benefit of the EOC lies in the fact that the convergence rate is much slower (polynomial Vs exponential) \citep{samuel2017, hayou2019impact}}. We discussed some aspects of the correlation kernel in \cref{sec:covariance} and showed empirically that with the $1/\sqrt{L}$ scaling, the correlation is well-behaved and does not converge to degenerate values (\cref{fig:correlations}).
\begin{wrapfigure}{r}{0.35\textwidth}
  \begin{center}
  \vspace{-2em}
    \includegraphics[width=0.35\textwidth]{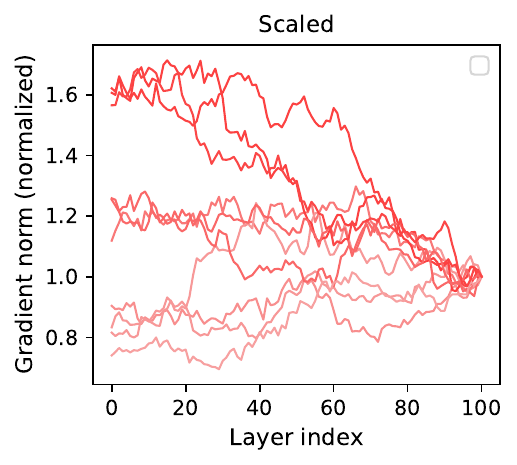}\\
    \includegraphics[width=0.35\textwidth]{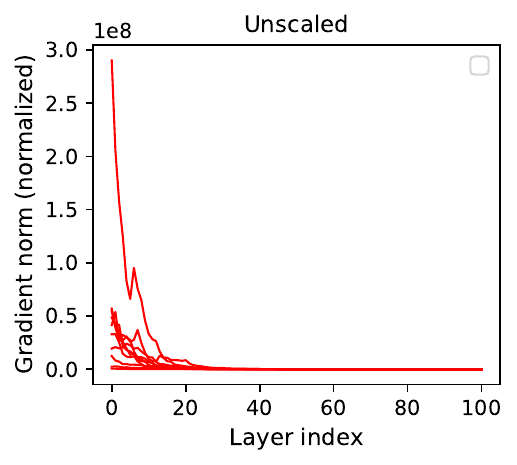}
  \end{center}
   \vspace{-2em}
  \caption{10 Simulations of the gradient norm for scaled ResNet (\cref{eq:resnet}) and non-scaled ResNet ($Y_l = Y_{l-1} + W_l \phi(Y_{l-1})$) for depth $L=100$ and width $n=10$. We normalize the gradients norms by the gradient norm of the last layer. The color code depends only on the ratio of the graident norm at the first layer to that of the last layer (dark red for the largest values). Without scaling, the gradient norm explodes (highly likely). The $1/\sqrt{L}$ stabilizes the gradients as they back-propagate through the network.}
  \label{fig:gradients}
  \vspace{-3em}
\end{wrapfigure}
\paragraph{Network collapse.} Another issue that could occur in finite-width networks is that of network collapse, i.e. when the pre-activations in a hidden layer are all negative, which causes the post-activations to be all zero. In ResNet (\cref{eq:resnet}), this implies that increasing depth beyond some level has no effect on the network output. This is problematic since the weights in those `inactive' layers have zero gradient and thus will not be updated when such event occurs. A simple way to understand network collapse is to see what happens at initialization. When the width $n$ is sufficiently large, one can expect that such event is unlikely to occur. What about small-width neural networks? we offer a simple answer to this question: for finite-width neural networks, increasing the depth $L$ ensures that such event is unlikely to happen. This is true even for extremely small widths, e.g. $n=2, 3$, which is counter-intuitive. Empirical results in \cref{fig:killed_probs} support this theoretical prediction.
\paragraph{No universal kernel regime.} An interesting application of fixed-depth infinite-width neural network is the so-called Neural Network Gaussian Process (NNGP). This is the Gaussian process limit of neural networks, that can be used to perform posterior inference and obtain uncertainty estimates \citep{lee_gaussian_process}. The converse case, i.e. fixed-width infinite-depth, has been however poorly understood, and the question of whether the infinite-depth limit of finite-width networks has some universal behaviour has been an open question since. We addressed this question in this work and showed that the limit (in the case of the ResNet architecture \cref{eq:resnet}) does not admit a universal distribution (e.g. Gaussian process in the infinite-width limit). More precisely, this limit is highly sensitive to the choice of the activation function. 
\paragraph{What about infinite-depth-then-width?} the infinite-depth limit of infinite-width neural networks has been studied in the literature \citep{hayou2019impact, hayou_ntk}. It is known that in this limit, the network behaves as a Gaussian process with a well-defined kernel. What about the converse limit, i.e. infinite-width limit of infinite-depth networks? this has been so far an open question, and our work addresses one part of it. We show that the marginal distributions are zero-mean Gaussians with the same variance as in the infinite-width-then-depth limit. Characterizing the full covariance kernel is still however an open question (see \cref{sec:covariance} for a discussion on this topic).

\section{Conclusion, discussion, and limitations}
Understanding the limiting laws of randomly initialized neural networks is important on many levels. Primarily, understanding these limiting laws allows us to derive new designs that are immune to exploding/vanishing pre-activations/gradients phenomena. Next, they also enable a deeper understanding of overparameterized neural networks, and (often) yield many interesting (and simple) justifications to the apparent advantage of overparameterization. So far, the focus has been mainly on the infinite-width limit (and infinite-width-then-infinite-depth limit) with few developments on the joint limit. Our work adds to this stream of papers by studying the infinite-depth limit of finite-width neural networks. We showed that unlike the infinite-width limit, where we always obtain (under some mild conditions on the activation function) a Gaussian distribution, the infinite-depth limit is highly sensitive to the choice of the activation function; using the \ito's lemma, we showed how we can obtain certain known distributions by carefully tuning the activation function. In the general width limit, we showed an important characteristic of infinite-depth neural networks with general activation functions (including ReLU, conditionally on $\|\phi(X_0)\|>0$): the probability of process collapse is zero, meaning that with probability one, the process $X_t$ does not get stuck at any absorbent point. This is not true for finite-depth ResNets as we can see in \cref{fig:killed_probs}, which highlights the fact that as we increase depth, the collapse probability tends to decrease, and eventually converges to zero in the infinite-depth limit, which is in agreement with our results.

This work, although novel in many aspects, is still far from depicting a complete picture of the infinite-depth limit of finite-width networks. There are still numerous interesting open questions in this research direction. Indeed, one of these is the dynamics of the gradient, and more specifically the behaviour of the NTK in the infinite-depth limit of finite-width neural networks. For instance, we already know that in the joint infinite-width-depth limit of MLPs, the NTK is random \citep{hanin2019finitewidth}; but what happens when the width is fixed and the depth goes to infinity? In the MLP case, a degenerate NTK should be expected. Henceforth, questions remain as to whether a suitable scaling leads to interesting (non-degenerate) infinite-depth limit of the NTK as is the case of the infinite-depth limit of infinite-width NTK \citep{hayou21stable}.


\printbibliography

\newpage
\appendix

\addcontentsline{toc}{section}{Appendix} 

\part{Appendix} 

\parttoc

\newpage
\section{Review of Stochastic Calculus}\label{appendix:stochastic_calculus}

In this section, we introduce the required mathematical framework and tools to handle stochastic differential equations (SDEs). We suppose that we have a probability space $(\Omega, \mathcal{F}, \mathbb{P})$, where $\Omega$ is the event space, $\mathbb{P}$ is the probability measure, and $\mathcal{F}$ is the sigma-algebra associated with $\Omega$. For $n \geq 1$, we denote by $B$ the standard  $n$-dimensional Brownian motion, and $\mathcal{F}_t$ its natural filtration. Equipped with $(\mathcal{F}_t)_{t \geq 0}$, we say that the probability space $(\Omega, \mathcal{F}, (\mathcal{F}_t)_{t \geq 0},  \mathbb{P})$ is a filtered probability space. $\mathcal{F}_t$ is the collection of events that are measurable up to  time $t$, i.e. can be verified if we have knowledge of the Brownian motion $B$ (and potentially some other independent source such as the initial condition of a process $X$ defined by a $B$-driven stochastic differential equation) up to time $t$. We are now ready to define a special type of stochastic processes known as \ito processes.

\subsection{Existence and uniqueness}
\begin{definition}[It$\hat{o}$ diffusion process]\label{def:ito_process}
A stochastic process $(X_t)_{t \in [0,T]}$ valued in $\reals^n$ is called an It$\hat{o}$ diffusion process if it can be expressed as 
$$
X_t = X_0 + \int_{0}^t \mu_s ds + \int_{0}^t \sigma_s dB_s,
$$
where $B$ is a $n$-dimensional Brownian motion and $\sigma_t \in \reals^{n \times n}, \mu \in \reals^n$ are predictable processes satisfying $\int_{0}^T (\|\mu_s\|_2 + \|\sigma_s \sigma_s^\top\|_2) ds < \infty$ almost surely.
\end{definition}

The following result gives conditions under which a strong solution of a given SDE exists, and is unique.

\begin{thm}[Thm 8.3 in \cite{touzi2018}]\label{thm:existence_and_uniqueness}
Let $n \geq 1$, and consider the following SDE
$$
dX_t = \mu(t, X_t) dt + \sigma(t, X_t) dB_t, \quad X_0 \in L_2,
$$
where $B$ is a $m$-dimensional Brownian process for some $m\geq 1$, and $\mu : \reals^+ \times \reals^n \to \reals^n$ and $\sigma:\reals^+ \times \reals^n \to \reals^{n \times m}$ are measurable functions satisfying
\begin{enumerate}
    \item there exists a constant $K>0$ such that for all $t \geq 0$, $x,x' \in \reals^n$
    $$
\|\mu(t, x) - \mu(t,x')\| + \|\sigma(t, x) - \sigma(t,x')\| \leq k \|x - x'\|.
$$
\item the functions $\|\mu(.,0)\|$ and $\|\sigma(.,0)\|$ are $L_2(\reals^+)$ with respect to the Lebesgue measure on $\reals^+$.
\end{enumerate}
Then, for all $T \geq 0$, there exists a unique strong solution of the SDE above.
\end{thm}

\subsection{\ito's lemma}
The following result, known as \ito's lemma, is a classic result in stochastic calculus. We state a version of this result from \cite{touzi2018}. Other versions and extensions exist in the literature (e.g. \citet{Ingersoll1987, Oksendal2003, kloeden}).
\begin{lemma}[\ito's lemma, Thm 6.7 in \cite{touzi2018}]\label{lemma:ito}
Let $X_t$ be an It$\hat{o}$ diffusion process (\cref{def:ito_process}) of the form
$$
dX_t = \mu_t dt + \sigma_t dB_t, t\in [0,T], X_0 \sim \nu
$$
where $\nu$ is some given distribution. Let $g: \reals^+ \times \reals^n \to \reals$ be $\mathcal{C}^{1,2}([0,T], \reals^n)$ (i.e. $\mathcal{C}^1$ in the first variable $t$ and $\mathcal{C}^2$ in the second variable $x$). Then, with probability $1$, we have that 
$$
f(t, X_t) = f(0,X_0) + \int_{0}^t \nabla_x f(s,X_s) \cdot dX_s + \int_{0}^t \left( \partial_t f(s,X_s) + \frac{1}{2} \textup{Tr}\left[ \sigma_s^\top \nabla_x^2 f(s,X_s) \sigma_s\right] \right) ds,
$$
where $\nabla_x f$ and $\nabla_x^2 f$ refer to the gradient and the Hessian, respectively. This can also be expressed as an SDE
$$
d f(t, X_t) =  \nabla_x f(t,X_t) \cdot dX_t + \left( \partial_t f(t,X_t) + \frac{1}{2} \textup{Tr}\left[ \sigma_t^\top \nabla_x^2 f(t,X_t) \sigma_t\right] \right) dt.
$$
\end{lemma}

\subsection{Convergence of Euler's scheme to the SDE solution}
The following result gives a convergence rate of the Euler discretization scheme to the solution of the SDE.
\begin{thm}[ Corollary of Thm 10.2.2 in \cite{kloeden}]\label{thm:convergence_euler}
Let $d \geq 1$ and consider the $\reals^d$-valued ito process $X$ (\cref{def:ito_process}) given by 
$$
X_t = X_0 + \int_{0}^t \mu(s, X_s) ds + \int_{0}^t \sigma(s, X_s) dB_s,
$$
where $B$ is a $m$-dimensional Brownian motion for some $m \geq 1$, $X_0$ satisfies $\E \|X_0\|^2 < \infty$, and $\mu: \reals^+ \times \reals^d \to \reals^d$ are $\sigma: \reals^+ \times \reals^d \to \reals^{d \times m}$ are measurable functions satisfying the following conditions:
\begin{enumerate}
    \item There exists a constant $K>0$ such that for all $t \in \reals, x,x' \in \reals^d$,
    $$
    \|\mu(t,x) - \mu(t,x')\| + \|\sigma(t,x) - \sigma(t,x')\| \leq K \|x - x'\|.
    $$
    \item There exists a constant $K'>0$ such that for all $t \in \reals, x \in \reals^d$
    $$
    \|\mu(t,x)\| + \|\sigma(t,x)\| \leq K'( 1 + \|x\|).
    $$
    \item There exists a constant $K''>0$ such that for all $t, s \in \reals, x \in \reals^d$,
    $$
    \|\mu(t,x) - \mu(s,x)\| + \|\sigma(t,x) - \sigma(s,x)\| \leq K''(1 + \|x\|) |t - s|^{1/2}.
    $$
\end{enumerate}
Let $\delta \in (0,1)$ such that $\delta^{-1} \in \mathbb{N}$ (integer), and consider the times $t_k = k \delta$ for $k \in \{1, \dots, \delta^{-1}\}$. Consider the Euler scheme given by 
$$
Y^i_{k+1} = Y^i_{k} + \mu^i(t_k, Y^k_{n}) \delta + \sum_{j=1}^m \sigma^{i,j}(t_k, Y^k_{n}) \Delta B^j_k, \quad Y^i_0 = X^i_0,
$$
where $Y^i, \mu^i, \sigma^{i,j}$ denote the coordinates of these vectors for $i \in [d], j\in[m]$, and $\Delta B^j_k \sim \normal(0, \delta)$. Then, we have that 
$$
\E \sup_{t \in [0,1]} \|X_{t} - Y_{\lfloor t \delta^{-1}\rfloor}\|^2 = \mathcal{O}(\delta).
$$
\end{thm}

We can extend the result of \cref{thm:convergence_euler} to the case of locally Lipschitz drift and volatility functions $\mu$ and $\sigma$. For this purpose, let us first define local convergence.

\begin{definition}\label{def:local_convergence}
Let $(X^L)_{L \geq 1}$ be a sequence of processes and $X$ be a stochastic process. For $r>0$, define the following stopping times
$$
\tau^L = \{ t \geq 0: |X^L_t| \geq r\}, \tau = \{t \geq 0: |X_t| \geq r\}.
$$ 
We say that $X^L$ converges locally to $X$ 
if for any $r > 0$,  $X^L_{t \land \tau^L}$ converge to $X_{t \land \tau}$. \\
This definition is general for any type of convergence, we will specify clearly the type of convergence when we use this notion of local convergence.
\end{definition}

\begin{lemma}[Locally-Lipschitz coefficients]\label{lemma:locally_lipschitz_convergence}
Consider the same setting of \cref{thm:convergence_euler} with the following conditions instead 
\begin{enumerate}
    \item For any $r>0$, there exists a constant $K>0$ such that for all $t \in \reals, x,x' \in \reals^d$ with $\|x\|,\|x'\| \leq r$,
    $$
    \|\mu(t,x) - \mu(t,x')\| + \|\sigma(t,x) - \sigma(t,x')\| \leq K \|x - x'\|.
    $$
    \item For any $r>0$, there exists a constant $K'>0$ such that for all $t \in \reals, x \in \reals^d$ satisfying $\|x\| \leq r$
    $$
    \|\mu(t,x)\| + \|\sigma(t,x)\| \leq K'( 1 + \|x\|).
    $$
    \item For any $r>0$, there exists a constant $K''>0$ such that for all $t, s \in \reals, x \in \reals^d$ satisfying $\|x\| \leq r$,
    $$
    \|\mu(t,x) - \mu(s,x)\| + \|\sigma(t,x) - \sigma(s,x)\| \leq K''(1 + \|x\|) |t - s|^{1/2}.
    $$
\end{enumerate}
Then, for any $r>0$, we have that 
$$
\E \sup_{t \in [0,1]} \|X_{t \land \tau} - Y_{\lfloor (t \land \tau_\delta) \delta^{-1}\rfloor}\|^2 = \mathcal{O}(\delta),
$$
where $\tau_\delta = \inf \{t \geq 0: \|Y_{\lfloor t \delta^{-1}\rfloor}\| > r\}$, and $\tau = \inf \{t \geq 0: \|X_{t}\| > r\}$.
\end{lemma}
We omit the proof here as it consists of the same techniques used in \cite{kloeden}, with the only difference consisting of considering the stopped process $X^\tau$. By stopping the process, we force the process to stay in a region where the coefficients are Lipschitz.

\subsection{Convergence of Particles to the solution of Mckean-Vlasov process}
The next result gives sufficient conditions for the system of particles to converge to its mean-field limit, known as the Mckean-Vlasov process.
\begin{thm}[ Mckean-Vlasov process, Corollary of Thm 3 in \cite{jourdain2007}]\label{thm:convergence_mckean}
Let $d \geq 1$ and consider the $\reals^d$-valued ito process $X$ (\cref{def:ito_process}) given by 
$$
d X_t = \sigma(X_t, \nu^n_t) dB_t, \quad X_0 \textrm{ has \iid components},
$$
where $B$ is a $d$-dimensional Brownian motion, $\nu^n_t \overset{def}{=} \frac{1}{d} \sum_{i=1}^d \delta_{\{X^i_t\}}$ the empirical distribution of the coordinates of $X_t$, and $\sigma$ is real-valued and Lipschitz-continuous when the space $\reals^n \times \mathcal{P}_2(\reals^n)$ is endowed with the product topology of the euclidean distance on $\reals^s$ and the Wasserstein metric on $\mathcal{P}_2(\reals^n)$. Then, we have that for all $T \in \reals^+$,
$$
\sup_{i \in [n]} \E \left( \sup_{t \leq T} |X^i_t - \tilde{X}^i_t|^2 \right) = \bigO(n^{-2/5}),
$$
where $\tilde{X}^i$ is the solution of the following Mckean-Vlasov equation
$$
d\tilde{X}^i_t = \sigma(\tilde{X}^i_t, \nu^i_t) dB^i_t, \quad \tilde{X}^i_0 = X^i_0,
$$
where $\nu^i_t$ is the distribution of $\tilde{X}^i$.
\end{thm}

\begin{proof}
This is a direct result of Thm 3 in \cite{jourdain2007}. The bounded moment condition holds for $k=1$ (dimension of the particles), and the conclusion is straightforward.
\end{proof}

\subsection{Other results from probability and stochastic calculus}
The next trivial lemma has been opportunely used in \cite{li21loggaussian} to derive the limiting distribution of the network output (multi-layer perceptron) in the joint infinite width-depth limit. This simple result will also prove useful in our case of the finite-width-infinite-depth limit.
\begin{lemma}\label{lemma:gaussian_vec}
Let $W \in \reals^{n\times n}$ be a matrix of standard Gaussian random variables $W_{ij} \sim \normal(0,1)$. Let $v \in \reals^n$ be a random vector independent from $W$ and satisfies $\|v\|_2 = 1$ . Then, $W v \sim \normal(0, I)$.
\end{lemma}
\begin{proof}
The proof follows a simple characteristic function argument. Indeed, by conditioning on $v$, we observe that $Wv \sim \normal(0, I)$. Let $u \in \reals^n$, we have that \begin{align*}
    \E_{W,v}[e^{i \langle u, Wv\rangle}]  &=  \E_v[ \E_W[e^{i \langle u, Wv\rangle}| v]] \\
    &= \E_v[ e^{-\frac{\|u\|^2}{2}}] \\
    &= e^{-\frac{\|u\|^2}{2}}.\\
\end{align*}
This concludes the proof as the latter is the characteristic function of a random Gaussian vector with Identity covariance matrix.
\end{proof}

The next theorem shows when a stochastic process (ito)

\begin{thm}[Variation of Thm 8.4.3 in \cite{Oksendal2003}]\label{thm:same_law}
Let $(X_t)_{t\in [0,T]}$ and $(Y_t)_{t\in [0,T]}$ be two stochastic processes given by 
$$
\begin{cases}
dX_t = b(X_t) dt + \sigma(X_t) dB_t, \quad X_0 = x \in \reals.\\
dY_t = b_t dt + v_t d\hat{B}_t, \quad Y_0 = X_0,
\end{cases}
$$
where $\sigma: \reals \to \reals^{1 \times k}$, $(b_t)_{t \geq 0}$ and $(v_t)_{t \geq 0}$ are real valued adapted stochastic processes, and $v$ is adapted to the filtration of the Brownian motion $(\hat{B}_t)_{t \geq 0}$, $(B_t)_{t \geq 0}$ is an $k$-dimensional Brownian motion and $(\hat{B}_t)_{t \geq 0}$ is a $1$-dimensional Brownian motion. Assume that $\E[b_t | \mathcal{N}_t] = b(Y_t)$ where $\mathcal{N}_t = \sigma((Y_s)_{s \leq t})$ is the $\sigma$-Algebra generated by $\{ Y_s: s\leq t\}$, and $v_t^2 = \sigma(Y_t) \sigma(Y_t)^\top$ almost surely (in terms of $dt \times dP$ measure where $dt$ is the natural Borel measure on $[0,T]$ and $dP$ is the probability measure associated with the probability space). Then, $X_t$ and $Y_t$ have the same distribution for all $t \in [0,T]$.
\end{thm}
\begin{proof}
The proof of this theorem is the same as that of Thm 8.4.3 in 
\cite{Oksendal2003} with small differences. Indeed, our result is slightly different from that of \cite{Oksendal2003} in the sense that here we consider Brownian motions with different dimensions, while in their theorem, the author considers the case where the Brownian motions involved in $(X_t)$ and $(Y_t)$ are of the same dimension. However, both results make use of the so-called Martingale problem, which characterizes the weak uniqueness and hence the distribution of Ito processes\footnote{We omit the details on the Martingale problem here. We invite the curious reader to check Chapter 8 in \cite{Oksendal2003} for further details.}. The generator of $X_t$ is given for $f \in \mathcal{C}^2(\reals)$ by 
$$
\mathcal{G}(f)(x) = b(x) \frac{\partial f}{\partial x}  + \frac{1}{2} \sigma(x) \sigma(x)^\top \frac{\partial^2 f}{\partial x^2}.
$$
Now define the process $\mathcal{H}(f)$ for $f \in \mathcal{C}^2(\reals)$ by
$$
\mathcal{H}(f)(t) = b_t \frac{\partial f}{\partial x}(Y_t) + \frac{1}{2} v^2_t \frac{\partial^2 f}{\partial x^2}(Y_t).
$$

Let $\mathcal{N}_t = \sigma((Y_s)_{s \leq t})$ be the $\sigma$-Algebra generated by $\{ Y_s: s\leq t\}$. Using \ito lemma, we have that for $s < t$, 

\begin{align*}
\E[f(Y_s) | \mathcal{N}_t] &= f(Y_t) + \E[ \int_{t}^s \mathcal{H}(f)(r) dr | \mathcal{N}_t]\\
&= f(Y_t) + \E\left[ \int_{t}^s \E[\mathcal{H}(f)(r) | \mathcal{N}_r] dr | \mathcal{N}_t\right]\\
&= f(Y_t) + \E\left[ \int_{t}^s \mathcal{G}(f)(Y_r) dr | \mathcal{N}_t\right],\\
\end{align*}
where we have used the fact that $\E[b_r| \mathcal{N}_r] = b(Y_r)$. Now define the process $M$ by 
$$
M_t = f(Y_t) - \int_{0}^t \mathcal{G}(f)(Y_r) dr.
$$
For $s > t$, we have that 
\begin{align*}
\E[M_s | \mathcal{N}_t] &= f(Y_t) + \E\left[ \int_{t}^s \mathcal{G}(f)(Y_r) dr | \mathcal{N}_t \right] - \E\left[ \int_{0}^s \mathcal{G}(f)(Y_r) dr | \mathcal{N}_t \right] \quad \textup{(by \ito lemma),}\\
&= f(Y_t) - \E\left[ \int_{0}^t \mathcal{G}(f)(Y_r) dr | \mathcal{N}_t \right] = M_t.\\
\end{align*}
Hence, $M_t$ is a martingale (w.r.t to $\mathcal{N}_t$). We conclude that $Y_t$ has the same law as $X_t$ by the uniqueness of the solution of the martingale problem (see 8.3.6 in \cite{Oksendal2003}).
\end{proof}

The next result is a simple corollary of the existence and uniqueness of the strong solution of an SDE under the Lipschitz conditions on the drift and the volatility. It basically shows that a zero-drift process collapses (becomes constant) once the volatility is zero.

\begin{lemma}\label{lemma:zero_after_hit}
Let $g : \reals^n \to \reals$ be a Lipschitz function. Let $Z$ be the solution of the stochastic differential equation
$$
dZ_t = g(Z_t) dB_t, \quad Z_0 \in \reals^n.
$$
If $g(Z_0) = 0$, then $Z_t = Z_0$ almost surely.
\end{lemma}
\begin{proof}
This follows for the uniqueness of the strong solution of an SDE(\cref{thm:existence_and_uniqueness}).
\end{proof}

\subsection{Proof of \cref{prop:main_conv}}\label{sec:proof_main_conv}
We are now ready to prove the following result.\\

\textbf{Proposition \ref{prop:main_conv}. }\emph{
Assume that the activation function $\phi$ is Lipschitz on $\reals^n$. Then, in the limit $L \to \infty$, the process $X^L_t = Y_{\lfloor t L\rfloor}$, $t\in [0,1]$, converges in distribution to the solution of the following SDE 
\begin{equation}\label{eq:main_sde}
    dX_t = \frac{1}{\sqrt{n}}\|\phi(X_t)\| dB_t, \quad X_0 = W_{in} x,
\end{equation}
where $(B_t)_{t\geq 0}$ is a Brownian motion (Wiener process). Moreover, we have that for any $t \in [0,1]$ Lipschitz function $\Psi:\reals^n \to \reals$, 
$$
\E \Psi(Y_{\lfloor t L\rfloor}) = \E \Psi(X_t) + \bigO(L^{-1/2}),
$$
where the constant in $\bigO$ does not depend on $t$.\\
Moreover, if the activation function $\phi$ is only locally Lipschitz, then $X^L_t$ converges locally to $X_t$. More precisely, for any fixed $r > 0$, we consider the stopping times 
$$
\tau^L = \inf \{t \geq 0: \|X^L_t\| \geq r\}, \quad \tau = \inf \{t \geq 0: \|X_t\| \geq r\},
$$
then the stopped process $X^L_{t \land \tau^L}$ converges in distribution to the stopped solution $X_{t \land \tau}$ of the above SDE.}\\

\begin{proof}
The proof is based on \cref{thm:convergence_euler} in the appendix. It remains to express \cref{eq:resnet}  in the required form and make sure all the conditions are satisfied for the result to hold. Using \cref{lemma:gaussian_vec}, we can write \cref{eq:resnet} as
$$
Y_l = Y_{l-1} + \frac{1}{\sqrt{L}} \sigma(Y_{l-1}) \zeta^L_{l-1},
$$
where $\sigma(y) \overset{def}{=} \frac{1}{\sqrt{n}} \|\phi(y)\|$ for all $y \in \reals^n$ and $\zeta^L_l$ are \iid random Gaussian vectors with distribution $\normal(0, I)$. This is equal in distribution to the Euler scheme of SDE \cref{eq:main_sde}. Since $\sigma$ trivially inherits the Lipschitz or local Lipschitz properties of $\phi$, we conclude for the convergence using \cref{thm:convergence_euler} and \cref{lemma:locally_lipschitz_convergence}.\\

Now let $\Psi$ be $K$-Lipschitz for some constant $K > 0$. We have that 
$$
|\E \Psi(Y_{\lfloor t L\rfloor}) - \E \Psi(X_t)| \leq K \E \sup_{t \in [0,1]}\|\bar{Y}_{\lfloor t L\rfloor} - X_t\| = \bigO(L^{-1/2}),
$$
where $\bar{Y}$ is the Euler scheme as in \cref{thm:convergence_euler}, and where we have used the fact that $Y_{\lfloor t L\rfloor}$ and $\bar{Y}_{\lfloor t L\rfloor}$ have the same distribution.
\end{proof}

The result of \cref{prop:main_conv} can be generalized to the case with multiple inputs with minimal changes in the proof. We summarize this result in the next proposition.

\begin{prop}\label{prop:main_conv_multiple}
Let $x_1, x_2, \dots, x_k \in \reals^d$ be non-zero inputs, and denote by $Y_l(x_i)$ the pre-activation vector in layer $l$ for the input $x_i$. Consider the vector $\bm{Y}_l^k = (Y_l(x_1)^\top, Y_l(x_2)^\top, \dots, Y_l(x_k)^\top)^\top \in \reals^{k\cdot n}$ consisting of the concatenation of the pre-activations vectors for all inputs $x_i$. Assume that the activation function $\phi$ is Lipschitz on $\reals^n$. Then, in the limit $L \to \infty$, the process $\bm{X}^{L,k}_t = \bm{Y}^k_{\lfloor t L\rfloor}$, $t\in [0,1]$, converges in distribution to the solution of the following SDE 
\begin{equation}\label{eq:main_sde}
    d\bm{X}^k_t = \frac{1}{\sqrt{n}}\Sigma(\bm{X}^k_t)^{1/2} d\bm{B}_t, \quad \bm{X}^k_0 = ((W_{in} x_1)^\top, \dots, (W_{in} x_k)^\top)^\top,
\end{equation}
where $(\bm{B}_t)_{t\geq 0}$ is an $kn$-dimensional Brownian motion (Wiener process), independent from $W_{in}$, and $\Sigma(\bm{X}^k_t) $ is the covariance matrix given by 
\[
  \Sigma(\bm{X}^k_t) = \left[\begin{array}{ c | c | c | c}
    \alpha_{1,1} I_n & \alpha_{1,2} I_n & \dots & \alpha_{1,k} I_n\\
    \hline
    \alpha_{2,1} I_n & \alpha_{2,2} I_n & \dots & \alpha_{2,k} I_n\\
    \hline
    \vdots & \vdots & \vdots & \vdots\\
    \alpha_{k,1} I_n & \dots & \dots & \alpha_{k,k} I_n\\
  \end{array}\right],
\]
where $\alpha_{i,j} = \langle \phi(\bm{X}_{t}^{k, i}), \phi(\bm{X}_{t}^{k, j}) \rangle$, with $((X_{t}^{k, 1})^\top, \dots, (X_{t}^{k, k})^\top )^\top \overset{def}{=} \bm{X}_t^k$.
Moreover, if the activation function $\phi$ is only locally Lipschitz, then $\bm{X}^{L,k}_t$ converges locally to $\bm{X}_t^k$. More precisely, for any fixed $r > 0$, we consider the stopping times $
\tau^L = \inf \{t \geq 0: \|\bm{X}^{L,k}_t\| \geq r\}$, and $\quad \tau = \inf \{t \geq 0: \|\bm{X}^{k}_t\| \geq r\},
$
then the stopped process $\bm{X}^{L,k}_{t \land \tau^L}$ converges in distribution to the stopped solution $\bm{X}^k_{t \land \tau}$ of the above SDE.
\end{prop}

\begin{proof}
The proof is similar to that of \cref{prop:main_conv}. The only difference lies the definition of the Gaussian vector $\zeta^L_l$. In this case, we have for all $x_i$
$$
Y_l(x_i) = Y_{l-1}(x_i) + \frac{1}{\sqrt{L}} \frac{1}{\sqrt{n}} \zeta^L_{l-1}(Y_{l-1}(x_i)),
$$
where $\zeta^L_{l-1}(Y_{l-1}(x_i)) \overset{def}{=} \sqrt{n} W_l \phi(Y_{l-1}(x_i))$. Concatenating these identities yield
$$
\bm{Y}^k_l = \bm{Y}^k_{l-1} + \frac{1}{\sqrt{L}} \frac{1}{\sqrt{n}} \bm{\zeta}^L_{l-1},
$$
where $\bm{\zeta}^L_{l-1}$ is the concatenation of the vector $\zeta^L_{l-1}(Y_{l-1}(x_i))$ for $i = 1, \dots, k$. It is straightforward that the covariance matrix of the Gaussian vector $\bm{\zeta}^L_{l-1}$ is given by the matrix $\Sigma$ above (with $X$ replaced by $Y$). We conclude using \cref{thm:convergence_euler}.

\end{proof}

\section{Some technical results for the proofs}
\subsection{Approximation of $X$}
In the next lemma, we provide an approximate stochastic process $X^m$ to $X$, that differs from $X$ by the volatility term. The upper-bound on the $L_2$ norm of the difference between $X^m$ and $X$ will prove useful in the proofs of other results. The proof of this lemma requires the use of Gronwall's lemma, a tool that is often used in stochastic calculus.
\begin{lemma}\label{lemma:convergence_Xm_X}
Let $x \in \reals^d$ such that $x \neq 0$, $m\geq 1$ be an integer, and consider the two stochastic processes $X^m$ and $X$ given by 
$$
\begin{cases}
dX^m_t = \frac{1}{\sqrt{n}} \|\phi_m(X^m_t)\| dB_t \, ,\quad  t\in [0,\infty), \quad X^m_0 = W_{in} x, \\
dX_t = \frac{1}{\sqrt{n}} \|\phi(X_t)\| dB_t \, ,\quad  t\in [0,\infty), \quad X_0 = W_{in} x,
\end{cases}
$$

where $\phi_m(z) = \int_{0}^z h(mu) du$ where $h$ is the Sigmoid function given by $h(u) = (1 + e^{-u})^{-1}$, $\phi$ is the ReLU activation function, and $(B_t)_{t \geq 0}$ is an $n$-dimensional Brownian motion. We have the following

$$
\forall t\geq 0, \, \, \E \left\|X^m_t - X_t \right\|^2 \leq \frac{2n t}{m^2} e^{2t}.
$$

\end{lemma}

\begin{proof}
Let $t \geq 0$. We have that 
\begin{align*}
\E \left\|X^m_t - X_t \right\|^2 = \frac{1}{n} \E \left\| \int_{0}^t (\|\phi_m(X^m_s)\| - \|\phi(X_s)\|) dB_s. \right\|^2 
\end{align*}
Using \ito isometry and the fact that $(\|\phi_m(X^m_s)\| - \|\phi(X_s)\|)^2 \leq \|\phi_m(X^m_s) - \phi(X_s)\|^2$, we obtain 
\begin{align*}
\E \left\|X^m_t - X_t \right\|^2 &\leq \int_{0}^t \E \left\|  \phi_m(X^m_s) - \phi(X_s) \right\|^2 ds\\
&\leq 2 \int_{0}^t \E \left\|  \phi_m(X^m_s) - \phi(X^m_s) \right\|^2 ds + 2 \int_{0}^t \E \left\|  \phi(X^m_s) - \phi(X_s) \right\|^2 ds \\
&\leq \frac{2 n t}{m^2} +  2 \int_{0}^t \E \left\|  X^m_s -X_s \right\|^2 ds,\\
\end{align*}
where we have used \cref{lemma:diff_phi_m_phi} and the fact that ReLU is $1$-Lipschitz. We concldue using Gronwall's lemma.

\end{proof}

\subsection{Approximation of $\phi$}
The next lemma provides a simple upper-bound on the distance between the ReLU activation $\phi$ and an approximate function $\phi_m$ that converges to $\phi$ in the limit of large $m$.
\begin{lemma}\label{lemma:diff_phi_m_phi}
Consider the function $\phi_m(z) = \int_{0}^z h(mu) du$ where $z \in \reals$ where $m \geq 1$. We have that 
$$
\sup_{z \in \reals}|\phi_m(z) - \phi(z)| \leq \frac{1}{m}.
$$
\end{lemma}

\begin{proof}
Let $m\geq 1$ and $z \in \reals$. Assume that $z > 0$. We have that 
\begin{align*}
|\phi_m(z) - \phi(z)| &= \int_{0}^z \frac{e^{-m u}}{1 + e^{-m u}} du\\
 &\leq \int_{0}^z e^{-m u} du\\
 &= \frac{1}{m} (1 -  e^{-m z}) \leq \frac{1}{m}.\\
\end{align*}

For the case where $z \leq 0$, the proof is the same. We have that 
\begin{align*}
|\phi_m(z) - \phi(z)| &= \int_{0}^z \frac{e^{m u}}{1 + e^{m u}} du\\
 &\leq \int_{0}^z e^{m u} du\\
 &= \frac{1}{m} (1 -  e^{m z}) \leq \frac{1}{m},\\
\end{align*}
which concludes the proof.
\end{proof}

\subsection{Other lemmas}
The next lemma shows that the logarithmic growth factor $\log\left( \frac{\|\phi_m(X^m_t) \|}{\|\phi_m(X^m_0) \|}\right)$ converges to $\log\left( \frac{\|\phi(X_t) \|}{\|\phi(X_0) \|}\right)$ when $m$ goes to infinity, where the convergence holds in $L_1$. The key ingredient in the use of uniform integrability coupled with convergence in probability, which is sufficient to conclude on the $L_1$ convergence. This result will help us conclude in the proof of \cref{thm:norm_post_act}.
\begin{lemma}\label{lemma:convergence_L1}
Let $x \in \reals^d$ such that $x \neq 0$, $m\geq 1$ be an integer, and consider the two stochastic processes $X^m$ and $X$ given by 
$$
\begin{cases}
dX^m_t = \frac{1}{\sqrt{n}} \|\phi_m(X^m_t)\| dB_t \, ,\quad  t\in [0,\infty), \quad X^m_0 = W_{in} x, \\
dX_t = \frac{1}{\sqrt{n}} \|\phi(X_t)\| dB_t \, ,\quad  t\in [0,\infty), \quad X_0 = W_{in} x,
\end{cases}
$$

where $\phi_m(z) = \int_{0}^z h(mu) du$ where $h$ is the Sigmoid function given by $h(u) = (1 + e^{-u})^{-1}$, $\phi$ is the ReLU activation function, and $(B_t)_{t \geq 0}$ is an $n$-dimensional Brownian motion. Then, conditionally on the fact that $\|\phi(X_0)\| > 0$, we have that

$$
\forall t\geq 0, \, \, \log\left( \frac{\|\phi_m(X^m_t) \|}{\|\phi_m(X^m_0) \|}\right) \overset{L^1}{\longrightarrow} \log\left( \frac{\|\phi(X_t) \|}{\|\phi(X_0) \|}\right).
$$

\end{lemma}

\begin{proof}

Let $t > 0$. From \cref{lemma:convergence_Xm_X}, we know that $X^m$ converges in $L^2$ to $X$. Using \cref{lemma:diff_phi_m_phi} and the fact that ReLU is $1$-Lipschitz, we obtain 

$$
\E \| \phi_m(X^m_t) - \phi(X_t)\|^2 \leq \frac{2 n}{m^2} + 2 \E \|X^m_t - X_t\|^2,
$$
which implies that $\phi_m(X^m_t)$ converges in $L^2$ to $\phi(X_t)$. In particular, the convergence holds in probability.  Using this fact with the Continuous mapping theorem, we obtain that 
\begin{equation}\label{eq:conv_prob_phim}
    \forall t\geq 0, \, \, \log\left(\|\phi_m(X^m_t) \|\right) \overset{\mathbb{P}}{\longrightarrow} \log\left(\|\phi(X_t) \|\right).
\end{equation}

Let us show the following,

$$
\forall t\geq 0, \, \, \log\left( \frac{\|\phi_m(X^m_t) \|}{\|\phi_m(X^m_0) \|}\right) \overset{\mathbb{P}}{\longrightarrow} \log\left( \frac{\|\phi(X_t) \|}{\|\phi(X_0) \|}\right).
$$

Let $\epsilon > 0$ and $t >0$. We have

\begin{align*}
\mathbb{P}\left(\left|\log\left( \frac{\|\phi_m(X^m_t) \|}{\|\phi_m(X^m_0) \|}\right) -  \log\left( \frac{\|\phi(X_t) \|}{\|\phi(X_0) \|}\right)\right| \geq \epsilon\right) &\leq  \mathbb{P}\left(\left|\log\|\phi_m(X^m_t) \| -  \log\left(\|\phi(X_t) \|\right)\right| \geq \epsilon/2 \right)\\
&+  \mathbb{P}\left(\left|\log\|\phi_m(X^m_0) \| -  \log\left(\|\phi(X_0) \|\right)\right| \geq \epsilon/2 \right),
\end{align*}
where the first term converges to zero by \cref{eq:conv_prob_phim}, and the second term converges to zero by \cref{lemma:diff_phi_m_phi}. Hence, the convergence in probability holds.

To conclude, it suffices to show that the sequence of random variables $\left(Y^m_t = \log\left( \frac{\|\phi_m(X^m_t) \|}{\|\phi_m(X^m_0) \|}\right)\right)_{ m \geq 1}$ is uniformly integrable.\\
Let $K > 0$. From the proof of \cref{lemma:tau_general_zeta}, with $\zeta = \phi_m$, we have that 

$$
Y^m_t = \frac{1}{\sqrt{n}}\int_{0}^t \mu(X^m_s) ds + \frac{1}{2n} \int_0^t \sum_{i=1}^n \sigma_i(X^m_s) dB^i_s,
$$
where $\sigma_i(X^m_s) = \frac{|\phi_m'(X^{m,i}_s) \phi_m(X^{m,i}_s)|}{\|\phi_m(X^m_s)\|}$, and $\mu(X^m_s) = \frac{1}{2} \sum_{i=1}^n \left(\phi_m''(X^{m,i}_s) \phi_m(X^{m,i}_s) + \phi_m'(X^{m,i}_s)^2 \right) - \frac{\|\phi_m'(X^m_s) \circ \phi_m(X^m_s)\|^2}{\|\phi_m(X^m_s)\|^2}$. Therefore, 
\begin{equation}
\begin{aligned}
\E|Y^m_t|^2 &= \frac{1}{n} \E \left( \int_{0}^t \mu(X^m_s) ds \right)^2 + \frac{1}{4n^2} \E \int_0^t \sum_{i=1}^m \sigma_i(X^m_s)^2 ds\\
&\leq \frac{t}{n}  \int_{0}^t \E \mu(X^m_s)^2 ds  + \frac{1}{4n^2} \E \int_0^t \sum_{i=1}^m \sigma_i(X^m_s)^2 ds,
\end{aligned}
\end{equation}
where we have used the \ito isometry and Cauchy-Schwartz inequality. Using the conditions on $\phi_m$, it is straightforward that term $\frac{1}{4n^2} \E \int_0^t \sum_{i=1}^m \sigma_i(X^m_s)^2 ds$ is uniformly bounded. It remains to bound the first term. Similarly to the proof of \cref{thm:norm_post_act}, we condition on the regions of $|X^{m,i}_s|$ and obtain that the terms $\E \mu(X^m_s)^2$ are uniformly bounded over $m$ (we omit the proof here as it is just a repetition of the techniques used in the proof of \cref{thm:norm_post_act}). Therefore, we have that $\sup_{m \geq 1} \E|Y^m_t|^2 < \infty$, which implies uniform integrability. This concludes the proof.
\end{proof}

\begin{lemma}\label{lemma:uppder_bound_delta}
Let $x \in \reals^d$ such that $x \neq 0$, $m\geq 1$ be an integer, and consider the stochastic processes $X$ given by 
$$
dX_t = \frac{1}{\sqrt{n}} \|\phi(X_t)\| dB_t \, ,\quad  t\in [0,\infty), \quad X_0 = W_{in} x,
$$

where $\phi$ is the ReLU activation function, and $(B_t)_{t \geq 0}$ is an $n$-dimensional Brownian motion independent from $X_0$. Then, conditionally on the fact that $\|\phi(X_0)\| > 0$, we have that for all $s\in [0,1], i \in [n]$

$$
\mathbb{P}(|X^i_s| \leq \delta) = \mathcal{O}_{\delta \to 0}(\delta),
$$
where the bound holds uniformly over $s \in [0,1]$.
\end{lemma}

\begin{proof}
We have that 
$$
X^i_s = X^i_0 + \frac{1}{\sqrt{n}}\int_{0}^s \|\phi(X_u)\| dB^i_u.
$$
Since $\phi(X_u) > 0$ for all $u \geq 0$ almost surely, and by the independence of $B$  and $X_0$, we can easily see that $X^i_s$ has no Dirac mass and is the sum of two continuous random variables (not independent) $X^i_0$ and $ \frac{1}{\sqrt{n}}\int_{0}^s \|\phi(X_u)\| dB^i_u$ that have bounded density functions, and thus $X^i_s$ has a bounded density function $h_s$. Hence, writing $\mathbb{P}(|X^i_s |\leq \delta) = \int_{-\delta}^\delta h_s(t) dt = \mathcal{O}(\delta)$ concludes the proof. The bound can be taken uniformly over $s \in [0,1]$ by taking $\sup_{s \in [0,1] |h_s(0)|}$.
\end{proof}

\section{The Ornstein-Uhlenbeck (OU) process}\label{sec:ou_process}
The OU process is the (unique) strong solution to the following diffusion
\begin{equation}\label{eq:OU}
    dX_t = a (b - X_t) dt + \sigma dB_t,
\end{equation}
where $a,b,\sigma \in \reals$ are constants, and $B$ is a one dimensional Brownian motion. In financial mathematics, the OU process is used as a model of short-term interest rate under the name of the Vasicek model. The OU process has a closed-form expression and its marginal distribution is Gaussian. The next lemma gives a full characterization of the marginal distributions of an OU process.

\begin{lemma}\label{lemma:ou}
\cref{eq:OU} admits the following solution
$$
X_t = X_0 e^{- a t} + b (1  - e^{-a t}) + \sigma \int_{0}^t e^{-a (t- s)} dB_s.
$$
As a result, we have the following
\begin{itemize}
    \item $X_t$ is Gaussian.
    \item $\E[X_t] =  X_0 e^{- a t} + b (1  - e^{-a t}).$
    \item $\textup{Cov}(X_t, X_s) = \frac{\sigma^2}{2 a} \left( e^{- a |t - s|} - e^{- a (t + s)} \right) $.
\end{itemize}
\end{lemma}

\begin{proof}
Consider the process $Z_t = e^{at} X_t$, using \ito lemma, we have that 
\begin{align*}
    dZ_t &= a Z_t dt + e^{at} dX_t \\
    &= ab e^{at} dt +  \sigma e^{at} dW_t.
\end{align*}
Integrating between $0$ and $t$ yields
$$
Z_t = Z_0  + b (e^{at} - 1) + \sigma \int_{0}^t e^{a s} dW_s.
$$
We conclude by multiplying both sides with $e^{-at}$.\\

The result for $\E[X_t]$ is straightforward since $\E \left[ \int_{0}^t e^{-a (t- s)} dW_s \right] = 0$ by the properties of \ito integral and the Brownian motion. For the covariance, without loss of generality assume that $t > s \geq 0$. We have that 
\begin{align*}
\textup{Cov}(X_t, X_s) &= \sigma^2 \E \left[ \int_{0}^t e^{-a (t- u)} dW_u \int_{0}^s e^{-a (s- u)} dW_u\right]\\
&= \sigma^2 \E \left[ \int_{0}^s e^{-a (t- u)} dW_u \int_{0}^s e^{-a (s- u)} dW_u\right]\\
&= \sigma^2 \int_{0}^s e^{-2a (\frac{s + t}{2} - u)} du\\
&= \frac{\sigma^2}{2a} \left(e^{-a (t- s)} - e^{-a (t + s)}\right),
\end{align*}
which completes the proof.\\
\end{proof}

We would like to find sufficient conditions on the activation function $\phi$ and a function $g$ such that the process $g(X_t)$ (\cref{eq:main_sde}) follows an the OU dynamics. For this purpose, we proceed by reverse-engineering the problem; Using \ito's lemma (\cref{eq:ito_lemma_1d}), this is satisfied when there exist constants $a,b,\sigma$ such that 
$$
\begin{cases}
\mu(y) = \frac{1}{2n} \phi(y)^2 g''(y) = a(b -g(y)) \\
\sigma(y) = \frac{1}{\sqrt{n}}\phi(y) g'(y) = \sigma.\\
\end{cases}
$$
This implies that $\frac{g''(y)}{g'^2(y)} = 2 a \sigma^{-2} (b - g(y))$. Letting $G = \int g$ be the primitive function of $g$, we obtain that $G$ satisfies a differential equation of the form
$$
\frac{1}{G''(y)} = \alpha y + \beta G(y) + \zeta,
$$
where $\alpha, \beta, \zeta \in \reals$ are constants.\\

Let us consider the case where $\alpha = \zeta = 0$ and $\beta \neq 0$, i.e. $\frac{1}{G''(y)} = \beta G(y)$. Equivalently, we solve the differential equation $G''(y) = \frac{\beta}{G}$ where $\beta \in \reals$. Multiplying both sides by $G'$ and integrating we obtain $\frac{1}{2} G'(y)^2 = \beta \log(|G|) + \gamma$. A sufficient condition for this to hold is to have $G>0$ and $G$ satisfies
$$
\frac{G'(y)}{\sqrt{\log(G) + \gamma}} = \zeta
$$
for some constants $\zeta, \gamma$. Integrating the left-hand side yields 
$$
\int^y \frac{G'(u)}{\sqrt{\log(G(u)) + \gamma}} du = \int^{G(y)} \frac{1}{\sqrt{\log(u) + \gamma}} du = \alpha \,  \textrm{Erfi}(\sqrt{\log(G(y)) + \gamma}) + \beta.
$$
where $\textup{Erfi}$ is the imaginary error function\footnote{Although the name might be misleading, the imaginary error function is real when the input is real.} given by 
$$
\textup{Erfi}(z) = \frac{2}{\sqrt{\pi}} \int_{0}^z e^{t^2} dt.
$$
To alleviate the notation, we denote $h := \textrm{Erfi}$ in the rest of this section. From the above, $G$ should have the form 
$$
G(y) = \exp\left( \zeta  + \left( h^{-1}(\alpha y + \beta)\right)^2\right),
$$
where $\alpha, \beta, \zeta$ are all constants, and $h^{-1}$ is the inverse function of the imaginary error function. We conclude that the activation function $\phi$ should have the form 
$$
\phi(y) = \frac{2 \sigma}{\alpha^2 \pi} \exp( -\zeta + h^{-1}(\alpha y + \beta)^2).
$$

In this case, the coefficients $a$ and $b$ are given by 
$$
b = 0, a = \frac{\sigma^2}{\alpha^2 \pi} \exp(-2 \zeta).
$$
Letting $g = G'$, the process $g(X_t)$ has the following dynamics
$$
d g(X_t) = - a g(X_t) dt + \sigma dB_t, 
$$
Hence $g(X_t)$ is an OU process, and we can conclude that the network output in the infinite-depth limit $X_1$ satisfies
$$
g(X_1) \sim \normal\left( g(X_0) e^{-a}, \frac{\sigma^2}{2a}\left( 1 - e^{-2a}\right) \right).
$$
We can then infer the distribution of $X_1$ by a simple change of variable. Note that this distribution is non-trivial, and unlike the infinite-width limit of the same ResNet (\cite{hayou21stable}) where the distribution is Gaussian, here the distribution of the pre-activations is directly impacted by the choice of the activation function $\phi$. \\

However, with this particular choice of the activation function $\phi$, the existence of the process $X$ can only be proven in the local sense, because $\phi$ is only locally Lipschitz. Let us first show this in the next lemma. We will see how we can mitigate this issue later.

\begin{lemma}
Let $\phi : \reals \to \reals$ defined by $$\phi(y) = \exp(h^{-1}(\alpha y + \beta)^2),$$
where $\alpha, \beta \in \reals$ are two constants.

We have that $\phi$ is locally Lipschitz, meaning that for any compact set $K \subset \reals$, there exists $C_K$ such that 
$$
\forall x,x' \in K, |\phi(x') - \phi(x) | \leq C_K |x' - x|.
$$
\end{lemma}
\begin{proof}
It suffices to show that the derivative of $\phi$ is locally bounded to conclude. We have that 
\begin{align*}
\phi'(y) &= 2 \alpha (h^{-1})'(\alpha y + \beta) h^{-1}(\alpha y + \beta) \exp(h^{-1}(\alpha y + \beta)^2)\\
&= \alpha \sqrt{\pi} \, h^{-1}(\alpha y + \beta).\\
\end{align*}
Since $h^{-1}$ is continuous on $\reals$, then  $\phi'$ is bounded on any compact set of $\reals$, which concludes the proof.
\end{proof}

Now we can rigorously prove the following result.
\begin{prop}
Let $x \in \reals$ such that $x \neq 0$. Consider the following activation function $\phi$
$$
\phi(y) = \frac{2 \sigma}{\alpha^2 \pi} \exp( -\zeta + h^{-1}(\alpha y + \beta)^2),
$$
where $\alpha, \beta \in \reals$ and $\sigma, \zeta >0$ are constants. Let $g$ be the function defined by 
$$
g(y) = \alpha \sqrt{\pi} \exp(\zeta) h^{-1}(\alpha y + \beta).
$$
Consider the stochastic process $X_t$ defined by 
$$
dX_t = |\phi(X_t)| dB_t, \quad X_0 = W_{in} x.
$$

Then, we have that for all $t \in [0,1]$,
$$
g(X_t) \sim \mathcal{N}\left( g(X_0) e^{-a t}, \frac{\sigma^2}{2 a} ( 1 - e^{-2a t})\right),
$$
where $a = \frac{\sigma^2}{\alpha^2 \pi} \exp(-2 \zeta)$.
\end{prop}

\begin{proof}
For $N>0$, consider the stopping time $\tau_N$ defined by 
$$
\tau_N = \inf \{t \geq 0: |X_t| \geq N\}.
$$
Using the continuity of paths of $X$, it is straightforward that $\lim_{N \to \infty} \tau_N = \infty$ almost surely. Let $N>0$ be large enough. The SDE satisfied by the process $X$ has a unique strong solution for $t \in [0, \tau_N)$ since the activation function $\phi$ is Lipschitz on the interval $(-N, N)$.
By applying \ito lemma for $t \in (0,\tau_N)$, we have that
$$
d g(X_t) = -a g(X_t) dt + \sigma dB_t,
$$
(from previous results). Using the fact that $\lim_{N \to \infty} \tau_N = \infty$ almost surely, and taking $N$ large enough, we obtain that for all $t \in (0,1]$, we have that 
$$
d g(X_t) = -a g(X_t) dt + \sigma dB_t,
$$
we conclude using \cref{lemma:ou}.

\end{proof}

\section{The Geometric Brownian Motion (GBM)}\label{sec:gbm}
The GBM dynamics refers to stochastic differential equations of the form 
\begin{equation}\label{eq:gbm}
    dX_t = a X_t dt + \sigma X_t dB_t,
\end{equation}
where $a, \sigma$ are constants and $B$ is a one dimensional Brownian motion. This SDE played a crucial role in financial mathematics and is often used as a model of stock prices. It admits a closed-form solution given in the next lemma.

\begin{lemma}\label{lemma:gbm}
\cref{eq:gbm} admits the following solution

$$
X_t = X_0 \exp\left( \left(a - \frac{1}{2} \sigma^2 \right) t + \sigma B_t\right).
$$
The distribution of $X_t$ is known as a log-Gaussian distribution. Moreover, the solution is unique.
\end{lemma}
\begin{proof}
The existence and uniqueness of the solution follows from \cref{thm:existence_and_uniqueness}. Indeed, it suffices to have the drift and the volatility both Lipschitz to obtain the result. This is satisfied in the case of GBM. Now consider the process $Z_t = \log(X_t)$. Using \ito lemma\footnote{Notice that here, $X_t$ should be positive in order to consider $\log(X_t)$. This is easy to show and the proof is similar to that of \cref{lemma:gbm_non_zero_1d}.}, it is easy to verify that 
$$
dZ_t = \left(a - \frac{1}{2} \sigma^2 \right) dt + \sigma dB_t,
$$
we conclude by integrating both sides.
\end{proof}

Now let us find sufficient conditions under which the infinite-depth network represented by the process $X$ has a GBM behaviour. In order for this to hold, it suffices to have 
$$
\begin{cases}
\mu(y) = \frac{1}{2n} \phi(y)^2 g''(y) = a g(y) \\
\sigma(y) = \frac{1}{\sqrt{n}}\phi(y) g'(y) = \sigma g(y).\\
\end{cases}
$$

This implies $\frac{g''}{g'^2}  \propto \frac{1}{g} $, or equivalently $\frac{g''}{g'}  \propto \frac{g'}{g} $, which in turn yields $\log(|g'|) = \alpha \log(|g|) + \beta$, and therefore $|g'| \propto |g|^\zeta$. Assuming that $g',g >0$, we can easily verify that functions of the form $g(y) = \alpha (y + \beta)^\gamma$ where $\alpha, \beta, \gamma> 0$ satisfy the requirements. Hence, the activation function should satisfy $\phi(y) = \sigma \gamma^{-1} (y + \beta) $, i.e. the activation should be linear. In this case, we have $a = \frac{1}{2} \sigma^2 \gamma^{-1}(\gamma -1)$ and the process $g(X_t)$ has the following GBM dynamics 
$$
d g(X_t) = a g(X_t) dt + \sigma g(X_t) dB_t.
$$
From \cref{lemma:gbm}, we conclude that 
$$
g(X_1) \sim g(X_0) \exp\left( \left(a - \frac{1}{2} \sigma^2 \right) t + \sigma B_1\right).
$$
Observe that in the special case of $\gamma =1, \beta =0, \alpha =1$, we have $g(y)=y$ and $a = 0$. In this case, we obtain $Y_1 \sim Y_0 \exp\left(- \frac{1}{2} \sigma^2 t + \sigma B_1\right)$.\\

We summarize the previous results in following proposition.

\begin{prop}
Let $x \in \reals$ such that $x \neq 0$. Consider the following activation function $\phi$
$$
\phi(y) = \alpha y + \beta,
$$
where $\alpha> 0, \beta \in \reals$ are constants. Let $\sigma>0$ and define the function $g$ by 
$$
g(y) = (\alpha y + \beta)^{\gamma}.
$$
where $\gamma =\sigma  \alpha^{-1}$. Consider the stochastic process $X_t$ defined by 
$$
dX_t = |\phi(X_t)| dB_t, \quad X_0 = W_{in} x.
$$

Then, the process $g(X_t)$ satisfies the following GBM dynamics
$$
d g(X_t) = a g(X_t) dt + \sigma g(X_t) dB_t,
$$
where $a = \frac{1}{2} \sigma^2 \gamma^{-1} (\gamma -1)$. As a result, we have that for all $t \in [0,1]$,
$$
g(X_t) \sim g(X_0) \exp\left( \left(a - \frac{1}{2}\sigma^2\right)t +  \sigma B_t\right).
$$

\end{prop}


\section{ReLU in the case $n=d=1$}\label{sec:relu_1d}
Consider the process $X$ given by the SDE
$$
dX_t = \phi(Y_t) dB_t, \quad t \in [0,1], X_0>0.
$$
where $\phi(z) = \max(z,0)$ for $z \in \reals$ is the ReLU activation function. Note that we assume $X_0>0$ in this case. We will deal with the general case later in this section.\\

It is straightforward that if $X_s \leq 0$ for some $s \in [0,1]$, then for all $t \geq s, X_t = X_s$. This is because $dX_t = 0 \times dB_t$ whenever $X_t \leq 0$. A rigorous justification is provided in \cref{lemma:constant_after_hit}. Hence, the event $\{X_s \leq 0\}$ constitutes a stopping event where the process becomes constant. We also say that $0$ is an absorbent point of the process $X$. A classic tool in stochastic calculus to deal with such situations is the notion of  \emph{stopping time} which is a random variable that depend on the trajectory of $X$ (or equivalently on the natural filtration $\mathcal{F}_t$ associated with the Brownian motion $B$). Consider the following stopping time
\begin{equation}\label{eq:tau_relu_1d}
    \tau = \inf \{ t \in [0,1], \textup{ s.t. } X_s \leq 0 \}.
\end{equation}
Observe that we have for all $t \in [0,\tau]$
$$
dX_t = X_t dB_t,
$$
which implies that $Y_t$ is a Geometric Brownian motion in the interval $[0,\tau]$. Hence, if $\tau > 1$ (a.s.), the network output has also a log-normal distribution in the infinite-depth limit. In the next lemma, we show that $\tau = \infty$ with probability $1$ which confirms the above.

\begin{lemma}\label{lemma:gbm_non_zero_1d}
Let $\tau$ be the stopping time defined by \cref{eq:tau_relu_1d}. We have that
$$
\mathbb{P}(\tau = \infty) = 1.
$$
\end{lemma}
\begin{proof}
By continuity of the Brownian path and the ReLU function $\phi$, the paths of the process $X$ are also continuous\footnote{This is a classic result in stochastic calculus. More rigorously, $X$ can be chosen to have continuous paths with probability 1.}. we have that $\tau > 0$ almost surely.
From the observation above, taking the limit $t \to \tau^{-}$ and using the continuity, we obtain 
$$
X_\tau = X_0 \exp\left( -\frac{1}{2}\tau  + B_\tau \right).
$$
For some $\omega \in \{\tau < \infty\}$, we have that $X_\tau(\omega) = 0$ (by continuity). Hence $-\frac{1}{2} \tau(w) + X_\tau(\omega) = -\infty$. This happens with probability zero, which means that the event $\{\tau < \infty\}$ has probability zero. This concludes the proof.
\end{proof}

Hence, with the ReLU activation function, given $X_0 > 0$, the network output is distributed as 
$$
X_1 \sim X_0 \exp\left( -\frac{1}{2}  + B_1 \right).
$$

Now let us go back to the original setup for $X_0$. Recall that $X_0 = W_{in} x$ for some $x \neq 0$ and $W_{in} \sim \normal(0,1)$. By conditioning on $X_0$ and observing that $0$ is an absorbent point of the process $X$, we obtain that

$$
X_1 \sim \ind_{\{X_0 >0\}} X_0 \exp\left( -\frac{1}{2}  + B_1 \right) + \ind_{\{X_0 \leq 0\}} X_0.
$$

We summarize these results in the next proposition.

\begin{prop}
Let $x \in \reals$ such that $x \neq 0$, and let $\phi$ be the ReLU activation function given by $\phi(z) = \max(z,0)$ for all $z\in \reals$.
 Consider the stochastic process $X_t$ defined by 
$$
dX_t = \phi(X_t) dB_t, \quad X_0 = W_{in} x.
$$

Then, the process $X$ is a mixture of a Geometric Brownian motion and a constant process. More precisely, we have for all $t \in [0,1]$

$$
X_t \sim \ind_{\{X_0 >0\}} \,  X_0 \exp\left( -\frac{1}{2} t  + B_t \right) + \ind_{\{X_0 \leq 0\}} X_0.
$$
Hence, conditionally on $X_0>0$, the process $X$ is a Geometric Bronwian motion.

\end{prop}

\section{Proof of \cref{lemma:tau_general_zeta} and \cref{lemma:tau_genera_n_relu}}\label{sec:proofs_lemmas_tau}
\textbf{Lemma \ref{lemma:tau_general_zeta}.} \emph{
Let $x \in \reals^d$ such that $x \neq 0$, and consider the stochastic process $X$ given by the following SDE
$$
dX_t =\frac{1}{\sqrt{n}} \|\phi(X_t)\| dB_t \, ,\quad  t\in [0,\infty), \quad X_0 = W_{in} x, 
$$
where $\phi(z) : \reals \to \reals$ is Lipschitz, injective, $\mathcal{C}^2(\reals)$ and satisfies $\phi(0)=0$, and $\phi'$ and $\phi'' \phi$ are bounded on $\reals$, and $(B_t)_{t \geq 0}$ is an $n$-dimensional Brownian motion independent from $W_{in} \sim \normal(0, d^{-1} I)$. Let $\tau$ be the stopping time given by 
$$
\tau = \min \{ t\geq 0 : \phi(X_t)\ = 0\}.
$$
Then, we have that 
$$
\mathbb{P}\left(\tau = \infty \right) = 1.
$$}\\

\begin{proof}
It is straightforward that with probability $1$ we have $\|\phi(X_0)\| > 0$, which implies that with probability $1$, $\tau > 0$. Let $t < \tau$. Using \ito's lemma with the function $g(z) = \frac{1}{2} \log( \| \zeta(x)\|^2)$, we obtain

$$
d g(X_t) = \nabla g(X_t)^\top dX_t + \frac{1}{2n} \|\zeta(X_t)\|^2 \Tr(\nabla^2 g(X_t)) dt.
$$
Therefore, 
$$
g(X_t) - g(X_0) = \frac{1}{\sqrt{n}}\int_{0}^t \mu(X_s) ds + \frac{1}{2n} \int_0^t \sum_{i=1}^n \sigma_i(X_s) dB^i_s,
$$
where $\sigma_i(X_s) = \frac{|\phi'(X^i_s) \phi(X^i_s)|}{\|\phi(X_s)\|}$, and $\mu(X_s) = \frac{1}{2} \sum_{i=1}^n \left(\phi''(X^i_s) \phi(X^i_s) + \phi'(X^i_s)^2 \right) - \frac{\|\phi'(X_s) \circ \phi(X_s)\|^2}{\|\phi(X_s)\|^2}$, and $\circ$ refers to the Hadamard product of vectors, i.e. coordinate-wise product.

For some $\omega \in \{ \tau < \infty\}$, using the path continuity of the process $X$ and the continuity of $g$, we have that $\lim_{t \to \tau(\omega)^-} g(X_{\tau(\omega)}(\omega)) = -\infty$. Therefore, we should also have 
$$
\frac{1}{\sqrt{n}}\int_{0}^{\tau(\omega)} \mu(X_s(\omega)) ds + \frac{1}{2n} \int_0^{\tau(\omega)} \sum_{i=1}^n \sigma_i(X_s(\omega)) dB^i_s(\omega) = - \infty.
$$
Hence, we have that
\begin{align*}
\mathbb{P}\left(\tau < \infty \right) &\leq \mathbb{P}\left( \frac{1}{\sqrt{n}}\int_{0}^t \mu(X_s) ds + \frac{1}{2n} \int_0^t \sum_{i=1}^n \sigma_i(X_s) dB^i_s = - \infty\right)\\
&= \lim_{A \to \infty} \mathbb{P}\left( \frac{1}{\sqrt{n}}\int_{0}^t \mu(X_s) ds + \frac{1}{2n} \int_0^t \sum_{i=1}^n \sigma_i(X_s) dB^i_s \leq - A\right)\\
&= \lim_{A \to \infty} \mathbb{P}\left( \frac{1}{\sqrt{n}}\int_{0}^t \mu(X_s) ds + \frac{1}{2n} \int_0^t \sigma(X_s) d\hat{B}_s \leq -A\right),\\
\end{align*}
where $\hat{B}$ is a one-dimensional Brownian motion, and where we use \cref{thm:same_law}, and $\sigma(X_s) = (\sum_{i=1}^n \sigma_i(X_s)^2)^{1/2} =  \frac{\| \phi'(X_s) \circ \phi(X_s)\|}{\|\phi(X_s)\|}$. Using the conditions on $\phi$, there exists a constant $K > 0$ such $|\phi'| \leq K, |\phi'' \phi| \leq K$. With this we obtain for all $Z \in \reals^n$
$$
|\sigma(Z)| = \frac{\| \phi'(Z) \circ \phi(Z)\|}{\|\phi(Z)\|} \leq K,
$$
and 
$$
|\mu(Z)| = \left|\frac{1}{2} \sum_{i=1}^n \left(\phi''(Z^i) \phi(Z^i) + \phi'(Z^i)^2 \right) - \frac{\|\phi'(Z) \circ \phi(Z)\|^2}{\|\phi(Z)\|^2} \right| \leq \frac{1}{2} n K +  \left(\frac{1}{2} n + 1\right) K^2.$$

Hence, the random variable $\frac{1}{\sqrt{n}}\int_{0}^t \mu(X_s) ds + \frac{1}{2n} \int_0^t \sigma(X_s) d\hat{B}_s$ is finite with probability $1$. We conclude that 
$$
\mathbb{P}\left(\tau  = \infty \right) = 1.
$$

\end{proof}

\noindent\textbf{Lemma \ref{lemma:tau_genera_n_relu}. }\emph{ 
Consider the stochastic process \eqref{eq:main_sde} given by the SDE
$$
dX_t = \frac{1}{\sqrt{n}} \|\phi(X_t)\| dB_t \, ,\quad  t\in [0,\infty), \quad X_0 = W_{in} x, 
$$
where $\phi$ is the ReLU activation function, and $(B_t)_{t \geq 0}$ is an $n$-dimensional Brownian motion. Let $\tau$ be the stopping time given by 
$$
\tau = \min \{ t\geq 0 : \|\phi(X_t)\| = 0\} = \min \{ t\geq 0 : \,  \forall i \in [n], \, X^i_t \leq 0\}.
$$
Then, we have that 
$$
\mathbb{P}\left(\tau = \infty \, \huge| \, \|\phi(X_0)\| > 0\right) = 1.
$$
As a result, we have that 
$$
\mathbb{P}(\tau = \infty) = 1 - 2^{-n}.
$$}\\

\begin{proof}
Let $t_0 > 0$. Using \cref{lemma:zero_after_hit}, we know that if for some $t_1$, $\|\phi(X_{t_1})\| = 0$, then for all $t \geq t_1$, we have that $X_t = X_{t_1}$ and $\|\phi(X_t)\| = 0$. Hence, we have that 
$$
\mathbb{P}\left(\tau \leq t_0  \, \huge| \, \|\phi(X_0)\| > 0\right) = \mathbb{P}\left(\|\phi(X_{t_0})\| = 0  \, \huge| \, \|\phi(X_0)\| > 0\right).
$$

Let $m \geq 1$ and consider the function $\phi_m(z) = \int_{0}^z h(m \,u) du$ and $h(t) = (1 + e^{-t})^{-1}$ is the Sigmoid function\footnote{Note that $\phi_m$ has a closed-form formula given by $\phi_m(z) = m^{-1} (\log(1 + e^{mz}) - \log(2))$, which can be seen as a shifted and scaled version of the Softplus function. However, we do not need the closed-form formula in our analysis.}. It is straightforward that $\phi_m$ satisfies the conditions of \cref{lemma:tau_general_zeta}. Let $X^m$ be the solution of the following SDE (the solution exists and is unique since $\phi_m$ is trivially Lipschitz)
$$
dX^m_t = \frac{1}{\sqrt{n}} \|\phi_m(X^m_t)\| dB_t \, ,\quad  t\in [0,\infty), \quad X_0 = W_{in} x.
$$
We know from \cref{lemma:convergence_Xm_X} that $X^m_t$
converges in $L^2$ to $X_t$ (uniformly over $t \in [0,T]$ for any $T>0$). In particular, this implies convergence in distribution. Moreover, observe that for all $t$

$$
\E \| \phi_m(X^m_t) - \phi(X_t)\|^2 \leq \frac{2 n}{m^2} + 2 \E \|X^m_t - X_t\|^2,
$$
where we used triangular inequality and the upperbound from \cref{lemma:diff_phi_m_phi}. Thus, we have that $\phi_m(X^m_t)$ converges in $L^2$ (and in distribution) to $\phi(X_t)$. 

Let $\delta_k = [1/(k+1), 1/{k})$ for $k \geq 1$, and define $\delta_0 = [1,\infty)$. For $m\geq 1$, using \cref{lemma:diff_phi_m_phi}, we have that
\begin{equation*}
\mathbb{P}\left(\|\phi(X_{t_0})\| = 0  \, \huge\cap \, \|\phi(X_0)\| > 0\right) \leq \sum_{k=0}^\infty \mathbb{P}\left(\|\phi_m(X_{t_0})\| \leq 1/m  \, \huge\cap \, \|\phi(X_0)\|\in \delta_k\right). \\
\end{equation*}
Given $k \geq 0$, we have that for $m > n^{1/2} (k+1)$,
\begin{equation}\label{eq:upper_bound_prob_phi}
\begin{aligned}
        \mathbb{P}\left(\|\phi_m(X_{t_0})\| \leq 1/m  \, \huge\cap \, \|\phi(X_0)\|\in \delta_k\right) &\leq \mathbb{P}\left(\|\phi_m(X^m_{t_0})\| \leq 1/m + \log(m)/m  \, \huge\cap \, \|\phi(X_0)\|\in \delta_k\right)\\
    &+ \mathbb{P}\left(\|\phi_m(X^m_{t_0}) - \phi(X_{t_0})\| > \log(m)/m  \, \huge\cap \, \|\phi(X_0)\|\in \delta_k\right)
\end{aligned}
\end{equation}

Let us deal with the first term. Using \cref{lemma:diff_phi_m_phi}, we have that 

\begin{equation*}
\begin{aligned}
    &\mathbb{P}\left(\|\phi_m(X^m_{t_0})\| \leq 1/m + \log(m)/m  \, \huge\cap \, \|\phi(X_0)\|\in \delta_k \right) \\
   & \leq \mathbb{P}\left(\|\phi_m(X^m_{t_0})\| \leq 1/m + \log(m)/m  \, \huge\cap \, \|\phi_m(X^m_0)\| \geq 1/(k+1) - n^{1/2} m^{-1}  \huge\cap \, \|\phi(X_0)\|\in \delta_k\right)\\
   & \leq \mathbb{P}\left(\log\left(\frac{\|\phi_m(X^m_{t_0})\|}{\|\phi_m(X^m_{0})\|}\right) \leq -\log(m / (1 + \log(m)) + \log((k+1)^{-1} - m^{-1} n^{1/2}) \,  \huge\cap \, \|\phi(X_0)\|\in \delta_k\right)
\end{aligned}
\end{equation*}

From  \cref{lemma:convergence_L1}, we know that the random variable $\log\left(\frac{\|\phi_m(X^m_{t_0})\|}{\|\phi_m(X^m_{0})\|}\right)$ converges in $L^1$ and thus it is bounded in $L^1$ norm (over $m$). Therefore, a simple application of Markov's inequality yields that the probability above goes to $0$ when $m$ goes to $\infty$.

The second term in \cref{eq:upper_bound_prob_phi} also converges to $0$ using the $L^2$ convergence of $\phi_m(X^{m}_{t_0})$ to $\phi(X_{t_0})$ coupled with a simple application of Markov's inequality. We therefore obtain that for all $k \geq 0$, $\lim_{m \to \infty} \mathbb{P}\left(\|\phi_m(X_{t_0})\| \leq 1/m  \, \huge\cap \, \|\phi(X_0)\|\in \delta_k\right) =0$. Using the Dominated convergence theorem, we obtain that for all $t_0$, 
$$
\mathbb{P}\left(\|\phi(X_{t_0})\| = 0  \, \huge\cap \, \|\phi(X_0)\| > 0\right) = 0,
$$
which implies that $\mathbb{E}\left[ \ind_{\{\tau > t\}} \huge| \|\phi(X_0)\| > 0\right] = 1$ for all $t > 0$. Another application of the Dominated convergence theorem yields the result.
The second part is straightforward by observing that $\mathbb{P}\left( \|\phi(X_0)\| = 0\right) = 2^{-n}$.
\end{proof}

\section{Proof of \cref{thm:norm_post_act}}\label{sec:proof_main_thm}

\textbf{Theorem \ref{thm:norm_post_act}. }
\emph{We have that for all $t \in [0,1]$,
$$
\|\phi(X_t)\| = \|\phi(X_0)\| \exp\left(\frac{1}{\sqrt{n}} \hat{B}_t + \frac{1}{n} \int_{0}^t \mu_s ds\right), \quad \textrm{almost surely,}
$$
where $\mu_s =  \frac{1}{2}\|\phi'(X_s)\|^2 - 1$, and $(\hat{B})_{t \geq 0}$ is a one-dimensional Brownian motion. As a result, we have that for all $ 0\leq s \leq t \leq 1$
$$
\E\left[ \log \left( \frac{\|\phi(X_t)\|}{\|\phi(X_s)\|} \right) \huge| \, \|\phi(X_0)\| > 0 \right] = \left(\frac{1- 2^{-n}}{4} - \frac{1}{n}\right) (t-s),
$$
Moreover, for $n\geq 2$, we have 
$$
\textup{Var}\left[ \log \left( \frac{\|\phi(X_t)\|}{\|\phi(X_s)\|} \right) \huge| \, \|\phi(X_0)\| > 0 \right] \leq  \left(n^{-1/2} + \Gamma_{s,t}^{1/2}\right)^2 (t-s),
$$
where $\Gamma_{s,t} = \frac{1}{4}\int_{s}^t \,\left(\left(\E \phi'(X^1_u) \phi'(X^2_u) - \frac{(1 - 2^{-n})^2}{4}\right) + n^{-1}\left(\frac{1 - 2^{-n}}{2} - \E \phi'(X^1_u) \phi'(X^2_u)\right)\right) du$. 
}
\begin{proof}
Let  $t \in [0,1]$. Let us firs consider the case where $\|\phi(X_0)\| = 0$. For all $t$ we have $\|\phi(X_t)\| = 0$ and the result is trivial.\\

We consider the  case where $\|\phi(X_0)\| > 0$ (happens with probability $1- 2^{-n}$), and all the expectations in this proof are conditionally on this event.
Consider the function $g : \reals^n \to \reals$ given by 
$$
g(x) = \log(\|\phi(x)\|) = \frac{1}{2} \log(\|\phi(x)\|^2).
$$
Ideally, we would like to use \ito's lemma and \cref{lemma:tau_genera_n_relu}, which ensures that $\|\phi(X_t)\|$ remains positive on $[0,1]$, and obtain for all $t \in [0,1]$
$$
dg(X_t) \overset{dist}{=} \frac{1}{\sqrt{n}} d \hat{B}_t + \frac{1}{n} \mu_t dt,
$$
where $\mu_t$ is some well defined quantity. This would let us conclude. However, \ito's lemma requires that the function be $\mathcal{C}^2(\reals^n)$, which is violated by our choice of $g$. To mitigate this issue, we consider a sequence of function $(g_m)_{m \geq m}$ that approximates the function $g$ when $m$ goes to infinity. For $m \geq 1$, let $g_m$ be defined by 
$$
g_m(x) = \frac{1}{2}\log( \|\phi_m(x)\|^2),
$$
where $\phi_m(t) = \int_{0}^t h(m \,u) du$ and $h(t) = (1 + e^{-t})^{-1}$ is the Sigmoid function. We have that 
$$
\begin{cases}
\frac{\partial g_m}{\partial x_i}(x) = \frac{h(m x_i) \phi_m(x_i)}{\|\phi_m(x)\|^2}\\
\frac{\partial^2 g_m}{\partial x_i^2}(x) = \frac{m h(m x_i)(1 - h(m x_i)) \phi_m(x_i) + h(m x_i)^2}{\| \phi_m(x)\|^2} - 2 \frac{h(m x_i)^2 \phi_m(x_i)^2}{\|\phi_m(x)\|^4}
\end{cases}
$$

Let $X^m$ be the solution of the following SDE
$$
dX^m_t = \frac{1}{\sqrt{n}} \|\phi_m(X^m_t)\| dB_t \, ,\quad  t\in [0,\infty), \quad X_0 = W_{in} x,
$$

Using \ito's lemma, we have that 
$$
dg_m(X^m_t) = \frac{1}{n} \mu_s^m ds + \frac{1}{\sqrt{n}} \sum_{i=1}^n \sigma_s^{m,i} dB_s^i,
$$
where  $$\mu^m_s = \frac{1}{2} \sum_{i=1}^n \left(m h(m X^{m,i}_s)(1 - h(m X^{m,i}_s)) \phi_m(X^{m,i}_s) + h(m X^{m,i}_s)^2 \right) - \frac{\|h(m X^m_s) \circ \phi_m(X^m_s)\|^2}{\|\phi_m(X^m_s)\|^2},$$
and $\sigma^{m,i}_s = \frac{|h(m X^{m,i}_s) \circ \phi_m(X^{m,i}_s)|}{\|\phi_m(X^m_s)\|}$. By \cref{lemma:convergence_L1}, we know that $g_m(X^m_s) - g_m(X^m_0)$ converges in $L_1$ to $g(X_s) - g(X_0)$. Let us now compute the limit of $g_m(X^m_s) - g_m(X^m_0)$  from the equation above to conclude. More precisely, let us show that for all $s \in (0,1)$
$$
\lim_{m \to \infty}\E \left| g_m(X^m_s) - g_m(X^m_0) -  (Z_s - Z_0)\right| =0,
$$
where $Z$ is the process given by 
$$
Z_t = g(X_0) + \frac{1}{\sqrt{n}}\left(\sum_{i=1}^n \int_0^t \sigma^i_s dB^i_t\right) + \frac{1}{n} \int_{0}^t \mu_s ds,
$$
where $\mu_s = \frac{1}{2} \|\phi'(X^i_s)\|^2 -1$, and $\sigma_s^i = \frac{\phi(X_s^i)}{\|\phi(X_s)\|}$.\\

Let $t \in (0,1]$. Using triangular inequality, \ito isometry, and Cauchy-Schwartz inequality, we have that 
\begin{equation}\label{eq:overall_decomp}
\begin{aligned}
    \E \left| g_m(X^m_t) - g_m(X^m_0) -  (Z_t - Z_0)\right| &\leq \frac{1}{n} \int_{0}^t \E\left|\mu^m_s -  \mu_s\right| ds + \frac{1}{\sqrt{n}}\E \left|\int_{0}^t  \sum_{i=1}^n (\sigma_s^i - \sigma^{m,i}_s ) dB^i_s\right|\\
&\leq \frac{1}{n} \int_{0}^t \E\left|\mu^m_s -  \mu_s\right| ds + \frac{1}{\sqrt{n}} \left(\E \int_{0}^t \sum_{i=1}^n (\sigma_s^i - \sigma^{m,i}_s)^2 ds \right)^{1/2}\\
\end{aligned}    
\end{equation}

We first deal with the the term $\int_{0}^t \E\left|\mu^m_s -  \mu_s\right| ds$. Let us show that this term converges to $0$. Let us show the following,
\begin{center}
    \textit{$\forall s >0, \lim_{m\to \infty} \E \, |\mu^m_s - \mu_s| = 0$.}\\
\end{center}

Let $s \in [0,1]$. We have that 
$$
\mu^m_s = \frac{1}{2} \sum_{i=1}^n J^m_i  - G^m,
$$
where $J^m_i = m h(m X^{m,i}_s)(1 - h(m X^{m,i}_s)) \phi_m(X^{m,i}_s) + h(m X^{m,i}_s)^2$, and $G^m = \frac{\|h(m X^m_s) \circ \phi_m(X^m_s)\|^2}{\|\phi_m(X^m_s)\|^2}$. Let us start with the term $G^m$. Observe that $G^m \leq 1$ almost surely. We have that 

\begin{align*}
\E |1 - G^m| &= \E \left[\frac{\|(1 - h(m X^m_s)^2)^{1/2} \circ \phi_m(X^m_s)\|^2}{\|\phi_m(X^m_s)\|^2}\right]\\
&= \E \left[\frac{\|(1 - h(m X^m_s)^2)^{1/2} \circ \phi_m(X^m_s)\|^2}{\|\phi_m(X^m_s)\|^2} \ind_{ \{\min_{i}|X^{m,i}_s| \geq \log(m)/m \}} \right]\\
&+ \E \left[\frac{\|(1 - h(m X^m_s)^2)^{1/2} \circ \phi_m(X^m_s)\|^2}{\|\phi_m(X^m_s)\|^2} \ind_{ \{\min_{i} |X^{m,i}_s| < \log(m)/m \}} \right]\\
\end{align*}

When $\min_i |x^i| \geq \log(m)/m$, we have that  for all $i \in [n]$
$$(1 - h(m x^i)^2) \leq 2 (1 - h(m x^i)) \leq 2\exp(-m \times \log(m) / m) = 2 m^{-1}.$$
Therefore, 
$$
\E \left[\frac{\|(1 - h(m X^m_s)^2)^{1/2} \circ \phi_m(X^m_s)\|^2}{\|\phi_m(X^m_s)\|^2} \ind_{ \{\min_{i}|X^{m,i}_s| \geq \log(m)/m \}} \right] \leq 2 m^{-1}.
$$

For the remaining term, using the fact that $1 - h^2 \leq 1$, we have that 
\begin{align*}
\E \left[\frac{\|(1 - h(m X^m_s)^2)^{1/2} \circ \phi_m(X^m_s)\|^2}{\|\phi_m(X^m_s)\|^2} \ind_{ \{\min_{i} |X^{m,i}_s| < \log(m)/m \}} \right] &\leq \mathbb{P}\left(\min_{i} |X^{m,i}_s| < \log(m)/m\right)\\
&\leq \sum_{i=1}^n \mathbb{P}\left(|X^{m,i}_s| < \log(m)/m\right).\\
\end{align*}

Now using \cref{lemma:convergence_Xm_X}, we have that 
$$
\mathbb{P}(\|X^m_s - X_s\| \geq \log(m)/m) \leq \frac{C_n}{\log(m)^2},
$$
for some constant $C_n$ that depends on $n$. Therefore, for all $i$, we have 

\begin{align*}
\mathbb{P}(|X^{m,i}_s| < \log(m)/m) &\leq \mathbb{P}(|X^{i}_s| <2  \log(m)/m) + \mathbb{P}(|X^{m,i}_s - X_{s}^i| \geq \log(m)/m)\\
&\leq \mathbb{P}(|X^{i}_s| < 2 \log(m)/m) + \frac{C_n}{\log(m)^2}.
\end{align*}
Recall that $X^i_s = X^i_0 + \frac{1}{\sqrt{n}} \int_{0}^s \| \phi(X_u)\| dB_u$. By \cref{lemma:uppder_bound_delta}, we know that
\begin{align*}
\mathbb{P}(|X^{i}_s| < 2 \log(m)/m) = \mathcal{O}(\log(m)/m),
\end{align*}
We conclude that $\lim_{m \to \infty} \E |1 - G^m|= 1.$\\

Now let us show that for all $i$, $\lim_{m \to \infty} \E|J^m_i - \phi'(X^i_s)| = 0$. Let $i \in [n]$ and $A^{m,i}_s = m h(m X^{m,i}_s)(1 - h(m X^{m,i}_s)) \phi_m(X^{m,i}_s)$. 
We have that 

\begin{align*}
\E \,|A^{m,i}_s|  &= \E \, |A^{m,i}_s|  \ind_{\{|X^{m,i}_{s}| \leq 2\log(m)/m\}} + \E \, |A^{m,i}_s|  \ind_{\{|X^{m,i}_{s}| \leq 2\log(m)/m\}}\\
&\leq m \times 2\log(m) /m \times  \mathbb{P}(|X^{m,i}_{s}| \leq 2\log(m)/m) + m^{-1} \E |\phi_m(X^{m,i}_s)|\\
&\leq 2\log(m) \times \left( \mathbb{P}(|X^{i}_{s}| \leq 3\log(m)/m) + \mathbb{P}(|X^{m,i}_{s} - X^{i}_s| \geq \log(m)/m)\right) + m^{-1} \E |\phi_m(X^{m,i}_s)|\\
&\leq 2\log(m) \times  \mathbb{P}(|X^{i}_{s}| \leq 3\log(m)/m) + \frac{2 C_n}{ \log(m)} + m^{-1} \E |\phi_m(X^{m,i}_s)|,\\
\end{align*}
where we have used \cref{lemma:convergence_Xm_X} and Markov's inequality. Using \cref{lemma:diff_phi_m_phi} and the fact that the absolute value function is Lipschitz, we know that $\lim_{m \to \infty} \E \, |\phi_m(X^{m,i}_s)| =  \E \, |\phi(X^{i}_s)| < \infty$. Therefore, the third term vanishes in the limit $m \to \infty$. The second term $2 C_n / \log(m)$ also vanishes. The first term also vanished using \cref{lemma:uppder_bound_delta}. Therefore, $\lim_{m \to \infty} \E |A^{m,i}_s| = 0$.\\

Let us now deal with the last term in $J^m_i$.  We have that 
\begin{align*}
    \E \, \left|h(m X^{m,i}_s)^2 - \phi'(X^{i}_s) \right|&= \E \, \left|h(m X^{m,i}_s)^2 - \phi'(X^{i}_s) \right| \ind_{\{ |X^{m,i}_s| \geq \log(m)/m\}} \\
    &+ \E \, \left|h(m X^{m,i}_s)^2 - \phi'(X^{i}_s) \right| \ind_{\{ |X^{m,i}_s| < \log(m)/m\}}\\
    &\leq 3 m^{-1} + 2 \mathbb{P}(|X^{m,i}_s| < \log(m)/m)\\
    &\leq 3 m^{-1} + 2 \mathbb{P}(|X^{i}_s| < 2\log(m)/m) + 2 \frac{C_n}{\log(m)^2}
\end{align*}

Using \cref{lemma:uppder_bound_delta}, we obtain that $\lim_{m \to \infty} \E |h^2(mX^{m,i}_s) - \phi'(X^i_s)| = 0$. Hence, we obtain that $\lim_{m \to \infty} \E|J^m_i - \phi'(X_s)| = 0$. We conclude that $\lim_{m \to \infty} \E |\mu^m_s - \mu_s| = 0$. Moreover, from the analysis above, it is easy to see that $\sup_{m\geq1, \,s\in (0,1]} \E |\mu^m_s - \mu_s| < \infty$.\\

We now deal with the second term $\left(\E \int_{0}^t \sum_{i=1}^n (\sigma_s^i - \sigma^{m,i}_s)^2 ds \right)^{1/2}$ from \cref{eq:overall_decomp}. For this part only, we define the stopping time $\tau_\epsilon$ for $\epsilon \in (0, \| \phi(X_0)\| \wedge \|\phi(X_0)\|^{-1})$ (Recall that the analysis is conducted conditionally on the fact that $\| \phi(X_0)\| > 0$) by 
$$
\tau_\epsilon = \inf\{t\geq 0, \textrm{ s.t. } \|\phi(X_t)\| \in [0,\epsilon] \cup [\epsilon^{-1}, \infty) \}.
$$
Notice that $\tau_\epsilon > 0$ almost surely since $\|\phi(X_0)\| \in (\epsilon, \epsilon^{-1})$.\\

Let $s \in (0,1]$. We have that 

\begin{align*}
\E \int_{0}^{t \wedge \tau_\epsilon}\sum_{i=1}^n (\sigma_s^i - \sigma^{m,i}_s)^2 &\leq 2 \E \int_{0}^{t \wedge \tau_\epsilon} \sum_{i=1}^n \left(\frac{\phi(X^i_s) }{\|\phi(X_s)\|} - \frac{\phi_m(X^{m,i})}{\|\phi_m(X^m_s)\|}\right)^2 ds \\
&+ 2 \E \int_{0}^{t \wedge \tau_\epsilon} \frac{\|(1-h(mX^m_s)^2)^{1/2} \circ \phi_m(X^m_s))\|}{\|\phi(X^m_s)\|^2} ds.\\
\end{align*}
The second term can be upperbounded in the following fashion 
\begin{align*}
\E \int_{0}^{t \wedge \tau_\epsilon} \frac{\|(1-h(mX^m_s)^2)^{1/2} \circ \phi_m(X^m_s))\|}{\|\phi(X^m_s)\|^2} ds &\leq \int_{0}^{t} \E \frac{\|(1-h(mX^m_s)^2)^{1/2} \circ \phi_m(X^m_s))\|}{\|\phi(X^m_s)\|^2} ds\\
&= \int_{0}^{t} \E|1 - G_m|,
\end{align*}
where $G_m$ is defined above. We know that $\int_{0}^{t} \E|1 - G_m|$ converges to $0$ in the limit $m \to \infty$ by the Dominated convergence theorem (the integrand is bounded). Let us show that the first term also vanishes. We have that 

\begin{align*}
\E \int_{0}^{t \wedge \tau_\epsilon} \sum_{i=1}^n \left(\frac{\phi(X^i_s) }{\|\phi(X_s)\|} - \frac{\phi_m(X^{m,i})}{\|\phi_m(X^m_s)\|}\right)^2 ds 
&\leq 2 \E \int_{0}^{t \wedge \tau_\epsilon} \sum_{i=1}^n \left(\frac{\phi(X^i_s) - \phi_m(X^{m,i})}{\|\phi(X_s)\|}\right)^2 ds \\
&+ 2 \E \int_{0}^{t \wedge \tau_\epsilon} \sum_{i=1}^n \phi_m(X^{m,i}_s)^2\left(\frac{1}{\|\phi(X_s)\|} -  \frac{1}{\|\phi(X^m_s)\|}\right)^2 ds\\
&\leq 2 \epsilon^{-2} \E \int_{0}^{t \wedge \tau_\epsilon}  \|\phi(X^i_s) - \phi_m(X^{m,i}_s) \|^2 ds \\
&+ 2 \E \int_{0}^{t \wedge \tau_\epsilon} \sum_{i=1}^n \phi_m(X^{m,i}_s)^2\left(\frac{1}{\|\phi(X_s)\|} -  \frac{1}{\|\phi(X^m_s)\|}\right)^2 ds.\\
\end{align*}
The first term $2 \epsilon^{-2} \E \int_{0}^{t \wedge \tau_\epsilon}  \|\phi(X^i_s) - \phi_m(X^{m,i}_s) \|^2 ds$ converges to 0 in the limit $m \to \infty$ by \cref{lemma:diff_phi_m_phi} and \cref{lemma:convergence_Xm_X}. Let us deal with the second term. Define the event $E = \{ \sup_{s \in (0, 1]} \|X^m_s - X_s\| \leq \log(m)/m\}$ for $m$ large enough such that $\log(m)/m < \epsilon$. Observe that on the event $E$, we have that for all $s\in(0,1]$, $\|\phi(X^m_s) - \phi(X_s)\| \leq (\sqrt{n} + \log(m)) m^{-1}$. Hence, 
\begin{align*}
    \E \, \ind_{E} \,\int_{0}^{t \wedge \tau_\epsilon} \sum_{i=1}^n \phi_m(X^{m,i}_s)^2\left(\frac{1}{\|\phi(X_s)\|} -  \frac{1}{\|\phi(X^m_s)\|}\right)^2 ds \leq \epsilon^{-2} (\sqrt{n} + \log(m))^2 m^{-2} \to_{m \to \infty} 0.
\end{align*}

Moreover, letting $E^c$ be the complementary event of $E$, we have 
\begin{align*}
    \E \, \ind_{E^c} \,\int_{0}^{t \wedge \tau_\epsilon} \sum_{i=1}^n \phi_m(X^{m,i}_s)^2&\left(\frac{1}{\|\phi(X_s)\|} -  \frac{1}{\|\phi(X^m_s)\|}\right)^2 ds \\
    &\leq \E \, \ind_{E^c} \,\int_{0}^{t \wedge \tau_\epsilon} \sum_{i=1}^n \phi_m(X^{m,i}_s)^2\left(\frac{2}{\|\phi(X_s)\|^2} +  \frac{2}{\|\phi(X^m_s)\|^2}\right) ds\\
    &\leq 2 \epsilon^{-2} \E \, \ind_{E^c} \,\int_{0}^{t \wedge \tau_\epsilon} \|\phi_m(X^m_s)\|^2 ds + 2 \mathbb{P}(E^c).\\
\end{align*}

Using the fact that $\|\phi_m(X^m_s)\| \leq \frac{\sqrt{n}}{m} + \|\phi(X_s)\| + \|X^m_s - X_s\|$ (by \cref{lemma:diff_phi_m_phi} and the fact that ReLU is Lipschitz), we obtain 

\begin{align*}
\E \, \ind_{E^c} \,\int_{0}^{t \wedge \tau_\epsilon} \|\phi_m(X^m_s)\|^2 ds &\leq  3n m^{-2} + 3 \sup_{s \leq 1} \E\|X^m_s - X_s\|^2 + 3 \epsilon^{-2} \mathbb{P}(E^c).
\end{align*}
The term $\sup_{s \leq 1} \E\|X^m_s - X_s\|^2$ converges to $0$ by \cref{lemma:convergence_Xm_X}. Using Doob's martingale inequality on the submartingale $\|X^m_s - X_s\|$ (with respect to the natural filtration generated by the Brownian motion $B$)\footnote{The submartingale behaviour is a result of the convexity of the norm function and the fact that $X^m_s X_s$ is a martingale since $d(X^m_s X_s) = \frac{1}{\sqrt{n}}(\phi_m(X^m_s) - \phi(X_s)) dB_s$. A simple application of Jensen's inequality yields the result.} we obtain that 
$$
\mathbb{P}(E^c) \leq \frac{\E \|X^m_1 - X_1\|^2}{m^{-2} \log(m)^2} = \frac{C_n}{\log(m)^2},
$$
where we have used \cref{lemma:convergence_Xm_X}. We conclude that $ \lim_{m \to \infty} \E \int_{0}^{t \wedge \tau_\epsilon}\sum_{i=1}^n (\sigma_s^i - \sigma^{m,i}_s)^2  = 0$.\\

By observing that $\E \int_0^{t \wedge \tau_\epsilon} |\mu^m_s - \mu_s| ds \leq \E \int_0^{t} \E|\mu^m_s - \mu_s| ds$, a simple application of the Dominated convergence theorem yields $\lim_{m \to \infty} \E \int_0^{t \wedge \tau_\epsilon} |\mu^m_s - \mu_s| ds = 0$. Hence, we proved that

$$\lim_{m \to \infty} \E \left| g_m(X^m_{t\wedge \tau_\epsilon}) - g_m(X^m_0) -  (Z_{t\wedge \tau_\epsilon} - Z_0)\right| = 0.$$

From \cref{lemma:convergence_L1}, we know that $g_m(X^m_{t\wedge \tau_\epsilon}) - g_m(X^m_0)$ converges in $L_1$ to $g(X_{t\wedge \tau_\epsilon}) - g(X_0)$\footnote{the result of \cref{lemma:convergence_L1} holds when $t$ is replaced by $t\wedge \tau_\epsilon$ and the proof is exactly the same. We omit the proof to avoid redundancies}, therefore $\E |g(X_{t\wedge \tau_\epsilon}) - g(X_0) - (Z_{t\wedge \tau_\epsilon} - Z_0)| = 0$ which implies that almost surely,
$$
\log\left(\frac{\|\phi(X_{t\wedge \tau_\epsilon})\|}{\|\phi(X_0)\| }\right) = Z_{t\wedge \tau_\epsilon} - Z_0.
$$
Recall that this holds for any $\epsilon$ small enough. Observe that $\tau_\epsilon$ is almost surely non-decreasing as we decrease $\epsilon$. Hence $\tau_\epsilon$ has a limit almost surely. Using \cref{lemma:tau_genera_n_relu} and the continuity of the paths of $X_s$ we have that  $\lim_{\epsilon \to 0^+} \tau_\epsilon = \infty$. Taking the limit $\epsilon \to 0^+$, we conclude that almost surely we have 
$$
\log\left(\frac{\|\phi(X_{t})\|}{\|\phi(X_0)\| }\right) = Z_{t} - Z_0.
$$
Now observe that the coordinates of $X_t$ are identically distributed (not independent since we condition on $\|\phi(X_0)\| > 0$). Thus, for all $i\in [n]$, $\E  \phi'(X^i_s)= \E \phi'(X^1_s)$ where $X^1_s$ is the first coordinate of the vector $X_s$. Another key observation is that the event $\{ X^1_s > 0\}$ is included in the event $\{ \|\phi(X_0)\| >0 \}$ (\cref{lemma:constant_after_hit}). Hence, $\mathbb{P}(X^1_s > 0 \cap \|\phi(X_0)\| > 0) = \mathbb{P}(X^1_s > 0)$, where the last term $\mathbb{P}(X^1_s > 0)$ is \emph{free} from any conditioning on $\|\phi(X_0)\| > 0$. By observing that the random variable $X^1_s = X^1_0 + \frac{1}{\sqrt{n}} \int_0^s \|\phi(X_u)\| dB_u^1$ has a symmetric distribution around zero (by properties of the $X^1_0$ and the Brownian motion $B$), we have that  $\mathbb{P}(X^1_s > 0 | \|\phi(X_0)\| > 0) = \mathbb{P}(X^1_s > 0) \mathbb{P}(\|\phi(X_0)\| > 0)^{-1} = \frac{1}{2}(1-2^{-n})^{-1}$.

We conclude that 
$$
\E \, \mu_s = \frac{n}{4}(1 - 2^{-n})^{-1} - 1,
$$
which yields the desired result for the conditional mean by substraction. \\

Now let us deal with the variance. To alleviate the notation, we omit the conditioning on the event $\{\|\phi(X_0)\|>0\}$. All the expectations below are taken conditionally on this event. Let $0 \leq s\leq t \leq 1$. Let $\lambda>0$. We have that 
\begin{align*}
\textup{Var} \left[\log\left(\frac{\|\phi(X_{t})\|}{\|\phi(X_s)\| }\right)^2 \right]& = \textup{Var} (Z_t - Z_s)^2\\
&\leq \frac{(1 + \lambda^{-1})(t -s)}{n} + \frac{1 + \lambda}{n^2} \E \left(\int_{s}^t (\mu_u - \E\,\mu_u) du\right)^2\\
&\leq (t-s) \left(\frac{1 + \lambda^{-1}}{n} + \frac{1 + \lambda}{n^2} \int_{s}^t \textrm{Var} \mu_u \, du\right),
\end{align*}
where we have used the inequality $(a+b)^2 \leq (1 + \lambda^{-1}) a^2 + (1 + \lambda) b^2$ and the Cauchy-Schwartz inequality. It remains to simplify $\textrm{Var} \mu_u^2$. Let $p_1^u = \E\phi'(X^1_u) = 2^{-1}(1- 2^{-n})$ and $p_2^u = \E \phi'(X^1_u)\phi'(X^2_u)$. We have that
\begin{align*}
\E \mu_u^2 &= \frac{1}{4}\E \|\phi'(X_u)\|^4 - \E \|\phi'(X_u)\|^2 + 1\\
&= \frac{n(n-1)}{4}p^u_2- \frac{3n}{4} p^u_1 + 1,\\
\end{align*}
where we have used the exchangeability property of the family $\{\phi'(X^i_u), i=1, \dots n\}$. Thus, for the variance $\textrm{Var} \mu_u^2$, we obtain
$$
\textrm{Var} \mu_u^2 = \frac{n^2}{4}((p^u_2 - (p^u_1)^2) + n^{-1}(p^u_1 - p^u_2)).
$$
Therefore,
\begin{align*}
\textup{Var} \left[\log\left(\frac{\|\phi(X_{t})\|}{\|\phi(X_s)\| }\right)^2 \right] \leq (t-s) \left(\frac{1 + \lambda^{-1}}{n} + (1 + \lambda) \Gamma_{s,t}\right),
\end{align*}
where $\Gamma_{s,t} \overset{def}{=} \int_{s}^t  \frac{1}{4}((p^u_2 - (p^u_1)^2) + n^{-1}(p^u_1 - p^u_2))\, du$. Optimizing over $\lambda$ yields 
$$
\textup{Var} \left[\log\left(\frac{\|\phi(X_{t})\|}{\|\phi(X_s)\| }\right)^2 \right] \leq (t-s) \left(n^{-1/2} + \Gamma_{s,t}^{1/2}\right)^2.
$$

The term $\Gamma_{s,t}$ can be shown to have $\bigO(n^{-1/2})$ asymptotic behaviour using tools from Mckean-Vlasov theory. Thus, the variance term has (atmost) $\bigO(n^{-1})$ behaviour.
\end{proof}

\section{Proof of \cref{thm:infinite_width}}\label{sec:proof_infinite_width}

In this section, we provide the proof of \cref{thm:infinite_width}. We use the following Law of Large numbers that does not require independence.

\begin{thm}[Corollary 3.1 in \cite{sung2008lln}]\label{thm:sung_lln}
Let $(Y_n^i)_{1 \leq i \leq n, n\geq 1}$ be a triangular array of random variables. Assume that the following holds
\begin{itemize}
    \item $\sup_{n\geq 1} \frac{1}{n} \sum_{i=1}^n \E|Y_n^i| < \infty$.
    \item $\lim_{a \to \infty} \sup_{n\geq 1} \frac{1}{n} \sum_{i=1}^n \E |Y_{n}^i| \ind_{\{ |Y^i_n| > a\}} = 0.$
\end{itemize}
Then, we have that 
$$
\frac{1}{n} \left( \sum_{i=1}^n Y^i_n - \zeta^i_n\right) \longrightarrow 0,
$$
where the convergence is in $L_1$ and $\zeta^i_n = \E [Y^i_n | \mathcal{F}_{n, i-1}]$, with $\mathcal{F}_{n, j} = \sigma\{ Y^k_n, 1 \leq k \leq j\}$, i.e. the sigma algebra generated by the variables $\{Y^k_n, 1 \leq k \leq j\}$, and $\mathcal{F}_{n,0} = \{\emptyset, \Omega\}$ by definition.
\end{thm}

Let us now prove our result.\\

\noindent\textbf{Theorem \ref{thm:infinite_width}. }
\emph{For $0 \leq s \leq t \leq 1$, we have
$$
\log\left(\frac{\|\phi(X_t)\|}{\|\phi(X_s)\|}\right)\,  \ind_{\{ \|\phi(X_0)\|>0\}} \underset{n \to \infty}{\longrightarrow} \frac{t - s}{4}, \quad \textrm{and, }\quad \frac{\|\phi(X_t)\|}{\|\phi(X_s)\|}\,  \ind_{\{ \|\phi(X_0)\|>0\}}  \underset{n \to \infty}{\longrightarrow} \exp\left(\frac{t-s}{4}\right).
$$
where the convergence holds in $L_1$.}\\
\emph{
Moreover, we have that
$$
\sup_{i \in [n]} \E \left(\sup_{t \in [0,1]} |X^i_t - \tilde{X}^i_t|^2\right) = \bigO(n^{-2/5}),
$$
where $X^{i}_t$ is the solution of the following (Mackean-Vlasov) SDE
$$
d\tilde{X}^i_t = \left(\E \phi(\tilde{X}^i_t)^2\right)^{1/2} dB^i_t, \quad \tilde{X}^i_0 = X^i_0. 
$$}\\
\emph{
As a result, the pre-activations $Y^i_{\lfloor t L \rfloor}$ (\cref{eq:resnet}) converge in distribution to a Gaussian distribution in the limit infinite-depth-then-infinite-width
$$
\forall i \in [n],\,\,\,\, Y^i_{\lfloor t L \rfloor}  \xrightarrow{L \to \infty \textrm{ then }n \to \infty} \normal(0, d^{-1} \|x\|^2 \exp(t/2)).
$$
}

\begin{proof}
Let $0 \leq s \leq t \leq 1$. From \cref{thm:norm_post_act}, we have that almost surely
$$
\log\left(\frac{\|\phi(X_t)\|}{\|\phi(X_s)\|}\right) = \exp\left(\frac{1}{\sqrt{n}} (\hat{B}_t - \hat{B}_s) + \frac{1}{n} \int_{s}^t \mu_u du\right).
$$
We know that $\frac{1}{\sqrt{n}}  (\hat{B}_t - \hat{B}_s)$ converges to zero almost surely (by continuity of Brownian paths) and in $L_1$. Let us now deal with the second term $n^{-1} \int_{s}^t \mu_u du$. We have that $\frac{1}{n} \mu_u = \frac{1}{2} \frac{1}{n}\sum_{i=1}^n \phi'(X^i_u) - \frac{1}{n}.$ Fix $u \in [s,t]$ and let $Z^i_n = \phi'(X_u^i)$ (recall that $X^i_u$ has an implicit dependence on $n$). Since $Z_n^i$ is uniformly bounded across $i$ and $n$, it is straightforward that the conditions of \cref{thm:sung_lln} are satisfied. Therefore, we have the following convergence in $L_1$
$$
\frac{1}{n} \left( \sum_{i=1}^n Z^i_n - \zeta^i_n\right) \longrightarrow 0,
$$
where $\zeta^i_n = \E [Z^i_n | \mathcal{F}_{n, i-1}]$. Recall from the proof of \cref{thm:norm_post_act} that the event $\{X^j_u > 0\}$ is included in the event $\{ \|\phi(X_0)\| > 0\}$. Another key observation that will allow us to conclude is that the distribution of $X^i_u$ given $\mathcal{F}_{n, i-1}$ is symmetric around $0$ since the dependence is reflected only in the variance of the Brownian motion. Hence, $\zeta^i_n = \frac{1}{2} (1 - 2^{-n})^{-1}$ almost surely. Since  $n^{-1} \sum_{i=1}^n \zeta^i_n = \frac{1}{2} (1 - 2^{-n})^{-1}\longrightarrow \frac{1}{2},$ (in $L_1$), then $n^{-1} \sum_{i=1}^n Z^i_n$ converges to $1/2$ in $L_1$. Using the Dominated convergence theorem, we obtain the first result.\\

Let us now deal with the second result on the absolute growth factor. Let $N>0$ and define the event 
$$E_N = \left\{ \frac{\|\phi(X_t)\|}{\|\phi(X_s)\|} \leq \exp(N)\right\},$$ and let $E^c_N$ be its complementary event. For $N$ large enough, we have that 
\begin{align*}
\E \left|\frac{\|\phi(X_t)\|}{\|\phi(X_s)\|} - \exp(t/4) \right| &= \E \left|\frac{\|\phi(X_t)\|}{\|\phi(X_s)\|} - \exp((t-s)/4) \right| \, \ind_{E_N} + \E \left|\frac{\|\phi(X_t)\|}{\|\phi(X_s)\|} - \exp((t-s)/4) \right| \, \ind_{E_N^c}\\
&\leq \exp(N) \times \E \left|\log\left(\frac{\|\phi(X_t)\|}{\|\phi(X_s)\|}\right) - (t-s)/4 \right| \\
&+ \E \left|\frac{\|\phi(X_t)\|}{\|\phi(X_s)\|} - \exp((t-s)/4) \right| \, \ind_{E_N^c}\\
&\leq \exp(N) \times \E \left|\log\left(\frac{\|\phi(X_t)\|}{\|\phi(X_s)\|}\right) - (t-s)/4 \right| + K \, \mathbb{P}(E_N^c),\\
\end{align*}
where $K$ is a ($t$ dependent) constant and where we have used \cref{thm:norm_post_act} obtain that $\E \left|\frac{\|\phi(X_t)\|}{\|\phi(X_s)\|} \right|$ is finite. Taking $n$ to infinity in the inequality above, we obtain that for $N$ large enough
\begin{align*}
\limsup_{n \to \infty} \E \left|\frac{\|\phi(X_t)\|}{\|\phi(X_s)\|} - \exp(t/4) \right| &\leq K\, \mathbb{P}(E_N^c)\\
&\leq  N^{-1} K\,\E \left| \log\left(\frac{\|\phi(X_t)\|}{\|\phi(X_s)\|}\right) \right|,
\end{align*}
where we have used Markov's inequality. Since this is true for all $N$ large enough, we conclude that $\lim_{n \to \infty} \E \left|\frac{\|\phi(X_t)\|}{\|\phi(X_s)\|} - \exp((t-s)/4) \right| = 0$.\\

The convergence to Mckean-Vlasov dynamics is straightforward from \cref{thm:convergence_mckean}, and the Gaussian distribution is given by \cref{lemma:distribution_mckean_X}.

\end{proof}

\subsection{Some technical lemmas}

\begin{lemma}\label{lemma:distribution_mckean_X}
Let $x \in \reals^d$ such that $x \neq 0$, $m\geq 1$ be an integer, and consider the real-valued (Mckean-Vlasov) stochastic process $\tilde{X}$ given by 
$$
d\tilde{X}_t = \left(\E \phi(\tilde{X}_t)^2\right)^{1/2} dB_t \, ,\quad  t\in [0,\infty), \quad \tilde{X}_0 = \tilde{W}_{in}^\top x,
$$

where $\phi$ is the ReLU activation function, $(B_t)_{t \geq 0}$ is a one-dimensional Brownian motion, and $\tilde{W}_{in} \sim \normal(0, d^{-1} I)$. We have the following

$$
\forall t\geq 0, X_t \sim \normal(0, d^{-1} \|x\|^2 \exp(t/2)).
$$
\end{lemma}
\begin{proof}
Let $t > 0$. From the SDE, it is clear that $\tilde{X}_t$ is Gaussian with zero mean and variance $\int_{0}^t \|\phi(\tilde{X}_s)\|_{L_2} ds$ (by \ito isometry). Therefore, since ReLU is homogeneous, it is straightforward that for all $s>0$, $\|\phi(\tilde{X}_s)\|^2_{L_2} = \frac{1}{2}  \|\tilde{X}_s\|^2_{L_2}$. Using \ito's lemma, we obtain 
$$
d\tilde{X}^2_t = 2 \tilde{X}_t d\tilde{X}_t  + \frac{1}{2}\|\tilde{X}_t\|_{L_2}^2 dt.
$$
Taking the expectation\footnote{This should be understood as integrating the SDE, then taking the expectation, then differentiating once again.} yields the following ordinary differential equation

$$
d \|\tilde{X}_t\|^2_{L_2} = \frac{1}{2} \|\tilde{X}_t\|_{L_2}^2 dt,
$$
which has a closed-form solution given by 
$$
\|\tilde{X}_t\|^2_{L_2} = \|\tilde{X}_0\|^2_{L_2} \exp(t/2).
$$
We conclude by observing that $\|\tilde{X}_0\|^2_{L_2} = \E \tilde{X}_0^2 = d^{-1} \|x\|^2.$
\end{proof}

\begin{lemma}\label{lemma:conv_mckean_Xm_X}
Let $x \in \reals^d$ such that $x \neq 0$, $m\geq 1$ be an integer, and consider the two real-valued (Mckean-Vlasov) stochastic processes $\tilde{X}^m$ and $\tilde{X}$ given by 
$$
\begin{cases}
d\tilde{X}^m_t = \left(\E \phi_m(\tilde{X}^m_t)^2\right)^{1/2} dB_t \, ,\quad  t\in [0,\infty), \quad \tilde{X}^m_0 = \tilde{W}_{in}^\top x, \\
d\tilde{X}_t = \left(\E \phi(\tilde{X}_t)^2\right)^{1/2} dB_t \, ,\quad  t\in [0,\infty), \quad X_0 = \tilde{W}_{in}^\top x,
\end{cases}
$$

where $\phi_m(z) = \int_{0}^z h(mu) du$ where $h$ is the Sigmoid function given by $h(u) = (1 + e^{-u})^{-1}$, $\phi$ is the ReLU activation function, $(B_t)_{t \geq 0}$ is a one-dimensional Brownian motion, and $\tilde{W}_{in} \sim \normal(0, d^{-1} I)$. We have the following

$$
\forall t\geq 0, \, \, \E |\tilde{X}^m_t - \tilde{X}_t |^2 \leq \frac{2 t}{m^2} e^{2t}.
$$
\end{lemma}
\begin{proof}
The proof of \cref{lemma:conv_mckean_Xm_X} is similar to that of \cref{lemma:convergence_Xm_X} with the only difference of replacing the euclidean norm with the $L_2$ norm in probability space. Let $t \geq 0$, we have that 

\begin{align*}
\E|\tilde{X}^m_t - \tilde{X}_t |^2 &= \E \left( \int_{0}^t (\|\phi_m(\tilde{X}^m_s)\|_{L_2} - \|\phi(\tilde{X}_s)\|_{L_2}) dB_s\right)^2\\
&=  \int_{0}^t (\|\phi_m(\tilde{X}^m_s)\|_{L_2} - \|\phi(\tilde{X}_s)\|_{L_2})^2 ds\\
&\leq  \int_{0}^t \|\phi_m(\tilde{X}^m_s) -\phi(\tilde{X}_s)\|^2_{L_2} ds\\
&\leq  \frac{2t}{m^2} + 2 \int_{0}^t \|\tilde{X}^m_s -\tilde{X}_s\|^2_{L_2} ds,
\end{align*}
where we have used the triangular inequality and \cref{lemma:diff_phi_m_phi}. We conclude using Gronwall's lemma.
\end{proof}


\section{Proof of \cref{thm:infnite_width_then_infinite_depth}}\label{sec:infinite_width_then_infinite_depth}
\textbf{Theorem \ref{thm:infnite_width_then_infinite_depth}. } \emph{
Let $t\in [0,1]$. Then, in the limit $\lim_{L \to \infty} \lim_{n \to \infty}$ (infinite width, then infinite depth), we have that
$$
 \frac{\|\phi(Y_{\lfloor t L \rfloor})\|}{\|\phi(Y_0)\|}\,  \ind_{\{ \|\phi(Y_0)\|>0\}} \longrightarrow \exp\left(\frac{t}{2}\right),
$$
where the convergence holds in probability.}\\

\emph{
Moreover, the pre-activations $Y^i_{\lfloor t L \rfloor}$ (\cref{eq:resnet}) converge in distribution to a Gaussian distribution in the limit infinite-width-then-infinite-depth
$$
\forall i \in [n],\,\,\,\, Y^i_{\lfloor t L \rfloor}  \xrightarrow{n \to \infty \textrm{ then }L \to \infty} \normal(0, d^{-1} \|x\|^2 \exp(t)).
$$\\
}

\begin{proof}
Let $t \in [0,1]$. It is straightforward that $\lim_{n \to \infty}\ind_{\{ \|\phi(Y_0)\|>0\}} = 1$ almost surely. Moreover, we have that for all $t \in [0,1]$, $n^{-1} \|\phi(Y_{\lfloor t L \rfloor})\|^2$ converges in distribution to $\E \phi(Y^1_{\lfloor t L \rfloor})^2$ when $n$ goes to infinity \citep{yang_tensor3_2020, hayou21stable, matthews}. Since the limiting value is constant, then the convergence holds also in probability. Now let $\epsilon > 0$. We have that 
\begin{align*}
\mathbb{P}\left( \left| \frac{\|\phi(Y_{\lfloor t L \rfloor})\|}{\|\phi(Y_0)\|}\,  \ind_{\{ \|\phi(Y_0)\|>0\}} - \exp\left(\frac{t}{2}\right)\right| > \epsilon\right) &\leq \mathbb{P}\left( \frac{\|\phi(Y_{\lfloor t L \rfloor})\|}{\|\phi(Y_0)\|}\,  \ind_{\{ \|\phi(Y_0)\|>0\}} - \exp\left(\frac{t}{2}\right) > \epsilon\right) \\
&+ \mathbb{P}\left( \frac{\|\phi(Y_{\lfloor t L \rfloor})\|}{\|\phi(Y_0)\|}\,  \ind_{\{ \|\phi(Y_0)\|>0\}} - \exp\left(\frac{t}{2}\right) <- \epsilon\right).
\end{align*}
Let us show that the first term in the right-hand side converges to $0$ in the sequential limit `infinite width then infinite depth'. The proof is similar for the second term. We have that

\begin{align*}
    \mathbb{P}\left( \frac{\|\phi(Y_{\lfloor t L \rfloor})\|}{\|\phi(Y_0)\|}\,  \ind_{\{ \|\phi(Y_0)\|>0\}} - e^{\frac{t}{2}} > \epsilon\right) &\leq \mathbb{P}\left( \frac{\|\phi(Y_{\lfloor t L \rfloor})\|}{\|\phi(Y_0)\|}\,  \ind_{\{ \|\phi(Y_0)\|>0\}} - \frac{\sqrt{\E \phi(Y^1_{\lfloor t L \rfloor})^2}}{\sqrt{\E \phi(Y^1_{0})^2}} > \epsilon/2\right) \\
    &+ \ind\left( \frac{\sqrt{\E \phi(Y^1_{\lfloor t L \rfloor})^2}}{\sqrt{\E \phi(Y^1_{0})^2}} - \exp\left(\frac{t}{2}\right) > \epsilon/2\right),
\end{align*}
where $\ind(z > \epsilon/2) \overset{def}{=} \ind_{\{z > \epsilon/2\}}$ (to alleviate the notation).
Using the convergence in probability of $n^{-1} \|\phi(Y_{\lfloor t L \rfloor})\|^2$ to $\E \phi(Y^1_{\lfloor t L \rfloor})^2$, we obtain for all $L$
$$
\lim_{n \to \infty} \mathbb{P}\left( \frac{\|\phi(Y_{\lfloor t L \rfloor})\|}{\|\phi(Y_0)\|}\,  \ind_{\{ \|\phi(Y_0)\|>0\}} - \exp\left(\frac{t}{2}\right) > \epsilon\right) \leq  \ind\left( \frac{\sqrt{\E \phi(Y^1_{\lfloor t L \rfloor})^2}}{\sqrt{\E \phi(Y^1_{0})^2}} - \exp\left(\frac{t}{2}\right) > \epsilon/2\right).
$$
Using Lemma 5 in \cite{hayou21stable}, and the homegenous property of ReLU, we have that 
$$\lim_{L \to \infty} \E \phi(Y^1_{\lfloor t L \rfloor})^2 =\frac{1}{2} q_t,$$
where $q_t : [0,1] \to \reals^+$ is the solution of the ordinary differential equation $q_t' = q_t$, which has a unique solution given by $q_t = q_0 \exp(t)$. Dividing by $q_0$ and taking the square root, and taking $L$ to infinity, we obtain the desired result.\\

Regarding the convergence in distribution of the pre-activations, in the limit $n\to \infty$, the pre-activations become Gaussian with zero mean and variance $\E (Y^1_{\lfloor tL \rfloor})^2$. This variance converges to $q_t$ given above in the limit $L \to \infty$. The conclusion is straightforward using Slutsky's lemma.
\end{proof}

\section{Piece-wise linear activation functions}\label{sec:linear_act}
We have seen in \cref{sec:general_width} that the distribution of $X_t$ is generally intractable for $n \geq 2$. This is purely due to \emph{finite width $n\geq 2$} and not to the non-linearity of the activation function. To understand this, let us see what happens when the activation function is the identity function. In this case the process $X_t$ is solution of the following SDE
\begin{equation}\label{eq:linear_sde}
dX_t = \frac{1}{\sqrt{n}} \|X_t\| dB_t.
\end{equation}
When $n=1$, the SDE \cref{eq:linear_sde} has a closed-form solution given by the (conditional) GBM distribution (\cref{prop:gbm}). For general $n\geq 2$, the entries of $X_t$ are dependent and the resulting dynamics (generally) do not admit closed-form solutions. However, we can obtain closed-form solutions for the norm $\|X_t\|$. Indeed, a simple application of \ito's lemma yields the following results.
\begin{thm}[Norms with the identity activation]\label{thm:linear_post_norm}
With the linear activation, we have that for all $t \in [0,1]$,
$$
\|X_t\| = \|X_0\| \exp\left(\frac{1}{\sqrt{n}} \hat{B}_t + \left(\frac{1}{2} - \frac{1}{n}\right) t\right), \quad \textrm{almost surely,}
$$
where $(\hat{B})_{t \geq 0}$ is a one-dimensional Brownian motion. As a result, we have that for all $ 0\leq s \leq t \leq 1$
$$
\E\left[ \log \left( \frac{\|X_t\|}{\|X_s\|} \right)\right] = \left(\frac{1}{2} - \frac{1}{n}\right) (t-s).
$$
\end{thm}
The proof of \cref{thm:linear_post_norm} is straightforward using \ito's lemma. We omit the proof here.

\paragraph{Role of the non-linearity. } By comparing the result of \cref{thm:norm_post_act} and \cref{thm:linear_post_norm}, we observe some differences between the case of ReLU and that of the identity activation function. With ReLU, the drift term in $\log(\|\phi(X_t)\|/\|\phi(X_s)\|)$ is given by $\frac{1}{n}\int_{0}^t \mu_s ds$ which is a stochastic term with mean given by  $\left(\frac{1 - 2^{-n}}{4} - \frac{1}{n}\right) t$. With the identity activation, this drift term is \emph{deterministic} and is equal to $\left(\frac{1}{2} - \frac{1}{n}\right) t$. This allows to conclude the following:
\begin{itemize}
    \item \emph{Non-linearity induces stochastic drift:} the non-linearity of ReLU induces stochasticity in the drift term of $\log(\|X_t\| / \|X_0\|)$, which results in the Quasi-GBM dynamics given by \cref{thm:norm_post_act}.
    \item \emph{Non-linearity induces change of regime: } with ReLU, the mean drift of $\log(\|\phi(X_t)\| / \|\phi(X_0)\|)$ is given by $\left(\frac{1 - 2^{-n}}{4} - \frac{1}{n}\right) t$ which is negative for $n =1, 2, 3$. This induces the change of regime we discussed after \cref{thm:norm_post_act} (having a negative mean drift implies that there is a significant mass of the distribution of $\|X_t\|/\|X_0\|$ in the regime $(0,1)$). With the identity activation function, the drift term is always non-negative for $n\geq 2$, and negative for $n=1$. Thus, the change of regime cover some values $n \geq 2$ only when there is a non-linearity. We give more details about this observation in the next result.
\end{itemize}
To capture the effect of non-linearity in the regime change phenomenon discussed above, we study the dynamics of the post-norm activation for a special class of piece-wise linear activations that include both ReLU and the identity function. The result of \cref{thm:norm_post_act} can be easily extended to the case of general piece-wise linear activation functions using the same proof techniques. We obtain the following result which generalizes that of \cref{thm:norm_post_act} and \cref{thm:linear_post_norm}.

\begin{thm}[Post-activations norm for piece-wise linear activations]\label{thm:post_act_norm_general}
Let $\alpha, \beta \in \reals$, and let $\phi_{\alpha, \beta}$ be the activation function given by $\phi_{\alpha, \beta}(z) = \alpha \texttt{ReLU}(z) + \beta \texttt{ReLU}(z)$. We have that for all $t \in [0,1]$,
$$
\|\phi_{\alpha, \beta}(X_t)\| = \|\phi_{\alpha, \beta}(X_0)\| \exp\left(\frac{1}{\sqrt{n}} \hat{B}_t + \frac{1}{n} \int_{0}^t \mu^{\alpha, \beta}_s ds\right), \quad \textrm{almost surely,}
$$
where $\mu^{\alpha, \beta}_s =  \frac{1}{2}\sum_{i=1}^n (\alpha^2 \ind_{x_i \geq 0} + \beta^2 \ind_{x_i < 0}) - 1$, and $(\hat{B})_{t \geq 0}$ is a one-dimensional Brownian motion. As a result, we have that for all $ 0\leq s \leq t \leq 1$

\begin{itemize}
    \item if $\alpha = 0, \beta \neq 0$
    $$
\E\left[ \log \left( \frac{\|\phi_{0, \beta}(X_t)\|}{\|\phi_{0, \beta}(X_s)\|} \right) \huge| \, \|\phi_{0, \beta}(X_0)\| > 0 \right] = \left(\frac{\beta^2(1- 2^{-n})}{4} - \frac{1}{n}\right) (t-s).
$$
\item if $\alpha \neq 0, \beta = 0$
    $$
\E\left[ \log \left( \frac{\|\phi_{\alpha, 0}(X_t)\|}{\|\phi_{\alpha, 0}(X_s)\|} \right) \huge| \, \|\phi_{\alpha, 0}(X_0)\| > 0 \right] = \left(\frac{\alpha^2(1- 2^{-n})}{4} - \frac{1}{n}\right) (t-s).
$$
\item if $\alpha \neq 0, \beta \neq 0$
    $$
\E\left[ \log \left( \frac{\|\phi_{\alpha, \beta}(X_t)\|}{\|\phi_{\alpha, \beta}(X_s)\|} \right) \right] = \left(\frac{\alpha^2 + \beta^2}{4} - \frac{1}{n}\right) (t-s).
$$
\end{itemize}

\end{thm}
\cref{thm:post_act_norm_general} generalizes that of ReLU (\cref{thm:norm_post_act}, $\alpha = 1, \beta=0$) and the identity activation (\cref{thm:linear_post_norm}, $\alpha= -\beta = 1$). The discontinuity of the mean of $\log \left( \frac{\|\phi_{\alpha, \beta}(X_t)\|}{\|\phi_{\alpha, \beta}(X_s)\|} \right)$ at the poles $\alpha =0$ (and $\beta \neq 0$) and $\beta =0$ (and $\alpha \neq 0$) is due to the fact that the event $ \{ \|\phi_{\alpha, \beta}(X_0)\| > 0\}$ has non-zero probability in these cases and zero probability when $\alpha \neq 0$ and $\beta \neq 0$. 
\paragraph{Perturbation analysis around the identity function. } Consider the case when $\alpha = 1$ and $\beta = 1 - \varepsilon$ for some $\varepsilon \ll 1$. The mean logarithmic growth factor is given by 
$$
G_{s,t}^n = \left(\frac{1 + (1 - \varepsilon)^2}{4} - \frac{1}{n}\right) (t-s) \approx \left(\frac{1 - \varepsilon}{2} - \frac{1}{n}\right) (t-s).
$$
Observe that for $\varepsilon = 0$, we recover the result of \cref{thm:linear_post_norm} (identity activation). Hence, a small perturbation of the identity function has the effect of decreasing the factor $G_{s,t}^n$ which results in having negative values for $G_{s,t}^n$ for certain values of $n$. Indeed, by fixing $\alpha = 1$, notice that the minimum values of $G_{s,t}^n$ is obtained when $\beta \approx 0$, for which $\phi_{1,0} = $ ReLU. Notice that we can also control the change of regime by tuning the parameter $\alpha$. This allows us to control the sign of $G_{s,t}^n$ for any $n$ by tuning the parameter $\alpha$. We leave the analysis of the practical implications of tuning $\alpha$ for future work.

\newpage
\section{Additional Experiments}\label{sec:additional_experiments}

\subsection{Geometric Brownian motion}
Additional histograms of $Y_L$ and $\log(Y_l)$ (\cref{prop:gbm}) are shown in \cref{fig:additional_gbm_figs_1} and \cref{fig:additional_gbm_figs_2}.

\begin{figure}[H]
    \centering
    \begin{subfigure}[b]{0.49\textwidth}
         \centering
         \includegraphics[width=\textwidth]{figures/gbm_t1_histogram_logx1_depth5.pdf}
         \caption{Distribution of $\log(Y_L)$ with $L=5$}
         \label{fig:gbm_depth5}
     \end{subfigure}
     \hfill
     \begin{subfigure}[b]{0.49\textwidth}
         \centering
         \includegraphics[width=\textwidth]{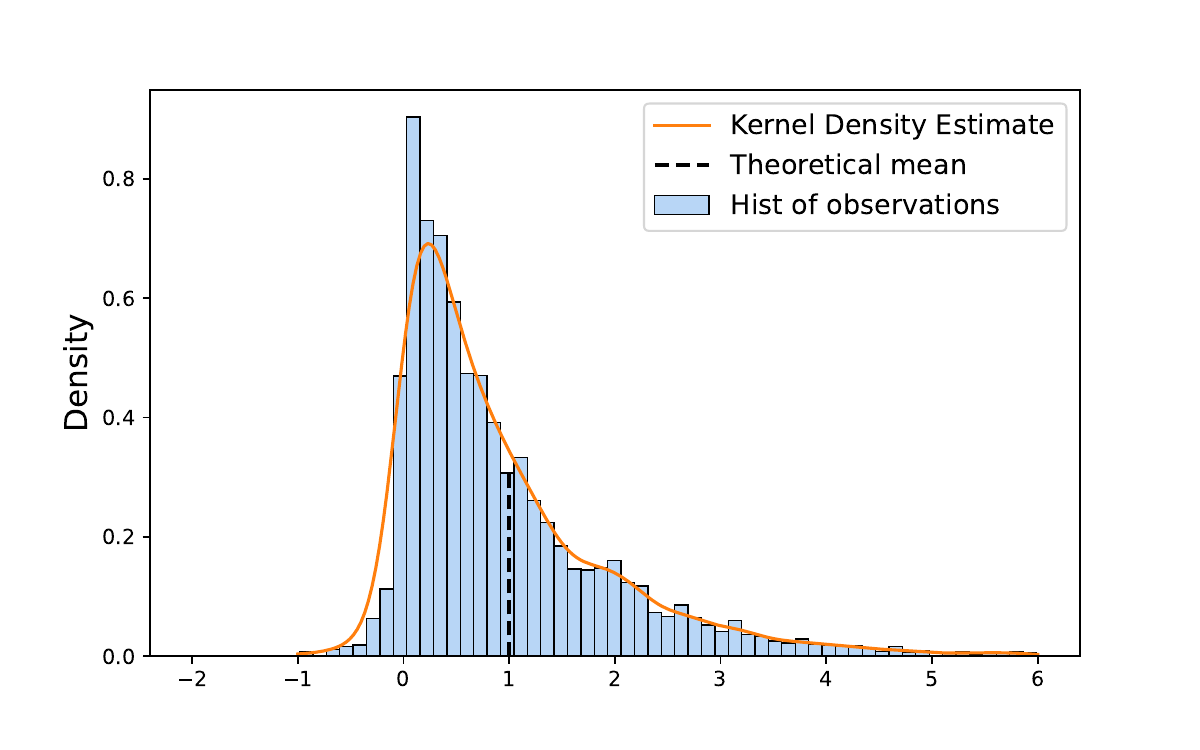}
         \caption{Distribution of $Y_L$ with $L=5$}
         \label{fig:gbm_depth10}
     \end{subfigure}
     \begin{subfigure}[b]{0.49\textwidth}
         \centering
         \includegraphics[width=\textwidth]{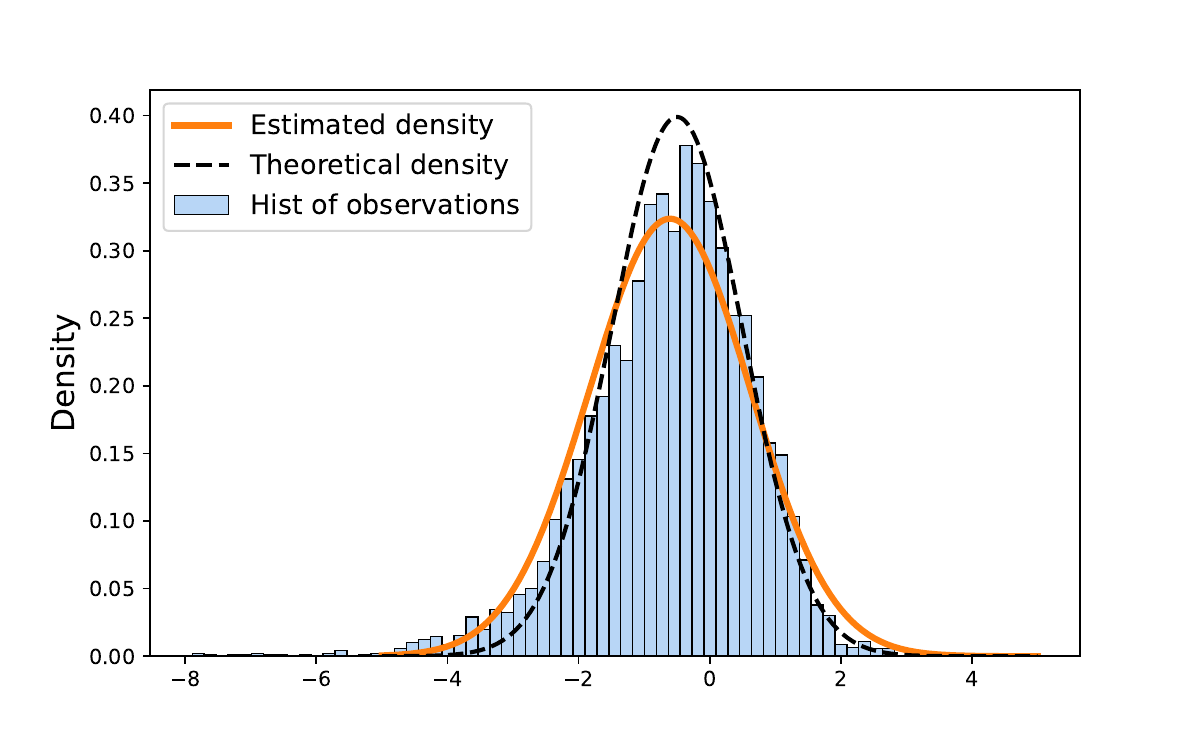}
         \caption{Distribution of $\log(Y_L)$ with $L=10$}
         \label{fig:gbm_depth10}
     \end{subfigure}
     \begin{subfigure}[b]{0.49\textwidth}
         \centering
         \includegraphics[width=\textwidth]{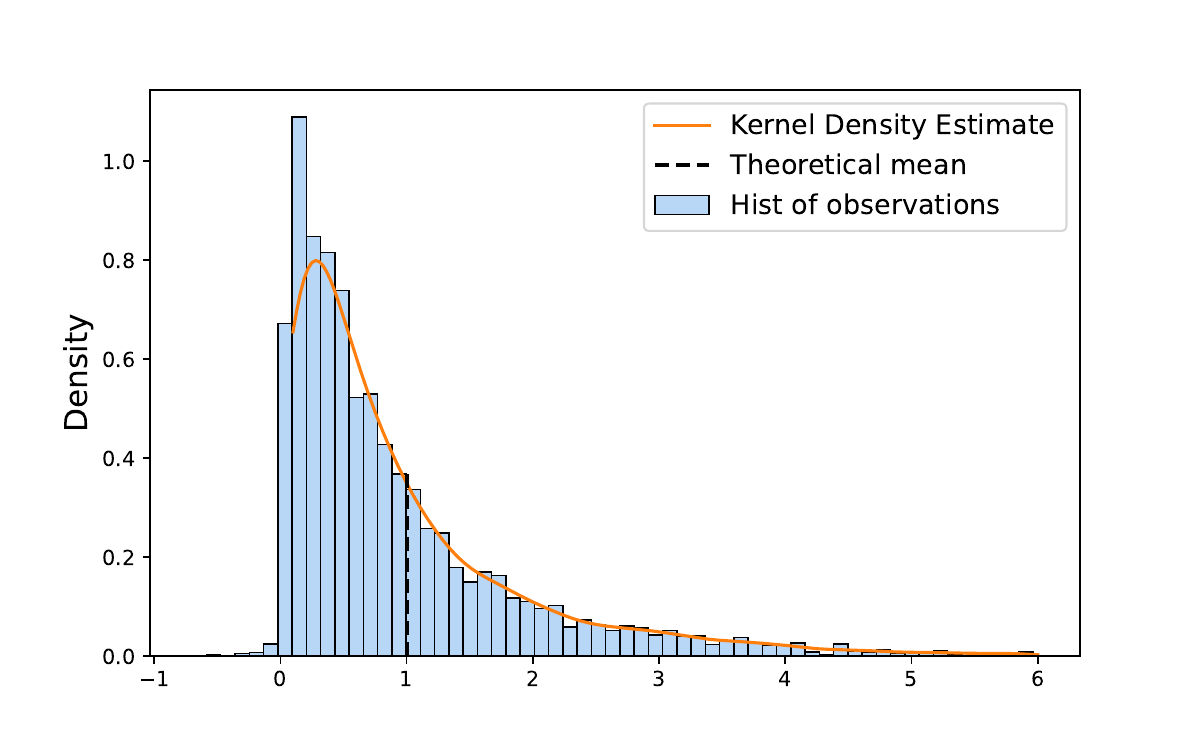}
         \caption{Distribution of $Y_L$ with $L=10$}
         \label{fig:gbm_depth10}
     \end{subfigure}
    \begin{subfigure}[b]{0.49\textwidth}
         \centering
         \includegraphics[width=\textwidth]{figures/gbm_t1_histogram_logx1_depth50.pdf}
         \caption{Distribution of $\log(Y_L)$ with $L=50$}
         \label{fig:gbm_depth10}
     \end{subfigure}
     \begin{subfigure}[b]{0.49\textwidth}
         \centering
         \includegraphics[width=\textwidth]{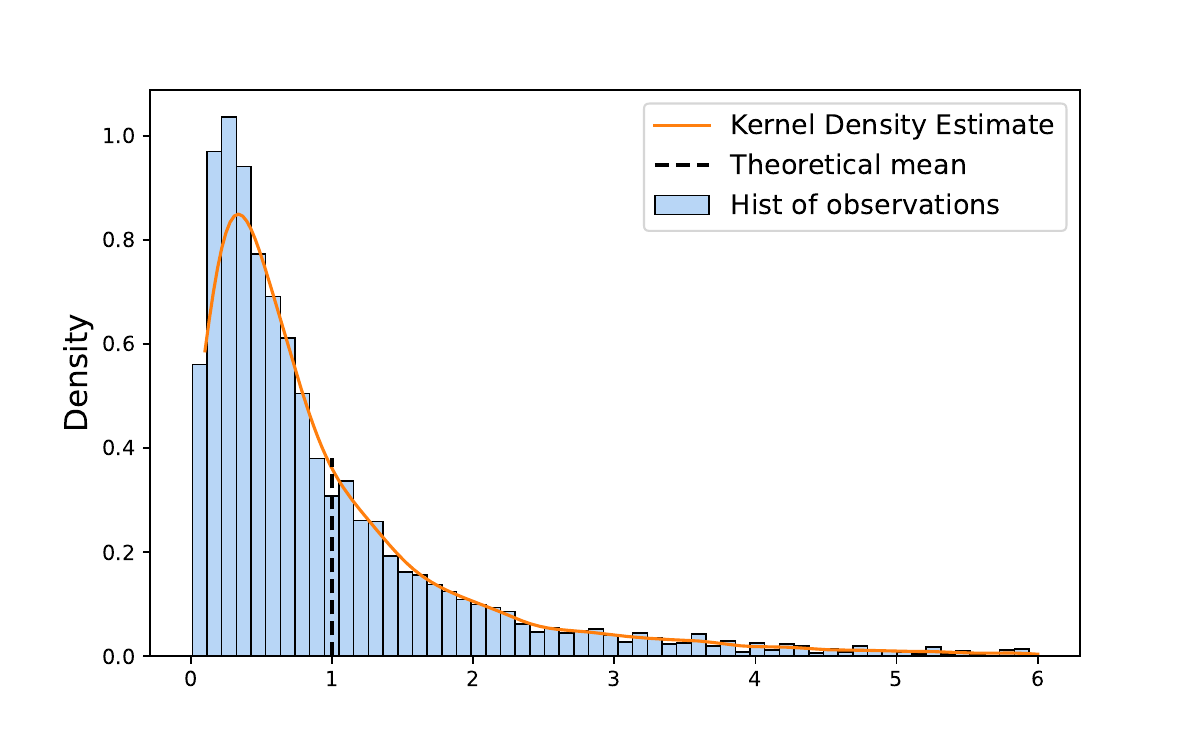}
         \caption{Distribution of $Y_L$ with $L=50$}
         \label{fig:gbm_depth50}
     \end{subfigure}  
    \caption{Empirical verification of \cref{prop:gbm}. \textbf{(a), (c), (e)} Histograms of $\log(Y_L)$ and based on $N=5000$ simulations for depths $L \in \{5, 10, 50\}$ with $Y_0 = 1$. Estimated density (Gaussian kernel estimate) and theoretical density (Gaussian) are  illustrated on the same graphs. 
    \textbf{(b), (d), (f)} Histograms of $Y_L$ based on $N=5000$ simulations for depths $L \in \{5, 10, 50\}$ with $Y_0 = 1$. }
    \label{fig:additional_gbm_figs_1}
\end{figure}

\begin{figure}[H]
    \centering
    \begin{subfigure}[b]{0.49\textwidth}
         \centering
         \includegraphics[width=\textwidth]{figures/gbm_t1_histogram_logx1_depth100.pdf}
         \caption{Distribution of $\log(Y_L)$ with $L=100$}
         \label{fig:gbm_depth5}
     \end{subfigure}
     \hfill
     \begin{subfigure}[b]{0.49\textwidth}
         \centering
         \includegraphics[width=\textwidth]{figures/gbm_t1_histogram_x1_depth100.pdf}
         \caption{Distribution of $Y_L$ with $L=100$}
         \label{fig:gbm_depth10}
     \end{subfigure}
     \begin{subfigure}[b]{0.49\textwidth}
         \centering
         \includegraphics[width=\textwidth]{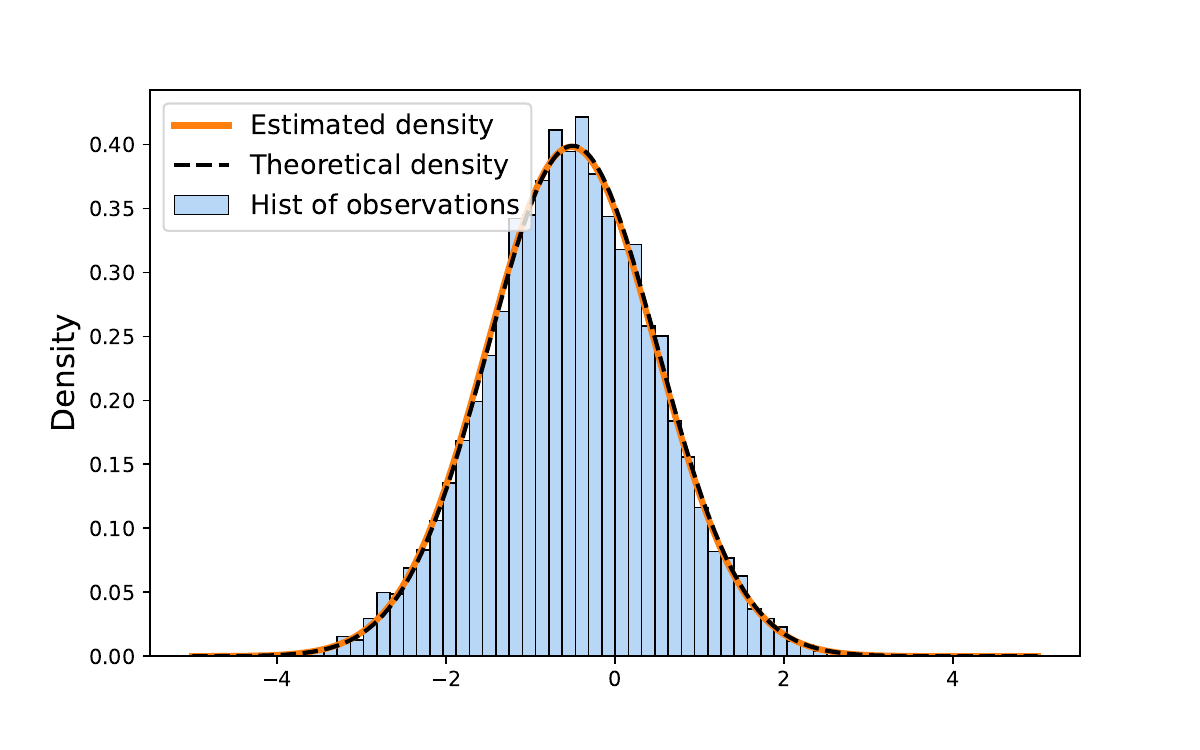}
         \caption{Distribution of $\log(Y_L)$ with $L=200$}
         \label{fig:gbm_depth10}
     \end{subfigure}
     \begin{subfigure}[b]{0.49\textwidth}
         \centering
         \includegraphics[width=\textwidth]{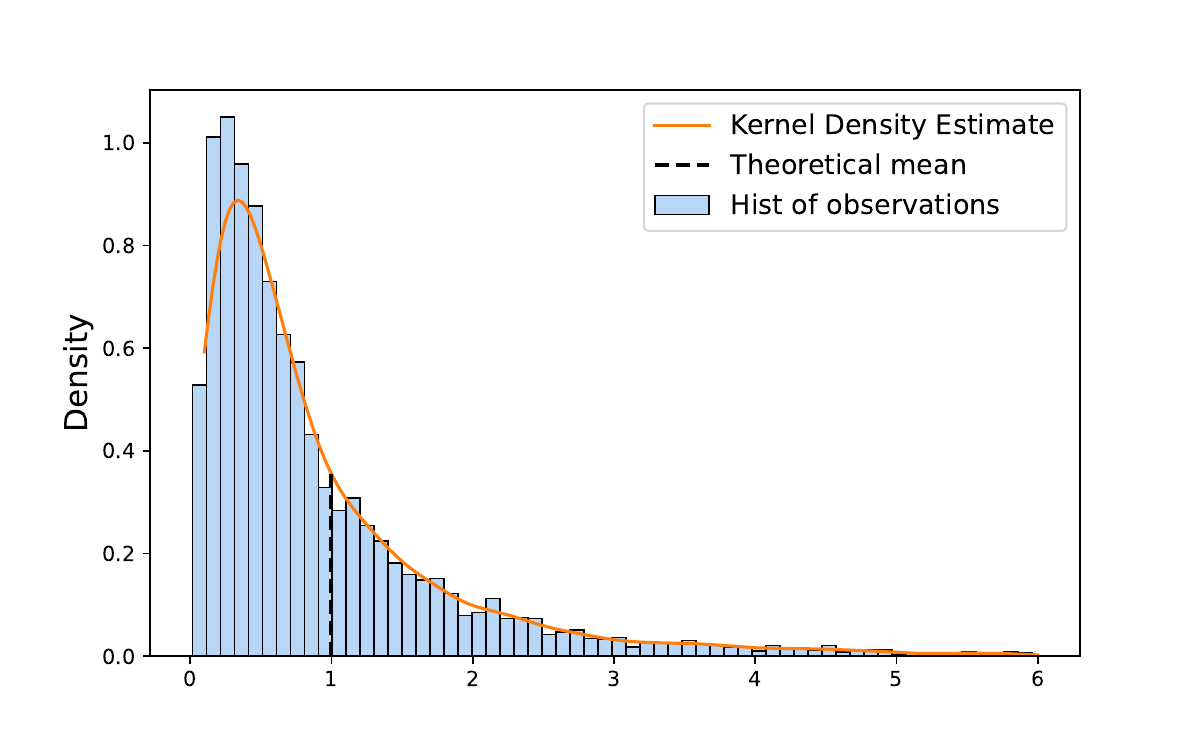}
         \caption{Distribution of $Y_L$ with $L=200$}
         \label{fig:gbm_depth10}
     \end{subfigure}

    \caption{Empirical verification of \cref{prop:gbm}. \textbf{(a), (c)} Histograms of $\log(Y_L)$ and based on $N=5000$ simulations for depths $L \in \{ 100, 200\}$ with $Y_0 = 1$. Estimated density (Gaussian kernel estimate) and theoretical density (Gaussian) are  illustrated on the same graphs. 
    \textbf{(b), (d)} Histograms of $Y_L$ based on $N=5000$ simulations for depths $L \in \{100, 200\}$ with $Y_0 = 1$. }
    \label{fig:additional_gbm_figs_2}
\end{figure}

\newpage

\subsection{Ornstein-Uhlenbeck process}
Additional histograms of $Y_L$ and $g(Y_l)$ (\cref{prop:ou_nn}) are shown in \cref{fig:additional_gbm_figs_1} and \cref{fig:additional_gbm_figs_2}.

\begin{figure}[H]
    \centering
    \begin{subfigure}[b]{0.49\textwidth}
         \centering
         \includegraphics[width=\textwidth]{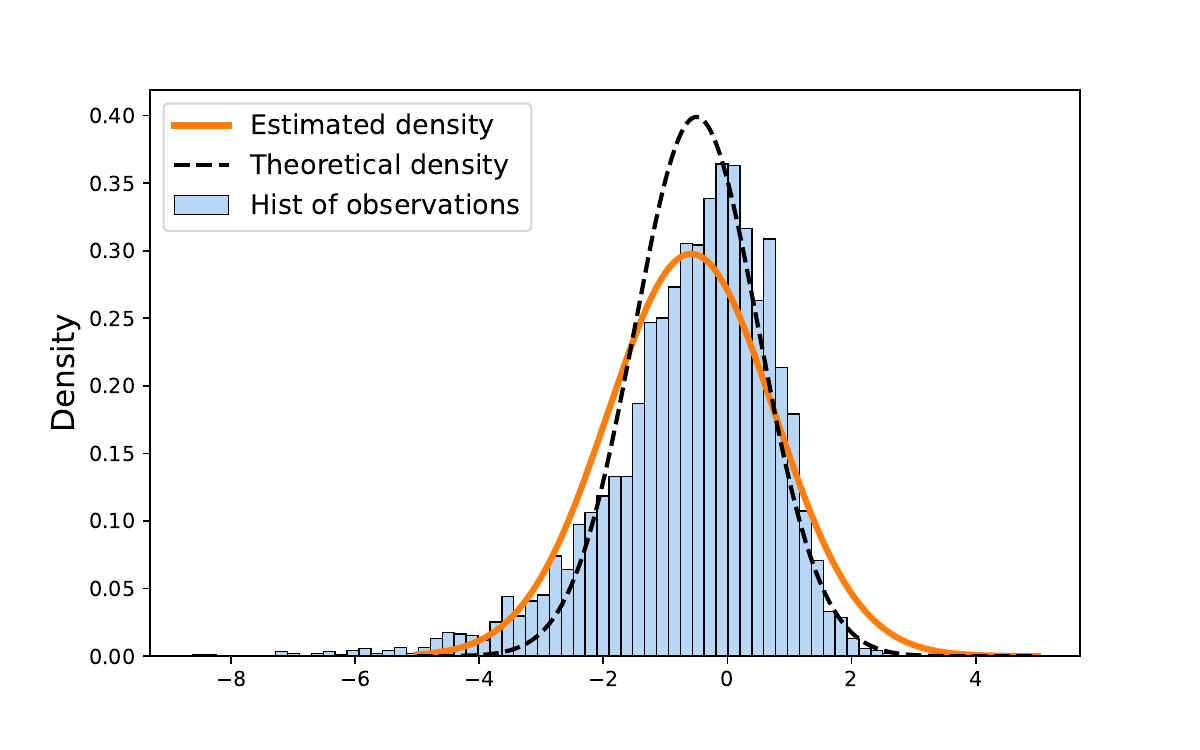}
         \caption{Distribution of $\log(Y_L)$ with $L=5$}
         \label{fig:gbm_depth5}
     \end{subfigure}
     \hfill
     \begin{subfigure}[b]{0.49\textwidth}
         \centering
         \includegraphics[width=\textwidth]{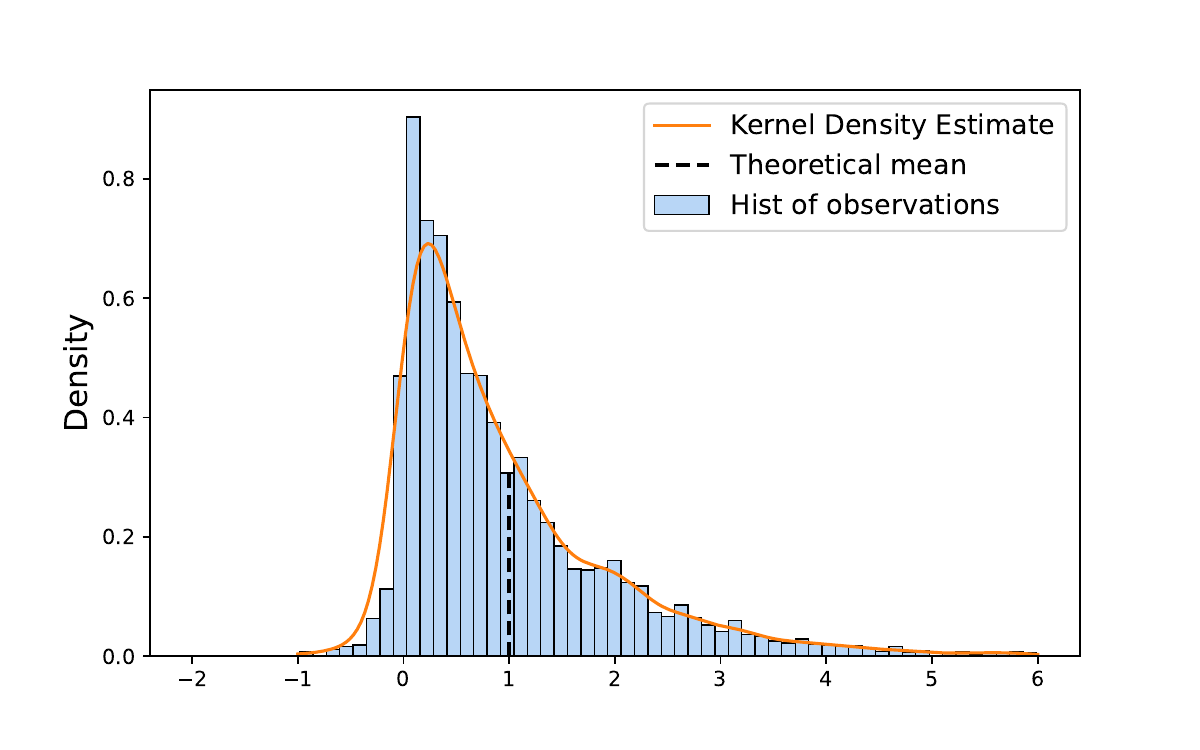}
         \caption{Distribution of $Y_L$ with $L=5$}
         \label{fig:gbm_depth10}
     \end{subfigure}
     \begin{subfigure}[b]{0.49\textwidth}
         \centering
         \includegraphics[width=\textwidth]{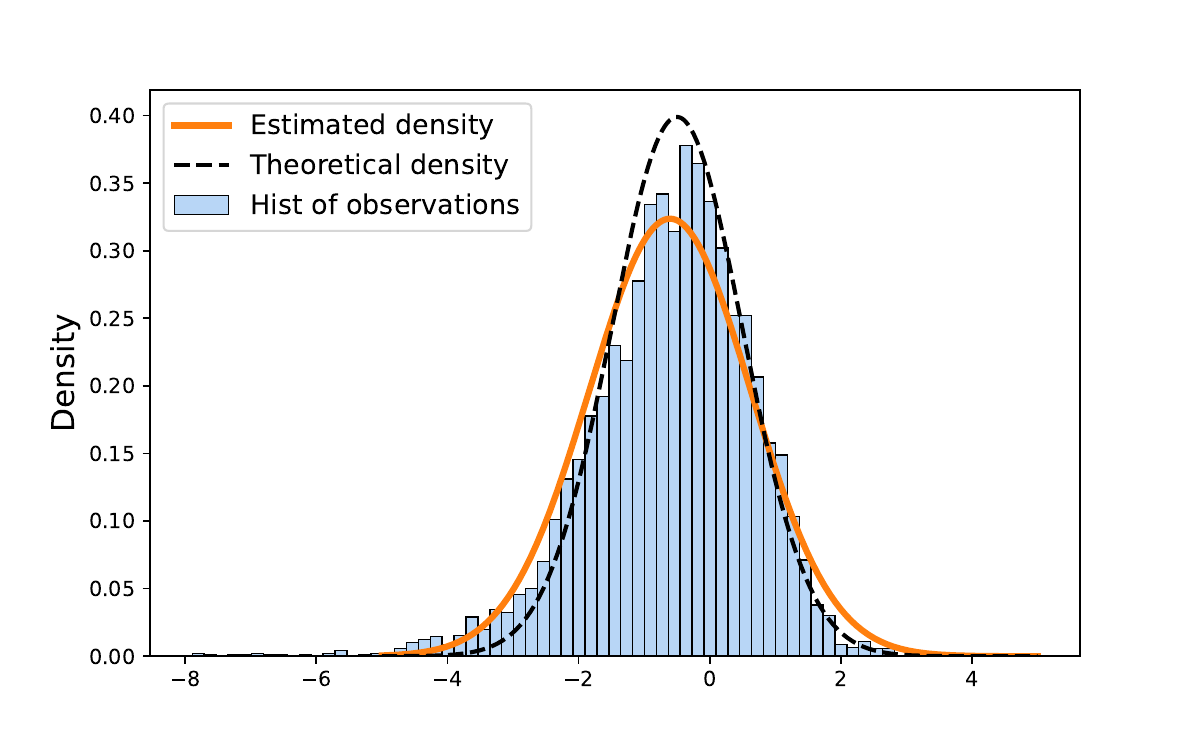}
         \caption{Distribution of $\log(Y_L)$ with $L=10$}
         \label{fig:gbm_depth10}
     \end{subfigure}
     \begin{subfigure}[b]{0.49\textwidth}
         \centering
         \includegraphics[width=\textwidth]{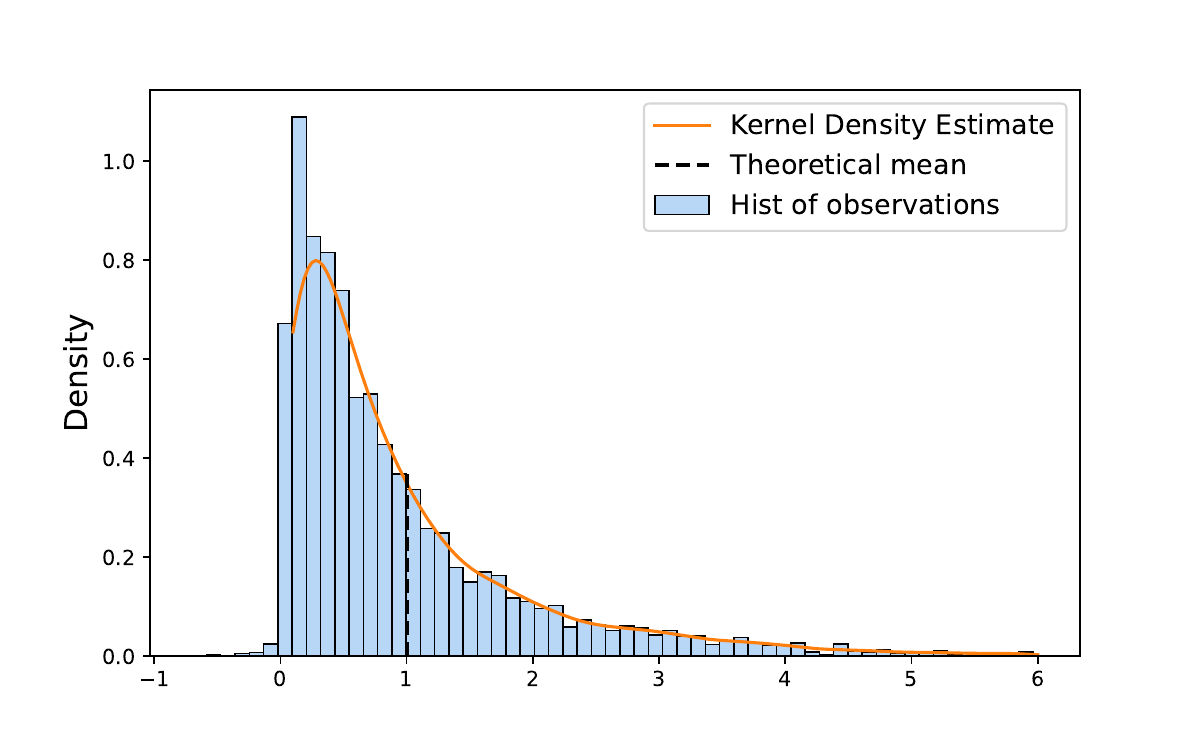}
         \caption{Distribution of $Y_L$ with $L=10$}
         \label{fig:gbm_depth10}
     \end{subfigure}
    \begin{subfigure}[b]{0.49\textwidth}
         \centering
         \includegraphics[width=\textwidth]{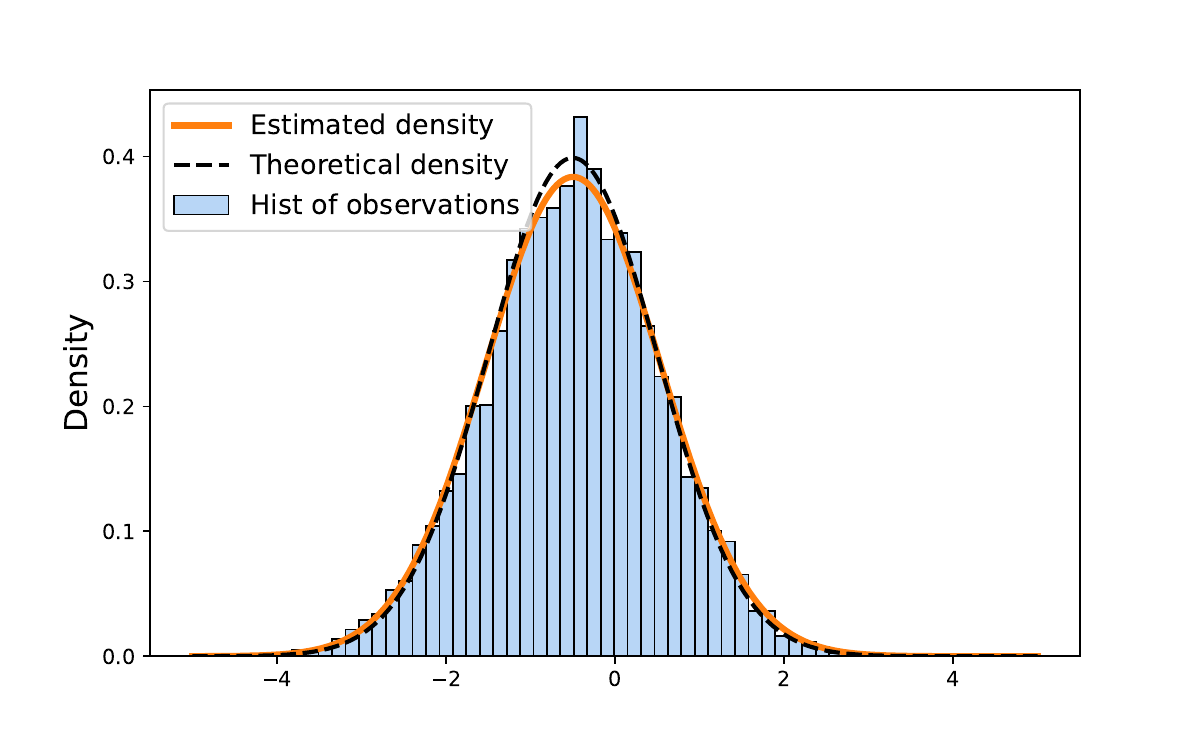}
         \caption{Distribution of $\log(Y_L)$ with $L=50$}
         \label{fig:gbm_depth10}
     \end{subfigure}
     \begin{subfigure}[b]{0.49\textwidth}
         \centering
         \includegraphics[width=\textwidth]{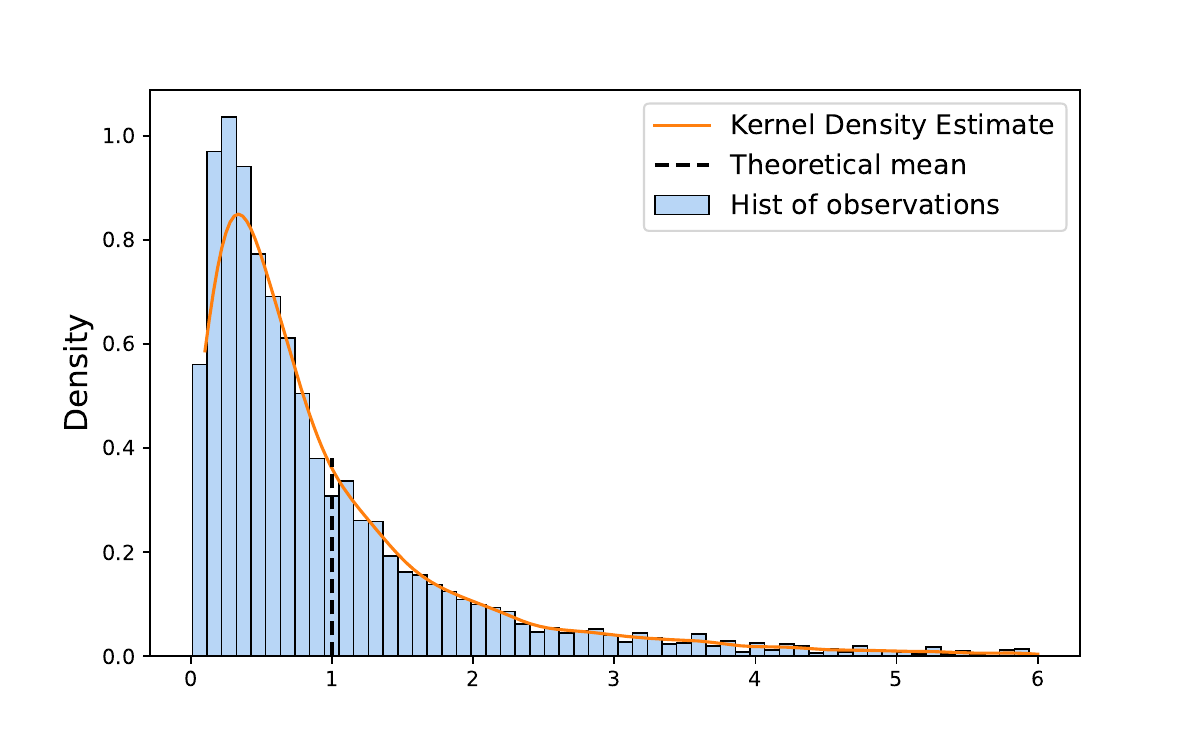}
         \caption{Distribution of $Y_L$ with $L=50$}
         \label{fig:gbm_depth50}
     \end{subfigure}  
    \caption{Empirical verification of \cref{prop:gbm}. \textbf{(a), (c), (e)} Histograms of $\log(Y_L)$ and based on $N=5000$ simulations for depths $L \in \{5, 10, 50\}$ with $Y_0 = 1$. Estimated density (Gaussian kernel estimate) and theoretical density (Gaussian) are  illustrated on the same graphs. 
    \textbf{(b), (d), (f)} Histograms of $Y_L$ based on $N=5000$ simulations for depths $L \in \{5, 10, 50\}$ with $Y_0 = 1$. }
    \label{fig:additional_ou_figs_1}
\end{figure}

\begin{figure}[H]
    \centering
    \begin{subfigure}[b]{0.49\textwidth}
         \centering
         \includegraphics[width=\textwidth]{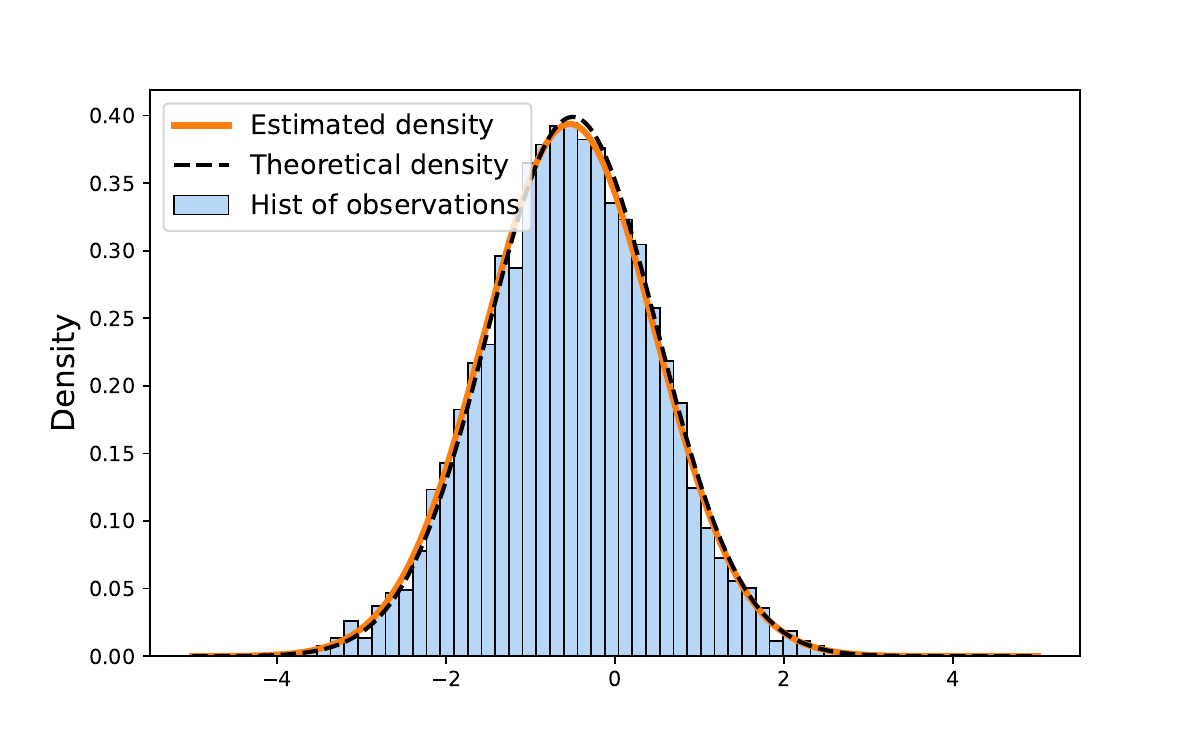}
         \caption{Distribution of $\log(Y_L)$ with $L=100$}
         \label{fig:gbm_depth5}
     \end{subfigure}
     \hfill
     \begin{subfigure}[b]{0.49\textwidth}
         \centering
         \includegraphics[width=\textwidth]{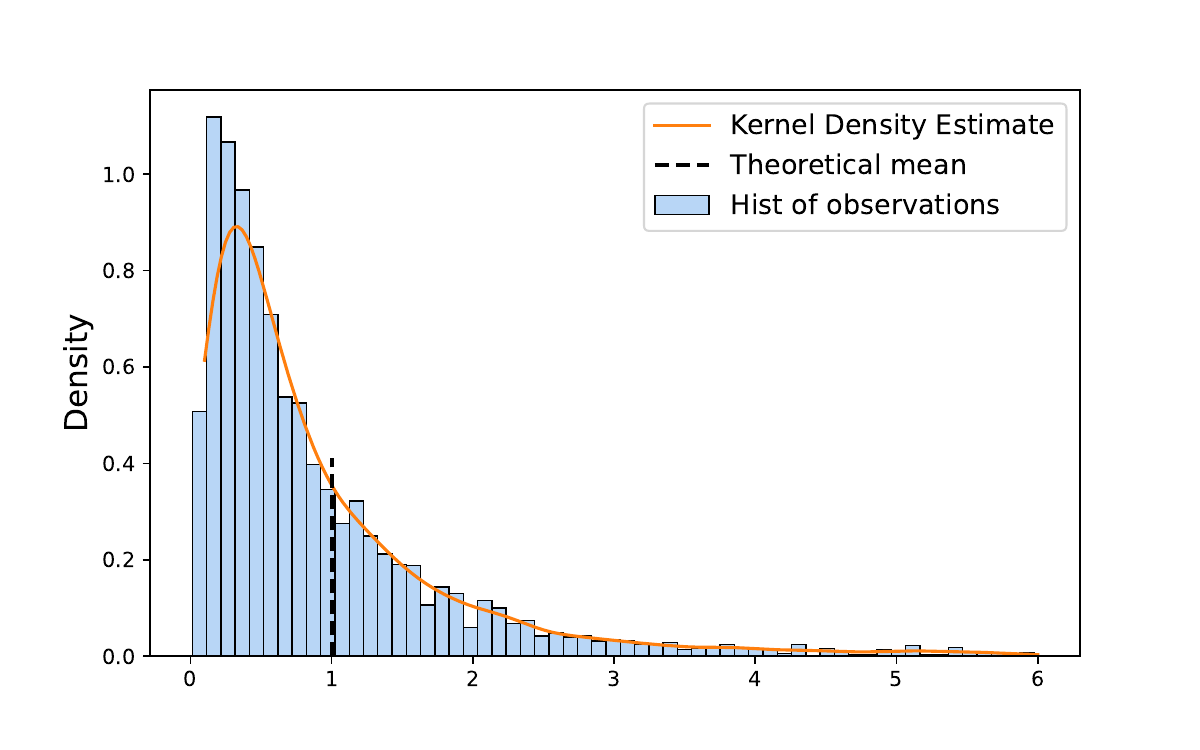}
         \caption{Distribution of $Y_L$ with $L=100$}
         \label{fig:gbm_depth10}
     \end{subfigure}
     \begin{subfigure}[b]{0.49\textwidth}
         \centering
         \includegraphics[width=\textwidth]{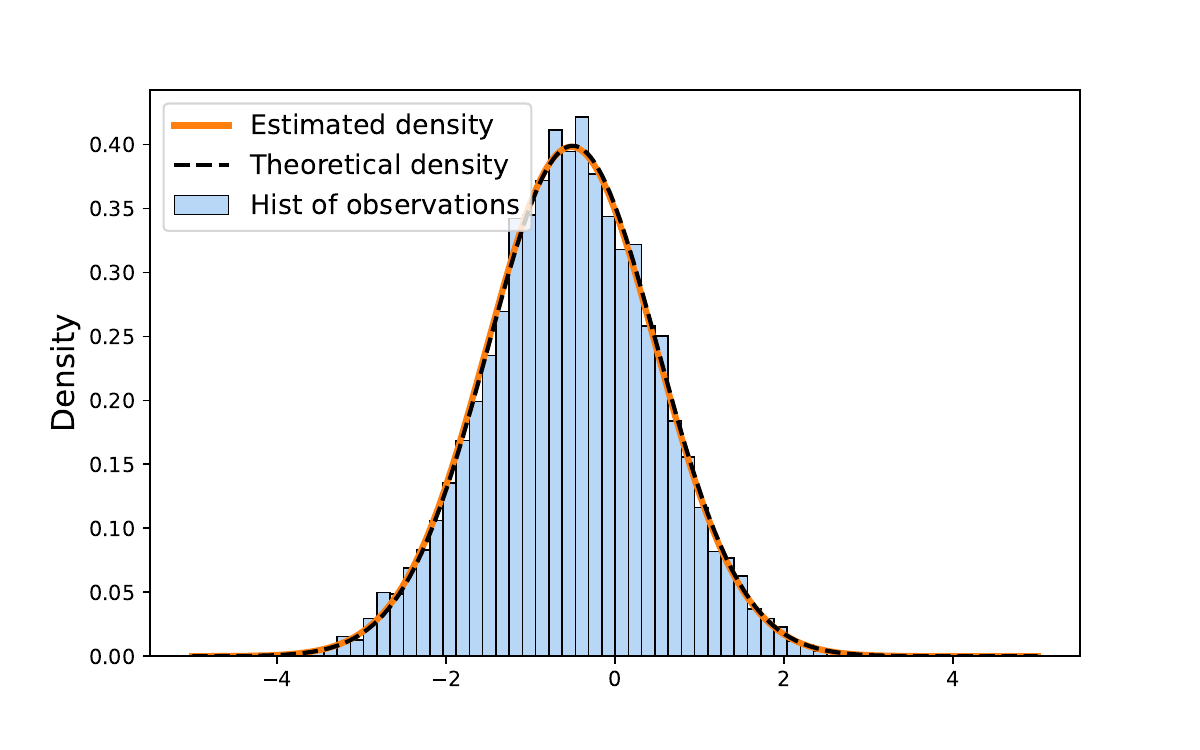}
         \caption{Distribution of $\log(Y_L)$ with $L=200$}
         \label{fig:gbm_depth10}
     \end{subfigure}
     \begin{subfigure}[b]{0.49\textwidth}
         \centering
         \includegraphics[width=\textwidth]{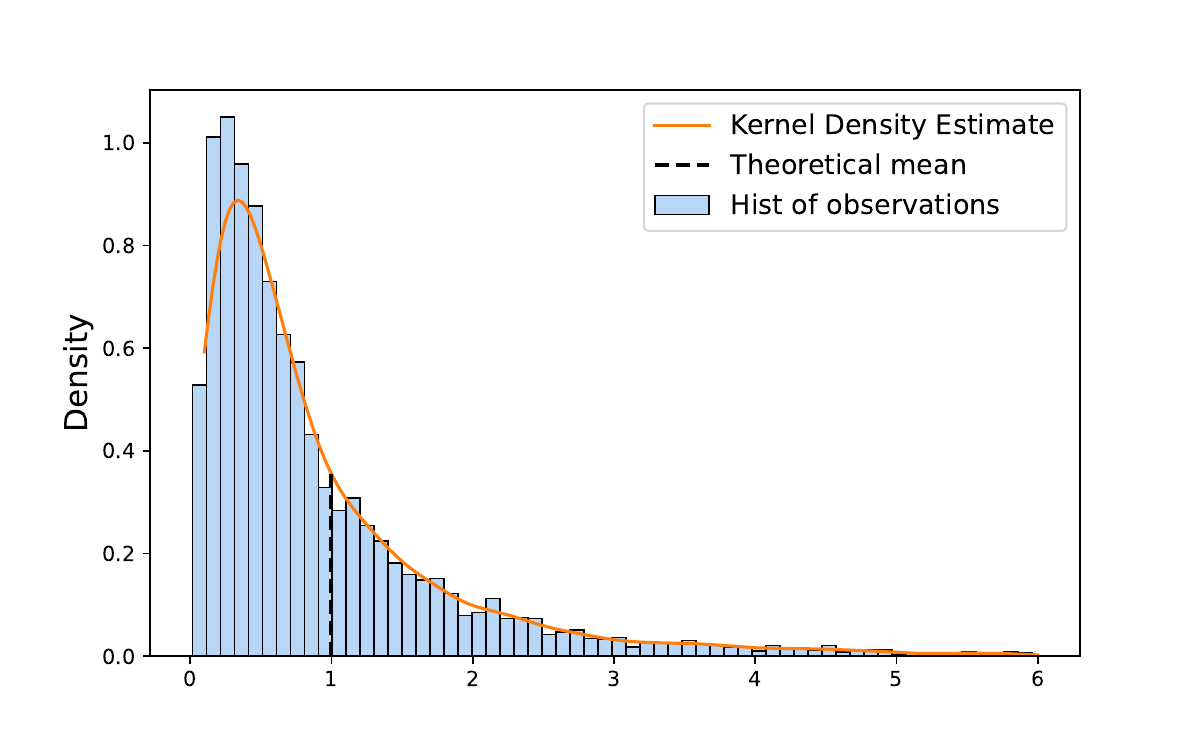}
         \caption{Distribution of $Y_L$ with $L=200$}
         \label{fig:gbm_depth10}
     \end{subfigure}

    \caption{Empirical verification of \cref{prop:gbm}. \textbf{(a), (c)} Histograms of $\log(Y_L)$ and based on $N=5000$ simulations for depths $L \in \{ 100, 200\}$ with $Y_0 = 1$. Estimated density (Gaussian kernel estimate) and theoretical density (Gaussian) are  illustrated on the same graphs. 
    \textbf{(b), (d)} Histograms of $Y_L$ based on $N=5000$ simulations for depths $L \in \{100, 200\}$ with $Y_0 = 1$. }
    \label{fig:additional_ou_figs_2}
\end{figure}

\newpage

\subsection{Histograms of non-scaled log-norm of post-activations}
In \cref{fig:quasi_gaussian_behaviour_non_scaled}, we show the histogram of $\log(\|\phi(Y_L)\|/\|\phi(Y_0)\|)$ based on $N=5000$ simulations. We observe that as the width $n$ increases, the Gaussian approximate is no longer accurate, which is due to the fact that $\|\phi(Y_L)\|/\|\phi(Y_0)\|$ converges to a deterministic value (\cref{thm:infinite_width}).
\begin{figure}[H]
     \centering
     \begin{subfigure}[b]{0.4\textwidth}
         \centering
         \includegraphics[width=\textwidth]{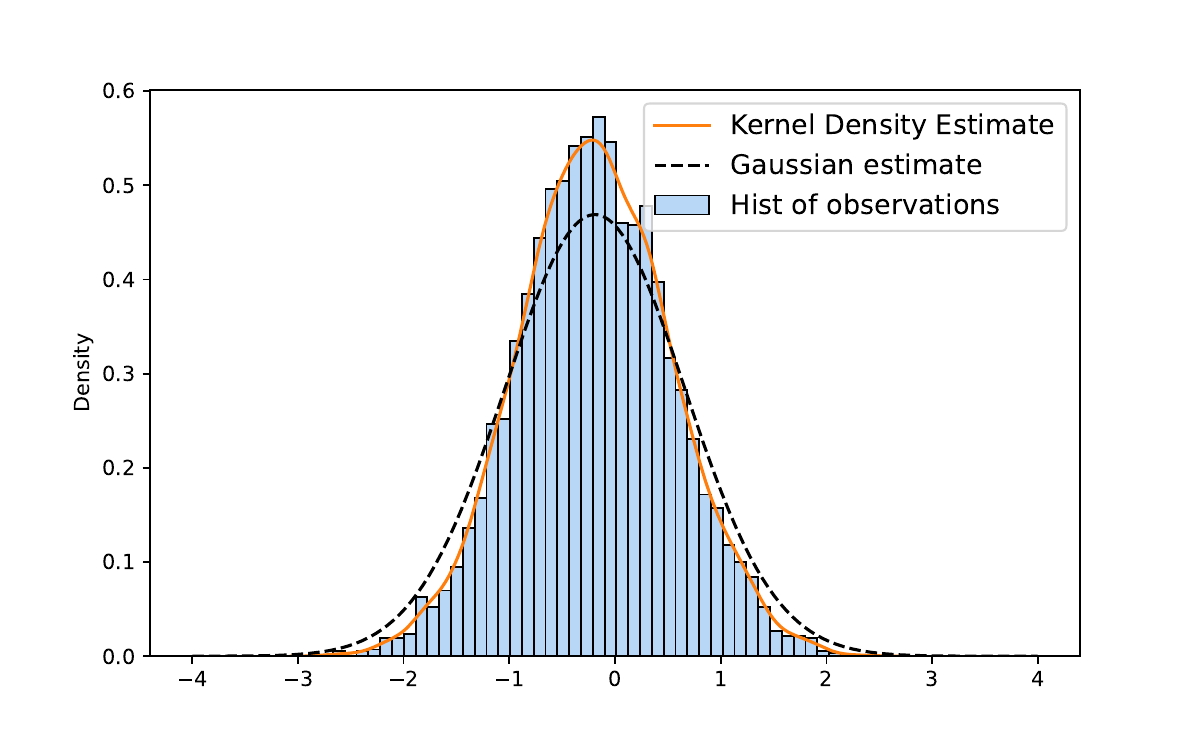}
         \caption{$n=2$}
     \end{subfigure}
     \begin{subfigure}[b]{0.4\textwidth}
         \centering
         \includegraphics[width=\textwidth]{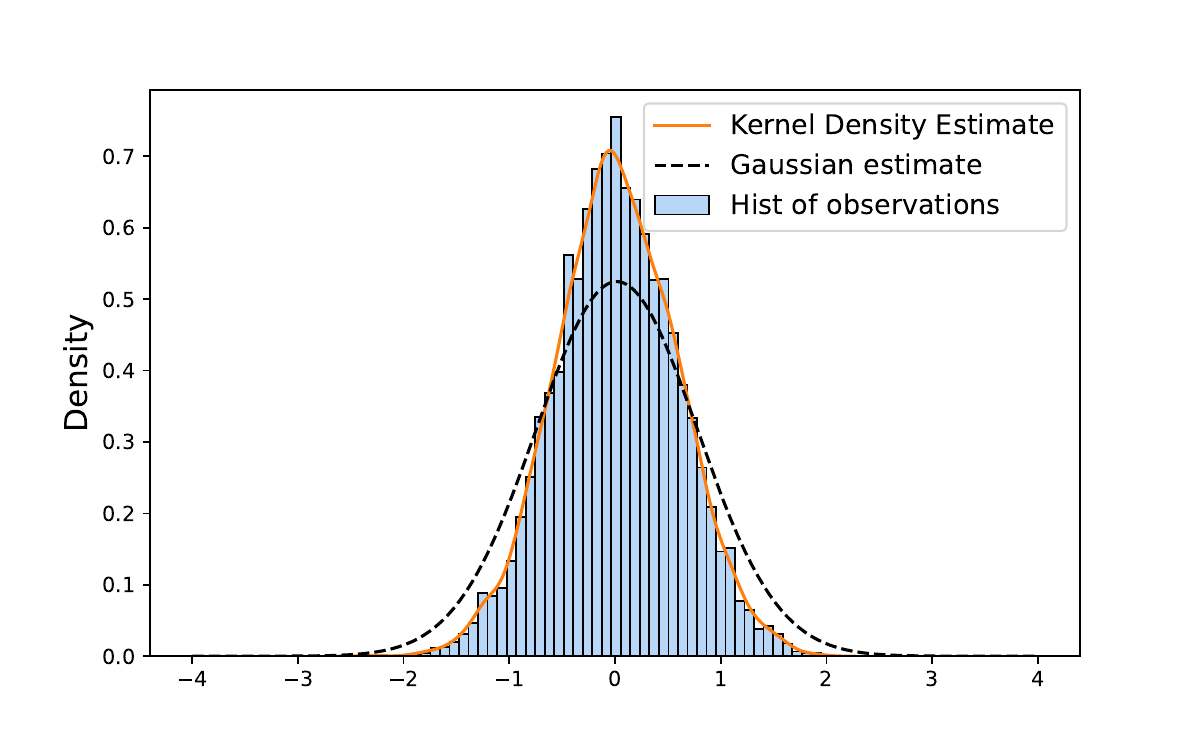}
         \caption{$n=3$}
     \end{subfigure}
     \begin{subfigure}[b]{0.4\textwidth}
         \centering
         \includegraphics[width=\textwidth]{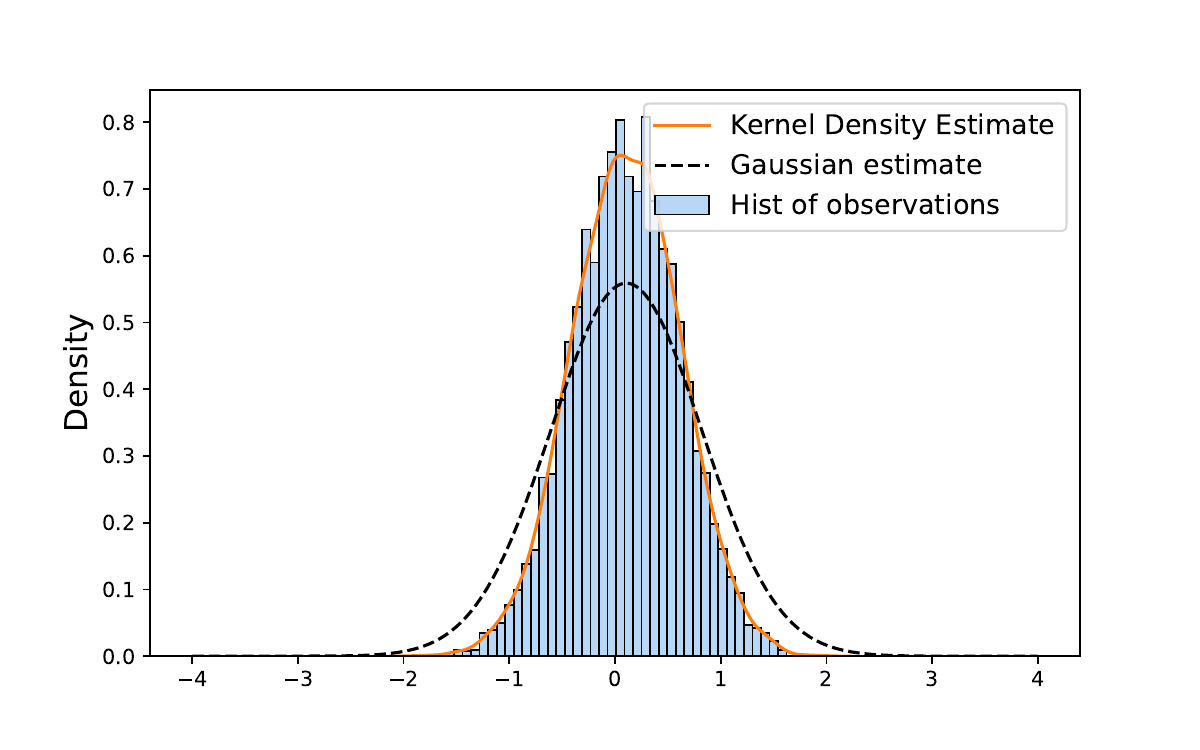}
         \caption{$n=4$}
     \end{subfigure}
     \begin{subfigure}[b]{0.4\textwidth}
         \centering
         \includegraphics[width=\textwidth]{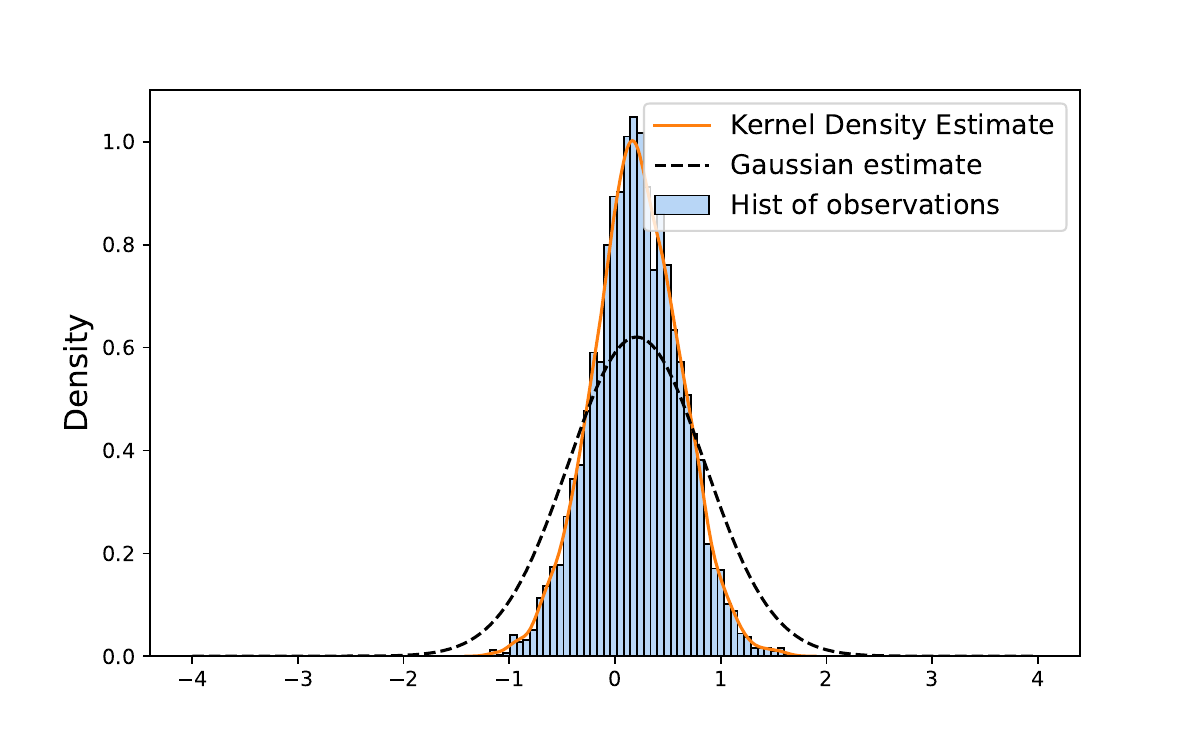}
         \caption{$n=6$}
     \end{subfigure}
     \begin{subfigure}[b]{0.4\textwidth}
         \centering
         \includegraphics[width=\textwidth]{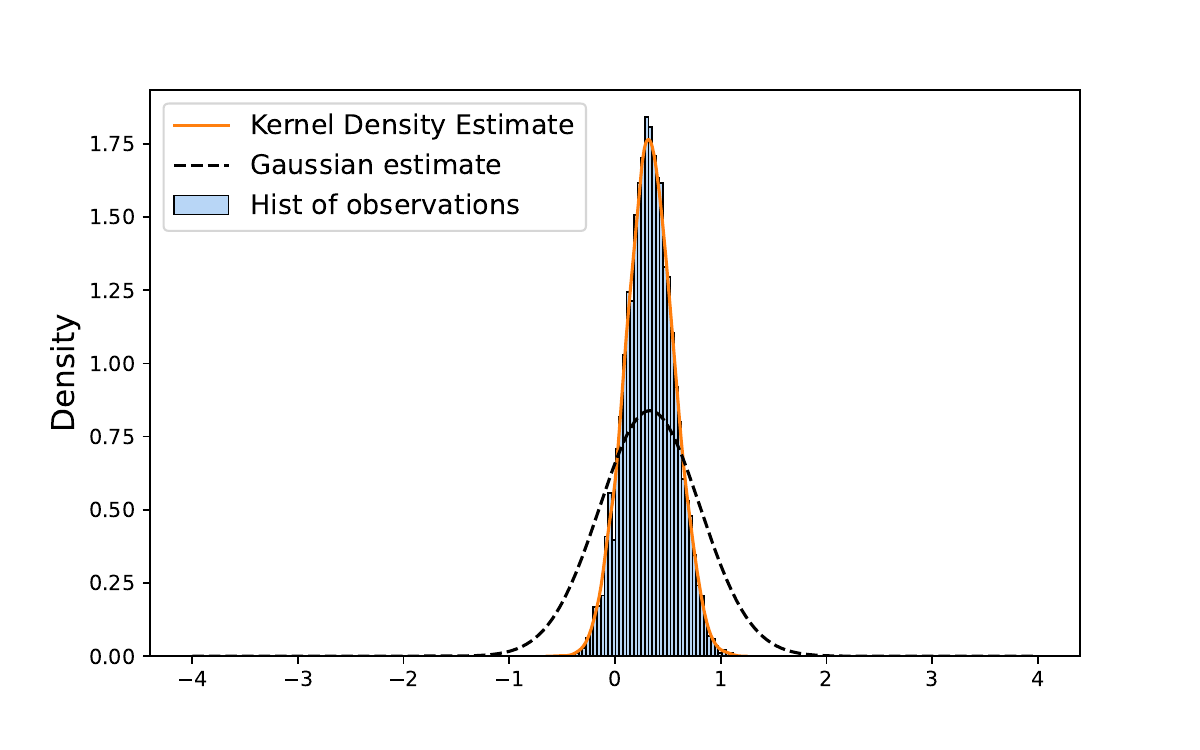}
         \caption{$n=20$}
     \end{subfigure}
     \begin{subfigure}[b]{0.4\textwidth}
         \centering
         \includegraphics[width=\textwidth]{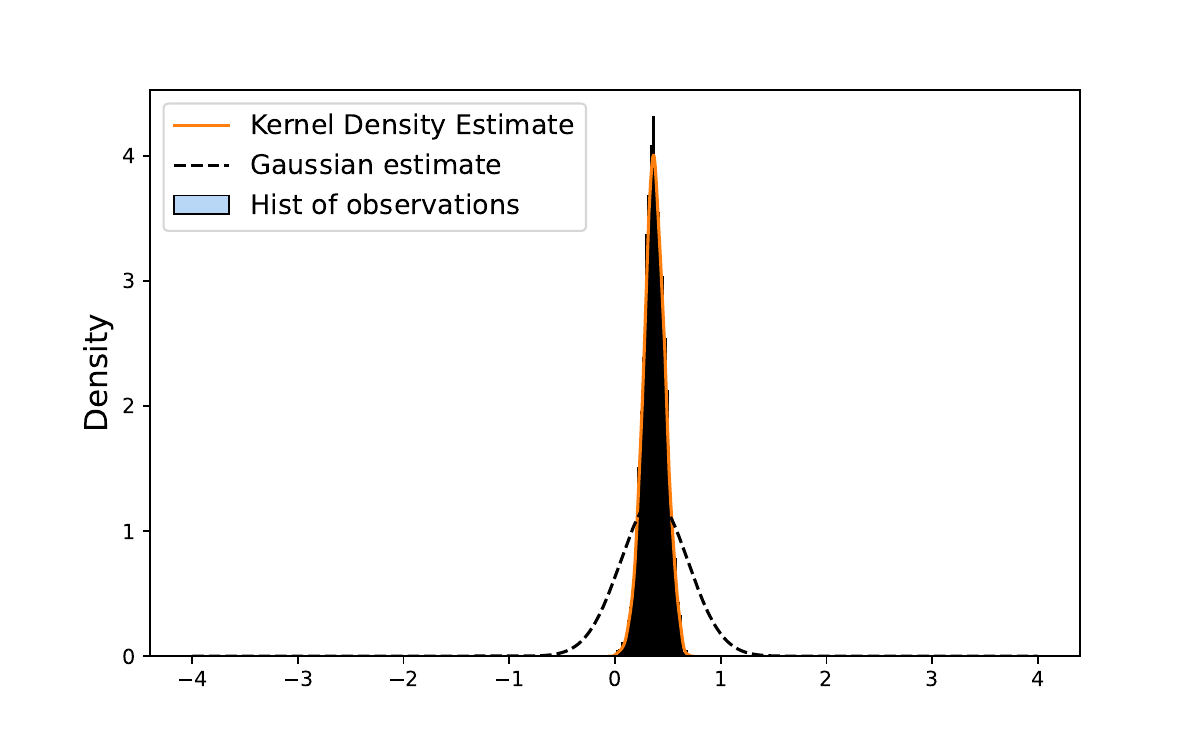}
         \caption{$n=100$}
     \end{subfigure}
        \caption{Histogram of  $ \log(\|\phi(Y_{L})\|/ \|\phi(Y_{0})\|)$ for depth $L = 100$ and different widths $n \in \{2,3,4,6,20, 100\}$. Gaussian density estimate and (Gaussian) kernel density estimate are shown. As the width increases, we observe a deterioration of the match between the best Gaussian estimate and the empirical distribution. This is due to the fact that the norm of the post-activations concentrates around a deterministic value when $n$ goes to infinity (\cref{thm:infinite_width}).}
        \label{fig:quasi_gaussian_behaviour_non_scaled}
\end{figure}

\newpage

\subsection{Evolution of $\sqrt{n}\log(\|\phi(Y_l)\|/ \|\phi(Y_0)\|)$.}

In \cref{fig:hist_series_width2_depth100_scaled}, \cref{fig:hist_series_width3_depth100_scaled}, \cref{fig:hist_series_width20_depth100_scaled}, and \cref{fig:hist_series_width100_depth100_scaled}, we show the histograms of $\sqrt{n} \log(\|\phi(Y_l)\|/\|\phi(Y_0)\|)$ for depth $L=100$, hidden layers $l \in \{10, 30, 40, 60, 70, 90\}$, and widths $n \in \{2, 3, 20, 100\}$. We observe that Gaussian distribution fits better the last layers. This was expected since the limiting distribution (Quasi-GBM) given in \cref{thm:norm_post_act} is only valid for layer indices $\lfloor t L\rfloor$ when $L$ goes to infinity. Thus, for small $l$, it should be expected that the Gaussian distribution would not be a good approximation.
\begin{figure}[H]
    \includegraphics[width=\linewidth]{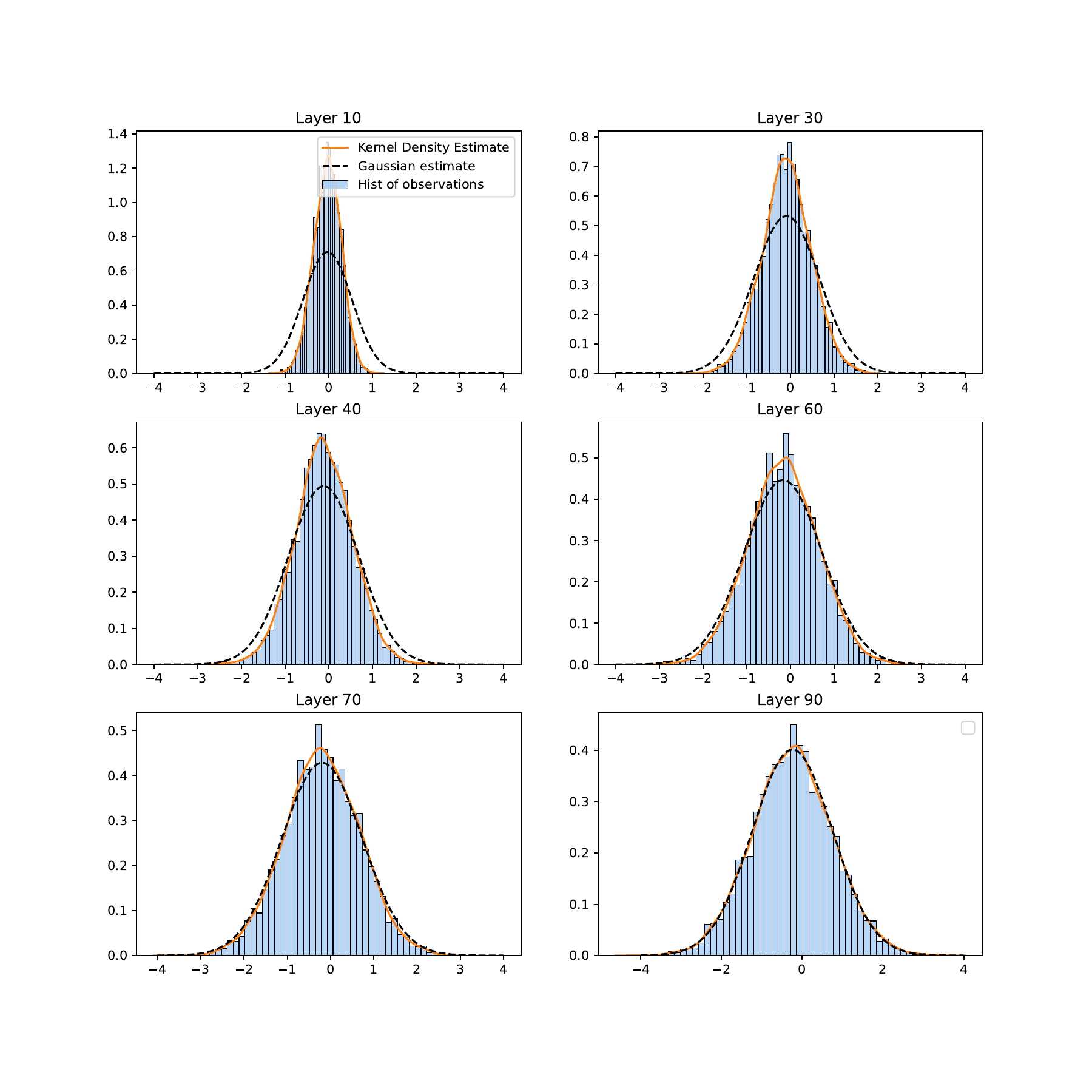}
    \caption{Distribution of $\sqrt{n}\log(\|\phi(Y_l)\|/ \|\phi(Y_0)\|)$ (\cref{eq:resnet}) for different layer indices, with depth $L=100$ and width $n = 2$}
    \label{fig:hist_series_width2_depth100_scaled}
\end{figure}

\begin{figure}[H]
    \centering\includegraphics[width=\linewidth]{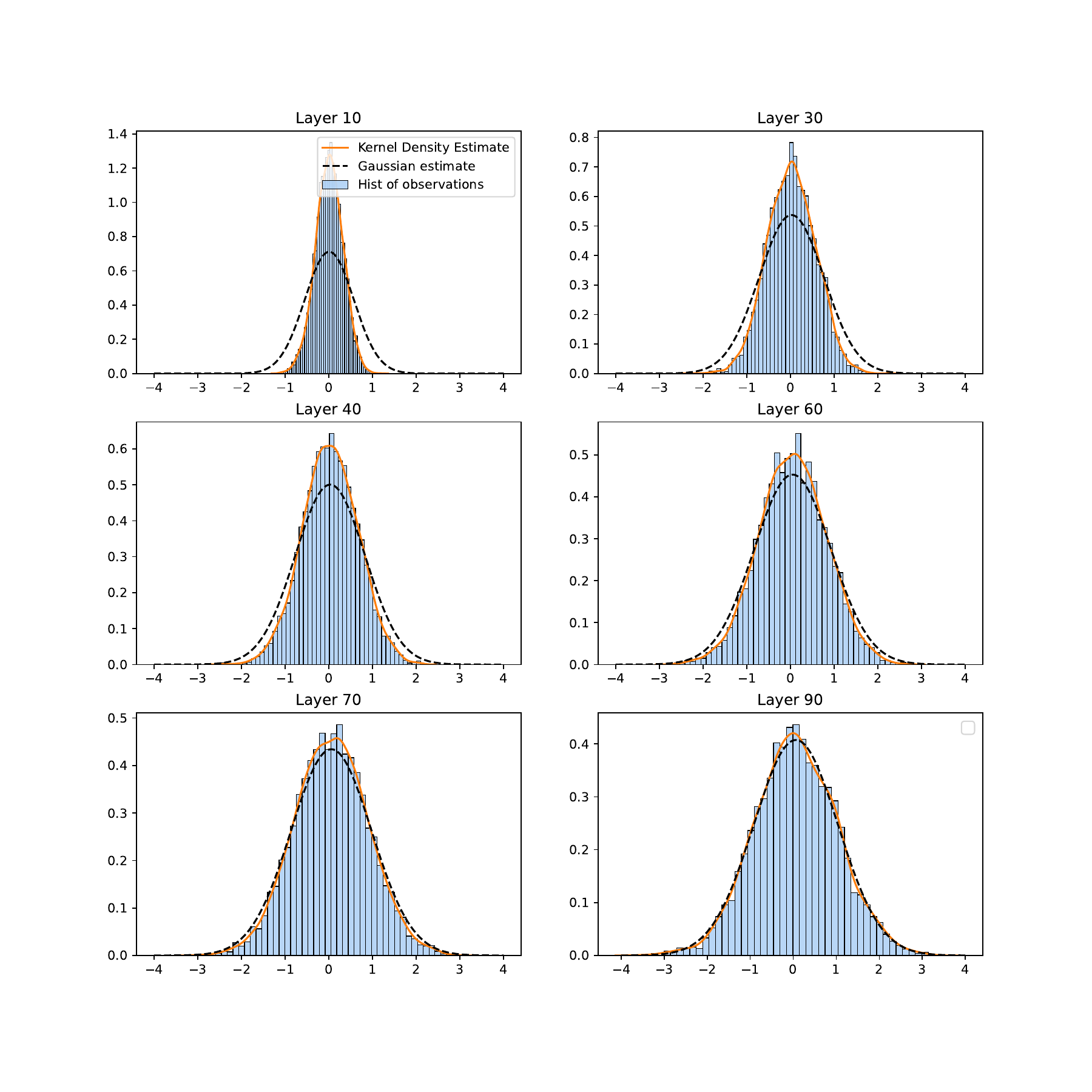}
    \caption{Distribution of $\sqrt{n}\log(\|\phi(Y_l)\|/ \|\phi(Y_0)\|)$ (\cref{eq:resnet}) for different layer indices, with depth $L=100$ and width $n = 3$}
    \label{fig:hist_series_width3_depth100_scaled}
\end{figure}

\begin{figure}[H]
    \centering
    \includegraphics[width=\linewidth]{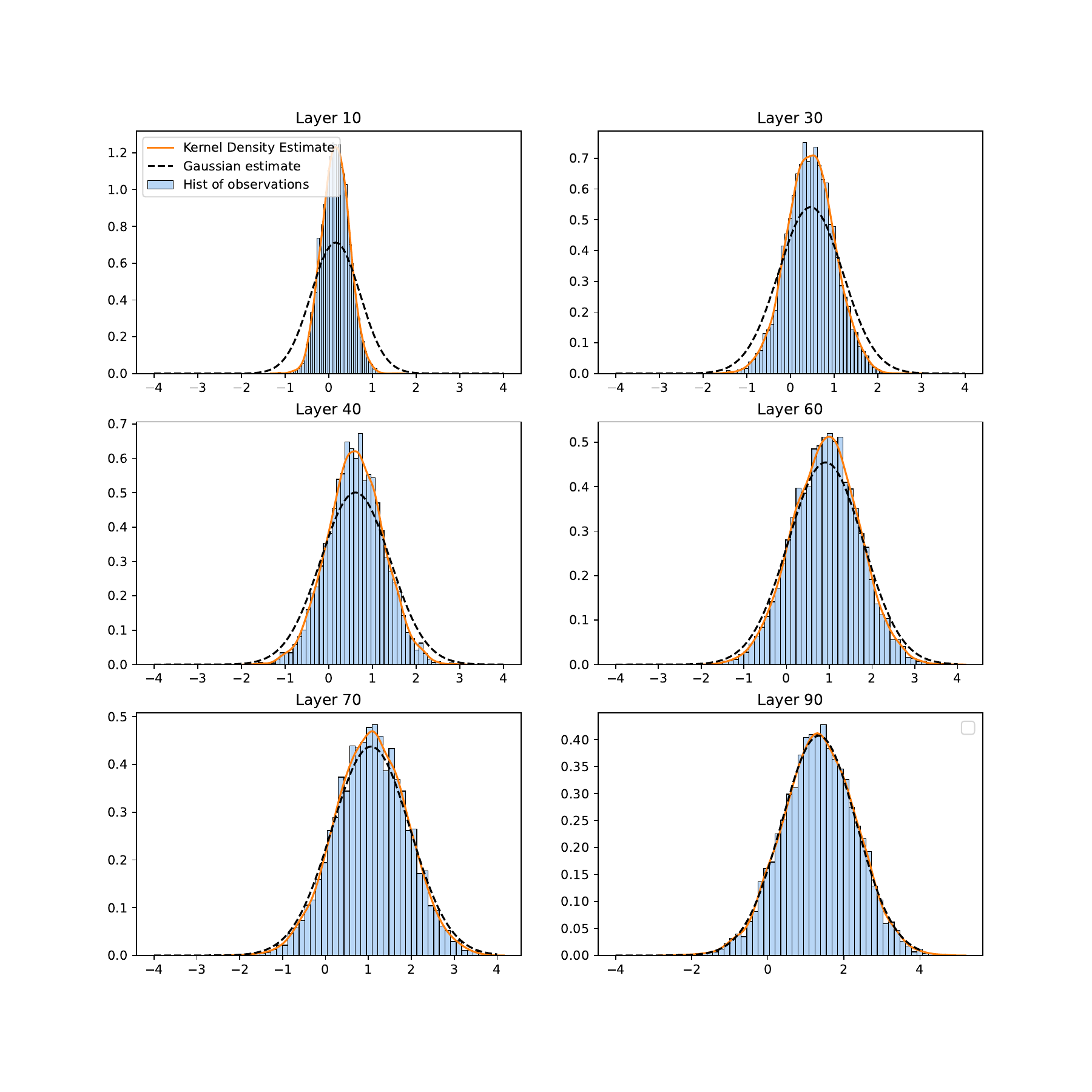}
    \caption{Distribution of $\sqrt{n}\log(\|\phi(Y_l)\|/ \|\phi(Y_0)\|)$ (\cref{eq:resnet}) for different layer indices, with depth $L=100$ and width $n = 20$}
    \label{fig:hist_series_width20_depth100_scaled}
\end{figure}

\begin{figure}[H]
    \centering
    \includegraphics[width=\linewidth]{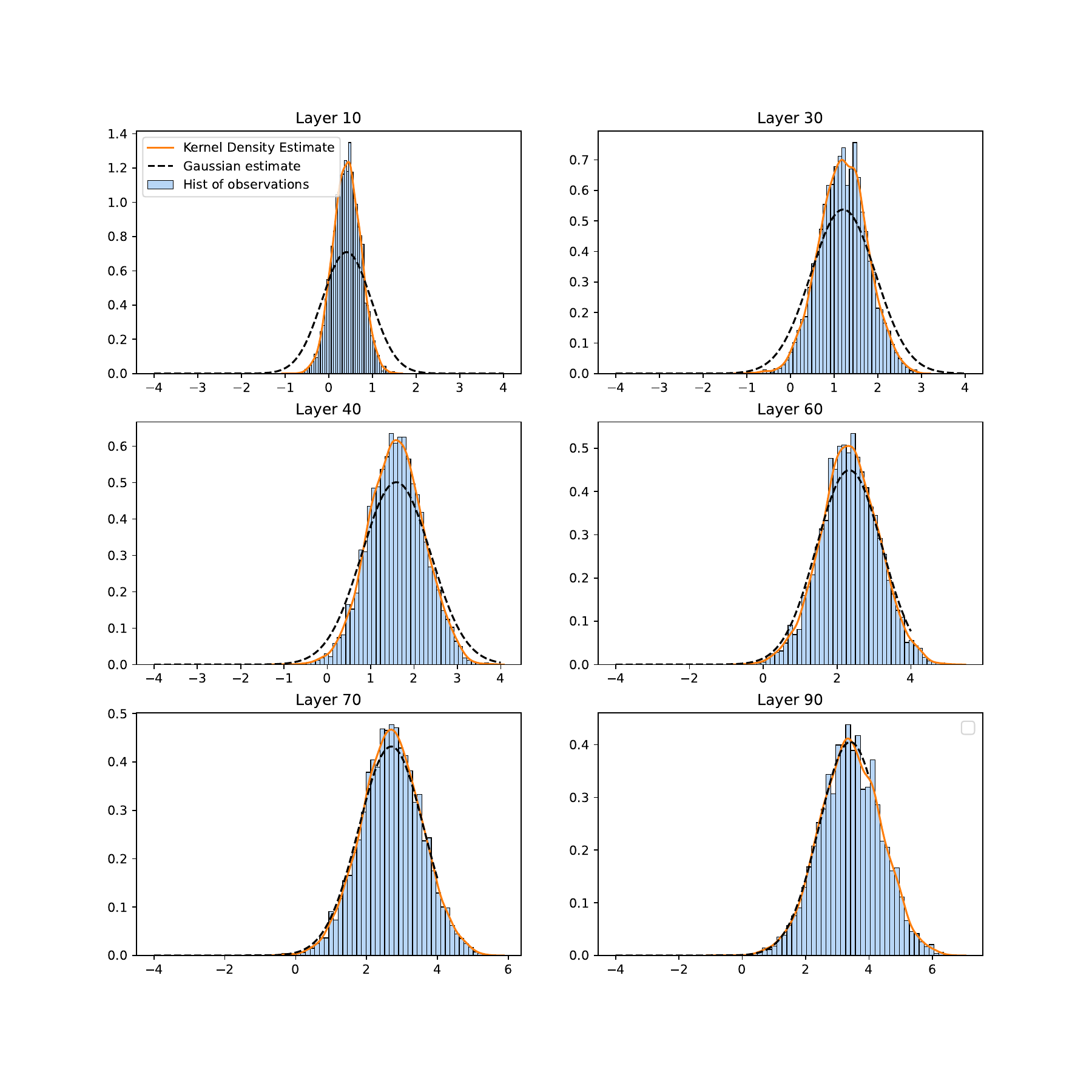}
    \caption{Distribution of $\sqrt{n}\log(\|\phi(Y_l)\|/ \|\phi(Y_0)\|)$ (\cref{eq:resnet}) for different layer indices, with depth $L=100$ and width $n = 100$}
    \label{fig:hist_series_width100_depth100_scaled}
\end{figure}

\newpage
\subsection{Evolution of $\log(\|\phi(Y_l)\|/ \|\phi(Y_0)\|)$ (non-scaled).}\label{sec:non_scaled_log_ratio}
In \cref{fig:hist_series_width2_depth100_ns}, \cref{fig:hist_series_width3_depth100_ns}, \cref{fig:hist_series_width20_depth100_ns}, and \cref{fig:hist_series_width100_depth100_ns}, we show the non-scaled versions of the histograms from the previous section. We observe that the histogram concentrates around a single value (the distribution converges to a Dirac mass) as $n$ increases. This is a result of the asymptotic behaviour of the ResNet in the infinite-depth-then-infinite-width limit as shown in \cref{thm:infinite_width}.
\begin{figure}[H]
    \includegraphics[width=\linewidth]{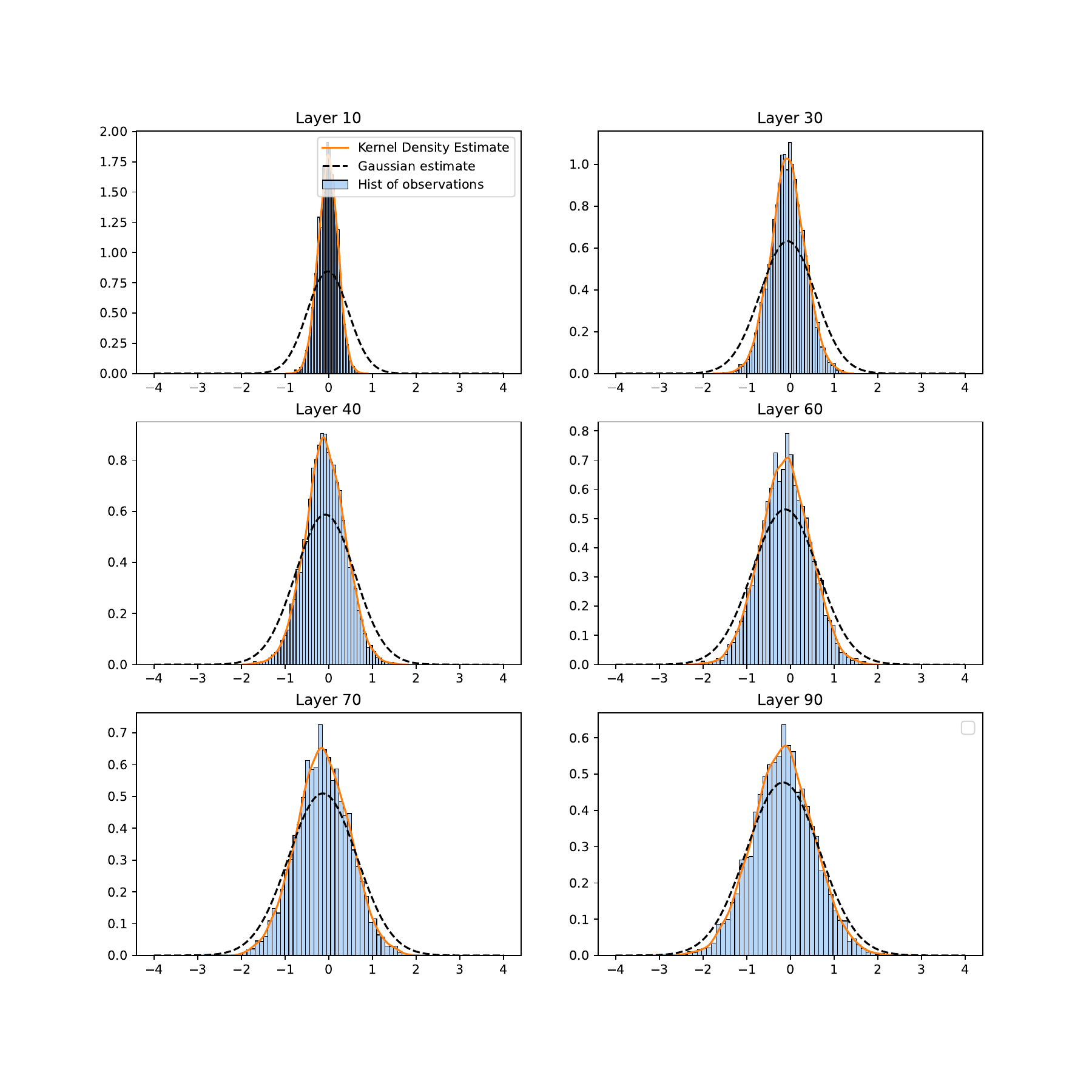}
    \caption{Distribution of $\log(\|\phi(Y_l)\|/ \|\phi(Y_0)\|)$ (\cref{eq:resnet}) for different layer indices, with depth $L=100$ and width $n = 2$}
    \label{fig:hist_series_width2_depth100_ns}
\end{figure}

\begin{figure}[H]
    \centering\includegraphics[width=\linewidth]{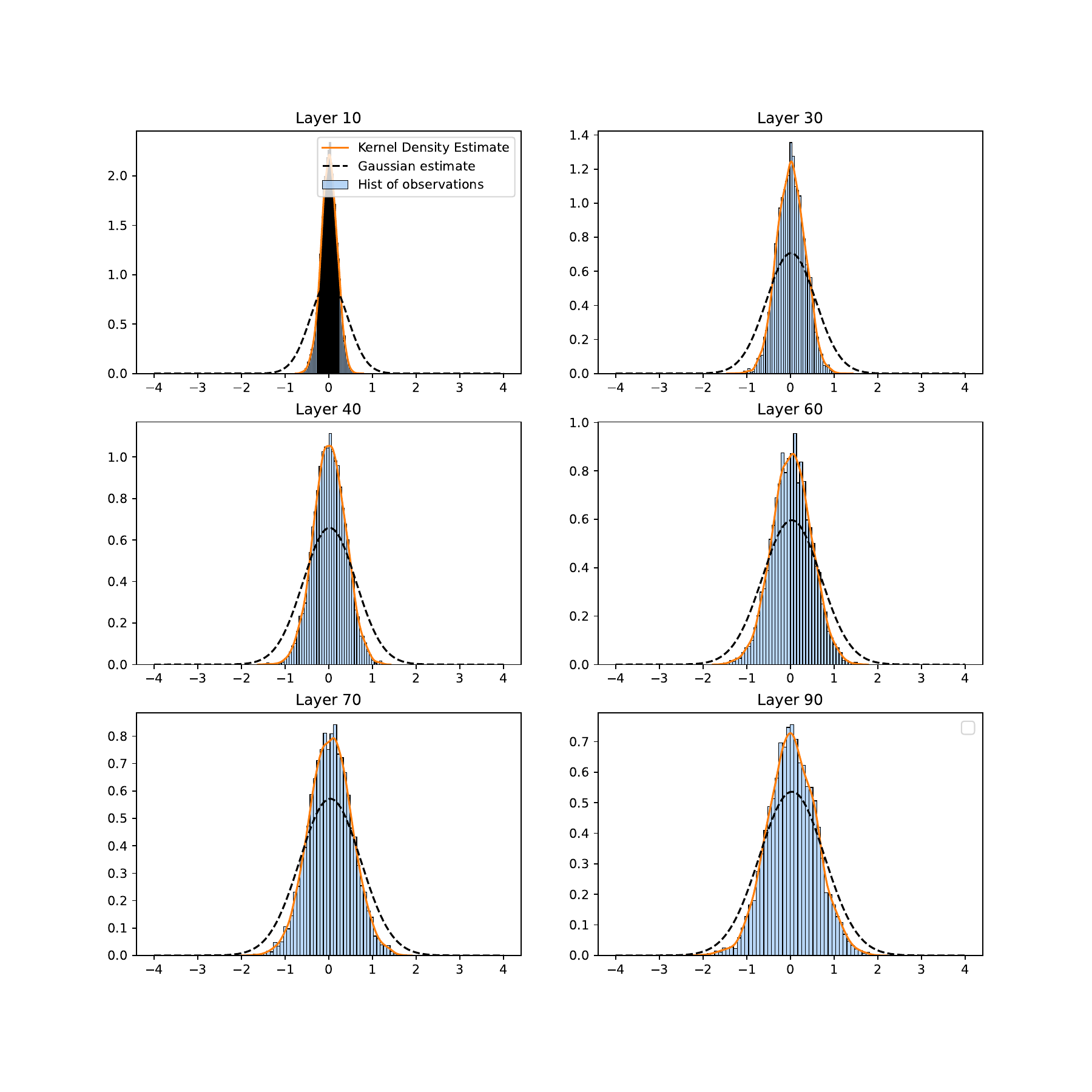}
    \caption{Distribution of $\log(\|\phi(Y_l)\|/ \|\phi(Y_0)\|)$ (\cref{eq:resnet}) for different layer indices, with depth $L=100$ and width $n = 3$}
    \label{fig:hist_series_width3_depth100_ns}
\end{figure}

\begin{figure}[H]
    \centering
    \includegraphics[width=\linewidth]{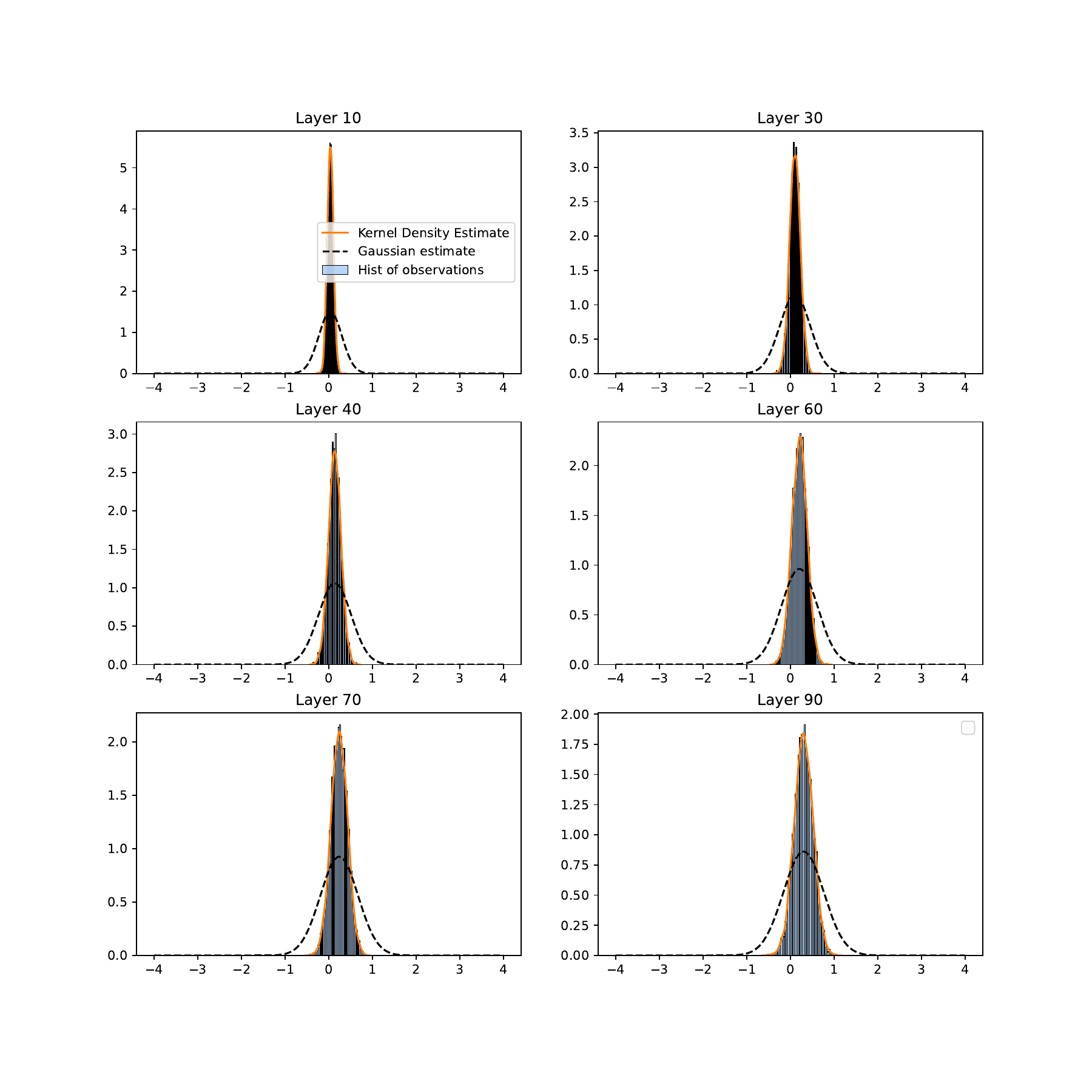}
    \caption{Distribution of $\log(\|\phi(Y_l)\|/ \|\phi(Y_0)\|)$ (\cref{eq:resnet}) for different layer indices, with depth $L=100$ and width $n = 20$}
    \label{fig:hist_series_width20_depth100_ns}
\end{figure}

\begin{figure}[H]
    \centering
    \includegraphics[width=\linewidth]{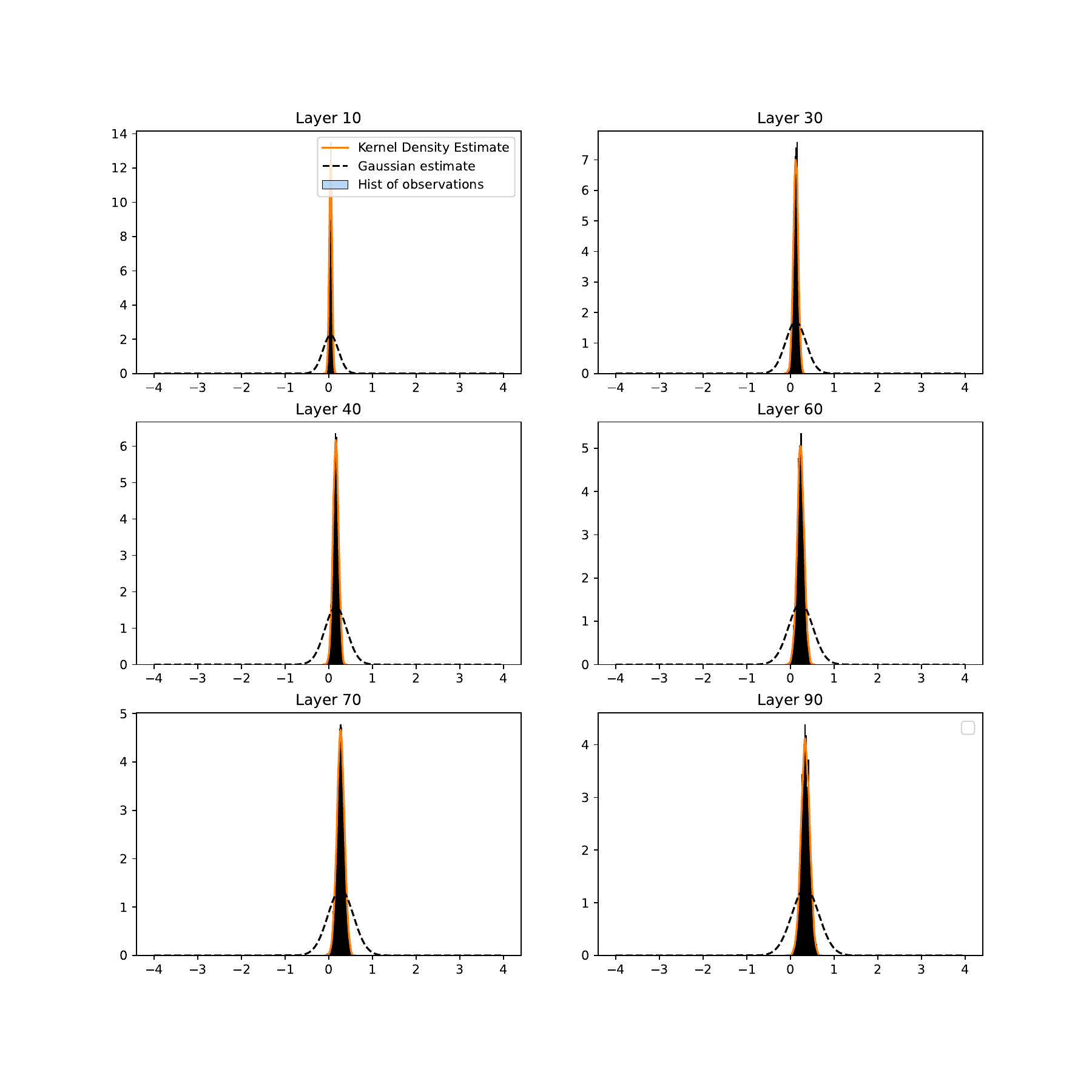}
    \caption{Distribution of $\log(\|\phi(Y_l)\|/ \|\phi(Y_0)\|)$ (\cref{eq:resnet}) for different layer indices, with depth $L=100$ and width $n = 100$.}
    \label{fig:hist_series_width100_depth100_ns}
\end{figure}

\end{document}